\documentclass{article}

\pdfoutput=1
\def\fix{\textrm{\rm fix}}
\def\splitrm{\textrm{\rm split}}
\def\leftrm{\textrm{\rm left}}
\def\rightrm{\textrm{\rm right}}
\def\cap{\textrm{cap}}
\def\leaf{\textrm{leaf}}
\def\h2{\hspace*{-2pt}}
\def\x{\mathbf{x}}
\def\y{\mathbf{y}}
\def\ch{\textrm{ch}}
\def\AUC{\textrm{AUC}}
\def\uncertain{\textrm{\rm uncertain}}
\def\supp{\textrm{\rm supp}}

\def\move{\textrm{\rm move}}

\def\cumulative{\textrm{\rm cumulative}}

\usepackage{times}
\usepackage{graphicx} 
\usepackage{colortbl}
\usepackage{subfig}

\usepackage{amsmath}
\usepackage{mathtools}

\usepackage{amssymb}
\usepackage{dsfont}
\usepackage{amsthm}

\usepackage[linguistics]{forest}

\usepackage{graphicx}
\graphicspath{ {./figure/} }
\usepackage{caption}

\usepackage{natbib}

\usepackage{algorithm}
\usepackage{algorithmic}
\usepackage{color,soul}

\usepackage[utf8]{inputenc}
\usepackage[english]{babel}
\newtheorem{theorem}{Theorem}[section]
\newtheorem{lemma}[theorem]{Lemma}

\usepackage{url}
\usepackage{enumitem}

 \newcommand{\argmin}{\textrm{\rm argmin}}

\usepackage[accepted]{icml2020} 

\icmltitlerunning{Generalized Scalable Optimal Sparse Decision Trees}

\begin{document} 

\twocolumn[
\icmltitle{Generalized and Scalable Optimal Sparse Decision Trees}



\icmlsetsymbol{equal}{*}
\begin{icmlauthorlist}
\icmlauthor{Jimmy Lin}{equal,ubc}
\icmlauthor{Chudi Zhong}{equal,duke}
\icmlauthor{Diane Hu}{duke}
\icmlauthor{Cynthia Rudin}{duke}
\icmlauthor{Margo Seltzer}{ubc}
\end{icmlauthorlist}

\icmlaffiliation{ubc}{University of British Columbia, Vancouver, Canada}
\icmlaffiliation{duke}{Duke University, Durham, North Carolina, USA}

\icmlcorrespondingauthor{Cynthia Rudin}{cynthia@cs.duke.edu}

\icmlkeywords{decision trees, discrete optimization, dynamic programming, sparsity}

\vskip 0.3in
]



\printAffiliationsAndNotice{\icmlEqualContribution} 

\begin{abstract} 
Decision tree optimization is notoriously difficult from a computational perspective but essential for the field of interpretable machine learning. Despite efforts over the past 40 years, only recently have optimization breakthroughs been made that have allowed practical algorithms to find \textit{optimal} decision trees. These new techniques have the potential to trigger a paradigm shift where it is possible to construct sparse decision trees to efficiently optimize a variety of objective functions without relying on greedy splitting and pruning heuristics that often lead to suboptimal solutions. The contribution in this work is to provide a general framework for decision tree optimization that addresses the two significant open problems in the area: treatment of imbalanced data and fully optimizing over continuous variables. We present techniques that produce optimal decision trees over a variety of objectives including F-score, AUC, and partial area under the ROC convex hull. We also introduce a scalable algorithm that produces provably optimal results in the presence of continuous variables and speeds up decision tree construction by several orders of magnitude relative to the state-of-the art.
\end{abstract} 

\section{Introduction}
\label{Sec:intro}
Decision tree learning has served as a foundation for interpretable artificial intelligence and machine learning for over half a century \cite{MorganSo1963,payne1977algorithm,Loh14}. 
The major approach since the 1980's has been decision tree induction, where heuristic splitting and pruning procedures grow a tree from the top down and prune it back afterwards \citep{Quinlan93,Breiman84}. The problem with these methods is that they tend to produce suboptimal trees with no way of knowing how suboptimal the solution is. This leaves a gap between the performance that a decision tree \textit{might} obtain and the performance that one actually attains, with no way to check on (or remedy) the size of this gap--and sometimes, the gap can be large. 

Full decision tree optimization is NP-hard, with no polynomial-time approximation
\citep{laurent1976constructing}, leading to challenges in proving optimality or bounding the optimality gap in a reasonable amount of time, even for small datasets. It is possible to create assumptions that reduce hard decision tree optimization to cases where greedy algorithms suffice, such as independence between features \cite{Klivans06}, but these assumptions do not hold in reality. If the data can be perfectly separated with zero error, SAT solvers can be used to find optimal decision trees rapidly \cite{narodytska2018learning}; however, real data is generally not separable, leaving us with no choice other than to actually solve the problem.

Decision tree optimization is amenable to branch-and-bound methods, implemented either via generic mathematical programming solvers or by customized algorithms. Solvers have been used from the 1990's \citep{Bennett92,Bennett96optimaldecision} to the present \cite{verwer2019learning,molerooptimal,nijssen2020,bertsimas2017optimal,ErtekinRu18,MenickellyGKS18,vilas2019}, but these generic solvers tend to be slow. A common way to speed them up is to make approximations in preprocessing to reduce the size of the search space. For instance, ``bucketization'' preprocessing is used in both generic solvers \cite{verwer2019learning} and customized algorithms \cite{nijssen2020} to handle continuous variables. Continuous variables pose challenges to optimality; even one continuous variable increases the number of possible splits by the number of possible values of that variable in the entire database, and each additional split leads to an exponential increase in the size of the optimization problem. 
Bucketization is tempting and seems innocuous, but we prove in Section \ref{sec:DoesntWork} that it sacrifices optimality.



Dynamic programming methods have been used for various decision tree optimization problems since as far back as the early 1970's \cite{Garey72,MeiselMi73}. Of the more recent attempts at this challenging problem, \citet{GarofalakisHyRaSh03} use a dynamic programming method for finding an optimal subtree within a predefined larger decision tree grown using standard greedy induction. Their trees inherit suboptimality from the induction procedure used to create the larger tree. 
The DL8 algorithm \cite{nijssen2007mining} performs dynamic programming on the space of decision trees. 
However, without mechanisms to reduce the size of the space and to reduce computation, the method cannot be practical. A more practical extension is the DL8.5 method \cite{nijssen2020}, which uses a hierarchical upper bound theorem to reduce the size of the search space. However, it also uses bucketization preprocessing, which sacrifices optimality; without this preprocessing or other mechanisms to reduce computation, the method suffers in computational speed.


The CORELS algorithm \citep{AngelinoLaAlSeRu17-kdd,AngelinoEtAl18,Larus-Stone17}, which is an associative classification method rather than an optimal decision tree method, breaks away from the previous literature in that it is a custom branch-and-bound method with custom bounding theorems, its own bit-vector library, specialized data structures, and an implementation that leverages computational reuse. CORELS is able to solve problems within a minute that, using any other prior approach, might have taken weeks, or even months or years.
Hu et al.~\cite{HuRuSe2019} adapted the CORELS philosophy to produce an Optimal Sparse Decision Tree (OSDT) algorithm that leverages some of CORELS' libraries and its computational reuse paradigm, as well as many of its theorems, which dramatically reduce the size of the search space. However, OSDT solves an exponentially harder problem than that of CORELS' rule list optimization, producing scalability challenges, as we might expect. 

This work addresses two fundamental limitations in existing work: unsatisfying results for imbalanced data and scalability challenges when trying to fully optimize over continuous variables.
Thus, the first contribution of this work is
\textit{to massively generalize sparse decision tree optimization to handle a wide variety of objective functions, including weighted accuracy (including multi-class), balanced accuracy, F-score, AUC and partial area under the ROC convex hull.
} Both CORELS and OSDT were designed to maximize accuracy, regularized by sparsity, and neither were designed to handle other objectives. CORELS has been generalized \citep{FairCORELS,ChenRu18} to handle some constraints, but not to the wide variety of different objectives one might want to handle in practice. Generalization to some objectives is straightforward (e.g., weighted accuracy) but non-trivial in cases of optimizing rank statistics (e.g., AUC), which typically require quadratic computation in the number of observations in the dataset. However, for \textit{sparse} decision trees, this time is much less than quadratic, because all observations within a leaf of a tree are tied in score, and  there are a sparse number of leaves in the tree. Taking advantage of this permits us to rapidly calculate rank statistics and thus optimize over them. The second contribution is \textit{to present a new representation of the dynamic programming search space that exposes a high degree of computational reuse when modelling continuous features}.
 The new search space representation provides a useful solution to a problem identified in the CORELS paper, which is how to use ``similar support'' bounds in practice. A similar support bound states that if two features in the dataset are similar, but not identical, to each other, then bounds obtained using the first feature for a split in a tree can be leveraged to obtain bounds for the same tree, were the second feature to replace the first feature. However, if the algorithm checks the similar support bound too frequently, the bound slows the algorithm down, despite reducing the search space. Our method uses hash trees that represent similar trees using shared subtrees, which naturally accelerates the evaluation of similar trees. The implementation, coupled with a new type of incremental similar support bound, is  efficient enough to handle a few continuous features by creating dummy variables for all unique split points along a feature. This permits us to obtain smaller optimality gaps and certificates of optimality for mixed binary and continuous data when optimizing additive loss functions several orders of magnitude more quickly than any other method that currently exists. 

Our algorithm is called Generalized and Scalable Optimal Sparse Decision Trees (GOSDT, pronounced ``ghost''). A chart detailing a qualitative comparison of GOSDT to previous decision tree approaches is in Appendix \ref{app:ComparisonTable}.


\section{Notation and Objectives}\label{Sec:notation}
We denote the training dataset as $\{(x_i, y_i)\}_{i=1}^N$ , where $x_i \in \{0,1\}^M$ are binary features. Our notation uses $y_i \in \{0, 1\}$, though  our code is implemented for multiclass classification as well. For each real-valued feature, we create a split point at the mean value between every ordered pair of unique values present in the training data. Following notation of \citet{HuRuSe2019}, we represent a tree as a set of leaves; this is important because it allows us \textit{not} to store the splits of the tree, only the conditions leading to each leaf. A leaf set $d = (l_1,l_2, ..., l_{H_d})$ contains $H_d$ distinct leaves, 
where $l_i$ is the classification rule of the leaf $i$, that is, the set of conditions along the branches that lead to the leaf, and $\hat{y}_i^{\leaf}$ is the label prediction for all data in leaf $i$. 
For a tree $d$, we define the objective function as a combination of the loss and a penalty on the number of leaves, with regularization parameter $\lambda$:
\begin{equation}\label{eq:risk}
    R(d,\x,\y) = \ell(d,\x,\y)+\lambda H_d.
\end{equation}
Let us first consider \textit{monotonic losses} $\ell(d,\x,\y)$, which are monotonically increasing in the number of false positives ($FP$) and the number of false negatives ($FN$), and thus can be expressed alternatively as $\tilde{l}(FP, FN)$. We will specifically consider the following objectives in our implementation. (These are negated to become losses.)
\begin{itemize}[leftmargin=*, topsep=0pt, noitemsep]
    \item Accuracy = $1-\frac{FP+FN}{N}$: fraction of correct predictions. 
    \item Balanced accuracy = $1-\frac{1}{2}(\frac{FN}{N^+}+\frac{FP}{N^-})$: the average of true positive rate and true negative rate. Let $N^+$ be the number of positive samples in the training set and $N^-$ be the number of negatives. 
    \item Weighted accuracy = $1- \frac{FP+\omega FN}{\omega N^+ + N^-}$ for a predetermined threshold $\omega$: the cost-sensitive accuracy that penalizes more on predicting positive samples as negative.  
    \item F-score = $1-\frac{FP+FN}{2N^+ + FP-FN}$: the harmonic mean of precision and recall.
\end{itemize}
Optimizing F-score directly is difficult even for linear modeling, because it is non-convex \citep{nan2012optimizing}. 
In optimizing F-score for decision trees, the problem is worse -- a conundrum is possible where two leaves exist, the first leaf containing a higher proportion of positives than the other leaf, yet the first is classified as negative and the second classified as positive. We discuss how this can happen in Appendix \ref{app:Fscore} and how we address it, which is to force monotonicity by sweeping across leaves from highest to lowest predictions to calculate the F-score (see Appendix \ref{app:Fscore}).

We consider two objectives that are rank statistics: 
\begin{itemize}[leftmargin=*, topsep=0pt, noitemsep]
    \item Area under the ROC convex hull (AUC$_\ch$): the fraction of correctly ranked positive/negative pairs.
    \item Partial area under the ROC convex hull (pAUC$_\ch$) for predetermined threshold $\theta$: the area under the leftmost part of the ROC curve. 
\end{itemize}
 Some of the bounds from OSDT \citep{HuRuSe2019} have  straightforward extensions to the objectives listed above, namely the \textbf{Upper Bound on Number of Leaves} and \textbf{Leaf Permutation Bound}. The remainder of OSDT's bounds do not adapt. Our new bounds are the \textbf{Hierarchical Objective Lower Bound}, \textbf{Incremental Progress Bound to Determine Splitting}, \textbf{Lower Bound on Incremental Progress}, \textbf{Equivalent Points Bound}, \textbf{Similar Support Bound}, \textbf{Incremental Similar Support Bound}, and a \textbf{Subset Bound}. To focus our exposition, derivations and bounds for balanced classification loss, weighted classification loss, and F-score loss are in Appendix \ref{App:Objectives}, and derivations for AUC loss and partial AUC loss are in Appendix \ref{App: Obj_rank}, with the exception of the hierarchical lower bound for AUC$_{\ch}$, which appears in Section \ref{Sec:NewObjectives} to demonstrate how these bounds work. 

\subsection{Hierarchical Bound for AUC Optimization}\label{Sec:NewObjectives}

Let us discuss objectives that are rank statistics. If a classifier creates binary (as opposed to real-valued) predictions, its ROC curve consists of only three points (0,0), (FPR, TPR), and (1,1). The AUC of a labeled tree is the same as the balanced accuracy, because $\AUC = \frac{1}{2}(\frac{TP}{N^+}\times \frac{FP}{N^-}) + (1-\frac{FP}{N^-})\times \frac{TP}{N^+} + \frac{1}{2}( (1-\frac{TP}{N^+})\times(1-\frac{FP}{N^-}))$, and since $TP=N^+-FN$, we have $\AUC = \frac{1}{2}(\frac{N^+-FN}{N^+}\times \frac{FP}{N^-})+ (1-\frac{FP}{N^-})\times \frac{N^+-FN}{N^+} + \frac{1}{2} (\frac{FN}{N^+}\times\frac{N^--FP}{N^-}) = 1-\frac{1}{2}(\frac{FP}{N^-}+\frac{FN}{N^+})$. The more interesting case is when we have real-valued predictions for each leaf and use the ROC convex hull (ROCCH), defined shortly, as the objective.

Let $n_i^+$ be the number of positive samples in leaf $i$ ($n_i^-$ is the number of negatives) and let $r_i$ be the fraction of positives in leaf $i$. Let us define the area under the ROC convex hull (ROCCH) \citep{ferri2002} for a tree. For a tree $d$ consisting of $H_d$ distinct leaves, $d=(l_1,...,l_{H_d})$, we reorder leaves according to the fraction of positives, $r_1 \geq r_2 \geq ...\geq r_{H_d}$. For any $i=0,..., H_d$, define a labeling $S_i$ for the leaves that labels the first $i$ leaves as positive and remaining $H_d-i$ as negative. The collection of these labelings is
$\Gamma=S_0, S_1, ..., S_{H_d}$, where each $S_i$ defines one of the $H_d+1$ points on the ROCCH \citep[see e.g.,][]{ferri2002}. The associated misranking loss is then 1-AUC$_{\textrm{ch}}$:
\begin{equation}\label{eq:auc_loss}
    \ell(d, \boldsymbol{x}, \boldsymbol{y}) \h2=\h2 1\h2-\frac{1}{2N^+N^-}\sum\limits_{i=1}^{H}n_i^-\bigg[\bigg(\sum\limits_{j=1}^{i-1}2n_j^+\bigg)+n_i^+\bigg]\h2.
\end{equation}

Now let us derive a lower bound on the loss for trees that are incomplete, meaning that some parts of the tree are not yet fully grown.
For a partially-grown tree $d$, the leaf set can be rewritten as $d = (d_{\fix}, r_{\fix}, d_{\splitrm}, r_{\splitrm}, K, H_d)$, where $d_{\fix}$ is a set of $K$ fixed leaves that we choose not to split further and $d_{\splitrm}$ is the set of $H_d-K$ leaves that can be further split; this notation reflects how the algorithm works, where there are multiple copies of a tree, with some nodes allowed to be split and some that are not. $r_{\fix}$ and $r_{\splitrm}$ are fractions of positives in the leaves. If we have a new fixed $d'_{\fix}$, which is a superset of $d_{\fix}$, then we say $d'_{\fix}$ is a child of $d_{\fix}$. We define $\sigma(d)$ to be all such child trees:
\begin{equation}\label{eq:child}
    \sigma(d) = \{(d'_{\fix}, r'_{\fix}, d'_{\splitrm}, r'_{\splitrm}, K', H_d'): d_{\fix} \subseteq d'_{\fix}\}.
\end{equation}
Denote $N^+_{\splitrm}$ and $N^-_{\splitrm}$ as the number of positive and negative samples captured by $d_{\splitrm}$ respectively. Through additional splits, in the best case, $d_{\splitrm}$ can give rise to pure leaves, where positive ratios of generated leaves are either 1 or 0. Then the top-ranked leaf could contain up to $N^+_{\splitrm}$ positive samples (and 0 negative samples), and the lowest-ranked leaf could capture as few as 0 positive samples and up to $N^-_{\splitrm}$ samples. Working now with just the leaves in $d_{\fix}$, we reorder the leaves in $d_{\fix}$ by the positive ratios ($r_{\fix}$), such that $\forall i\in \{1,...,K\}, r_1 \geqslant r_2 \geqslant ... \geqslant r_K$. Combining these fixed leaves with the bounds for the split leaves, we can define a lower bound on the loss as follows.
\begin{theorem}\label{thm: lb_auc}(Lower bound for negative AUC convex hull)
For a tree $d= (d_{\fix}, r_{\fix}, d_{\splitrm}, r_{\splitrm}, K, H_d)$
using ${\textrm{\rm AUC}_{\rm{ch}}}$ as the objective, a lower bound on the loss is $b(d_{\fix}, \x, \y) \leqslant R(d,\x,\y)$, where:\vspace*{-4pt}
\begin{eqnarray}\nonumber
    b(d_{\fix},\x,\y) \h2\h2&\h2\h2=\h2\h2&\h2\h2 1-\frac{1}{2N^+N^-}\bigg (\sum\limits_{i=1}^K n_i^-\bigg[2N_{\splitrm}^+ + \bigg(\sum\limits_{j=1}^{i-1} 2n_j^+\bigg)\\
    \label{eq:auc_lb}
    && +n_i^+\bigg]+2N^+N_{\splitrm}^- \bigg) + \lambda H_d.
\end{eqnarray}
\end{theorem}
This leads directly to a hierarchical lower bound for the negative of the AUC convex hull.
\begin{theorem} \label{thm: hierarchy_auc}(Hierarchical objective lower bound for negative AUC convex hull)
Let $d = (d_{\fix}, r_{\fix}, d_{\splitrm}, r_{\splitrm}, K, H_d)$ be a tree with fixed leaves $d_{{\rm \fix}}$ and $d' = (d'_{{\rm \fix}}, r'_{\fix}, d'_{\splitrm}, r'_{\splitrm}, K', H_d') \in \sigma (d)$ be any child tree such that its fixed leaves $d'_{\fix}$ contain $d_{\fix}$, and $H_d' > H_d$, then $b(d_{{\rm \fix}},\x,\y) \leqslant R(d',\x,\y)$. 
\end{theorem}
This type of bound is the fundamental tool that we use to reduce the size of the search space: if we compare the lower bound $b(d_{{\rm \fix}},\x,\y)$ for partially constructed tree $d$ to the best current objective $R^{c}$, and find that $b(d_{{\rm \fix}},\x,\y)\geq R^{c}$, then there is no need to consider $d$ or any subtree of $d$, as it is provably non-optimal.  
The hierarchical lower bound dramatically reduces the size of the search space. However, we also have a collection of tighter bounds at our disposal, as summarized in the next subsection.

We leave the description of partial AUC to Appendix \ref{App: Obj_rank}. Given a parameter $\theta$, the partial AUC of the ROCCH focuses only on the left area of the curve, consisting of the top ranked leaves, whose FPR is smaller than or equal to $\theta$. This metric is used in information retrieval and healthcare.

\subsection{Summary of Bounds}\label{sec:summbounds}
Appendix \ref{App:Objectives} presents our bounds, which are crucial for reducing the search space.
 Appendix \ref{App:Objectives} presents the Hierarchical Lower Bound (Theorem \ref{thm:hierarchy}) for any objective (Equation \ref{eq:risk}) with an arbitrary monotonic loss function. This theorem is analogous to the Hierarchical Lower Bound for AUC optimization above.
Appendix \ref{App:Objectives} also contains the Objective Bound with One-Step Lookahead (Theorem \ref{thm:onestep}), Objective Bound for Sub-Trees (Theorem \ref{thm:hierarchy_subtree}), Upper Bound on the Number of Leaves (Theorem \ref{thm:ub_leaf_num}),
Parent-Specific Upper Bound on the Number of Leaves (Theorem \ref{thm:ub_leaf_num_ps}), Incremental Progress Bound to Determine Splitting
(Theorem \ref{thm:leaf_supp}), Lower Bound on Incremental Progress (Theorem \ref{thm:increm}), Leaf Permutation Bound (Theorem \ref{thm:perm}), Equivalent Points Bound (Theorem \ref{thm:equiv}), and General Similar Support Bound (Theorem \ref{thm:similar}). As discussed, no similar support bounds have been used successfully in prior work. In Section \ref{sec:SimSupp}, we show how a new \textit{Incremental Similar Support Bound} can be implemented within our specialized DPB (\textit{dynamic programming with bounds}) algorithm to make decision tree optimization for additive loss functions (e.g., weighted classification error) much more efficient. 
Bounds for AUC$_\ch$ and pAUC$_\ch$ are in Appendix \ref{App: Obj_rank}, including the powerful Equivalent Points Bound (Theorem \ref{thm:equiv_rank}) for AUC$_\ch$ and pAUC$_\ch$ and proofs for Theorem \ref{thm: lb_auc} and Theorem \ref{thm: hierarchy_auc}. In Appendix \ref{app:subset_bound}, we provide a new \textit{Subset Bound} implemented within DPB algorithm to effectively remove thresholds introduced by continuous variables.

Note that we do not use convex proxies for rank statistics, as is typically done in supervised ranking (learning-to-rank). Optimizing a convex proxy for a rank statistic can yield results that are far from optimal \citep[see][]{RudinWa18}. Instead, we optimize the original (exact) rank statistics directly on the training set, regularized by sparsity.


\section{Data Preprocessing Using Bucketization Sacrifices Optimality}\label{sec:DoesntWork}
As discussed, a preprocessing step common to DL8.5 \cite{nijssen2020} and BinOct \cite{verwer2019learning} reduces the search space, but as we will prove, also sacrifices accuracy; we refer to this preprocessing as \textit{bucketization}.

\noindent \textbf{Definition:} The \textit{bucketization} preprocessing step proceeds as follows. Order the observations according to any feature $j$. For any two neighboring positive observations, no split can be considered between them. For any two neighboring negative observations, no split can be considered between them. All other splits are permitted. While bucketization may appear innocuous, we prove it sacrifices optimality. 
\begin{theorem}
The maximum training accuracy for a decision tree on a dataset preprocessed with bucketization can be lower (worse) than the maximum accuracy for the same dataset without bucketization. 
\end{theorem}
The proof is by construction. In Figure \ref{fig:bucket1} we present a data set such that optimal training with bucketization cannot produce an optimal value of the objective. In particular, the optimal accuracy without bucketization is 93.5\%, whereas the accuracy of training with bucketization is 92.2\%. These numbers were obtained using BinOCT, DL8.5, and GOSDT. 
Remember, the algorithms provide a proof of optimality; these accuracy values are optimal, and the same values were found by all three algorithms. 

\begin{figure}
  \centering
  \includegraphics[width=0.4\linewidth]{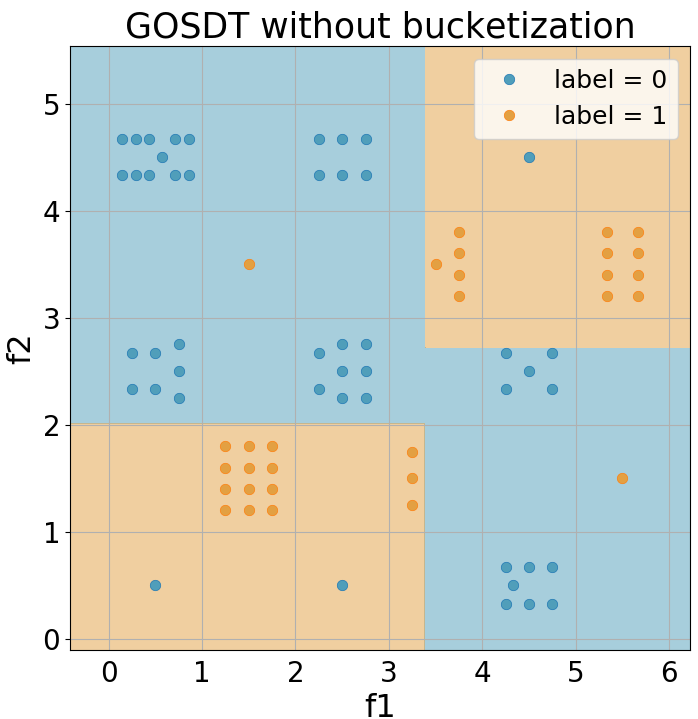}
  \includegraphics[width=0.4\linewidth]{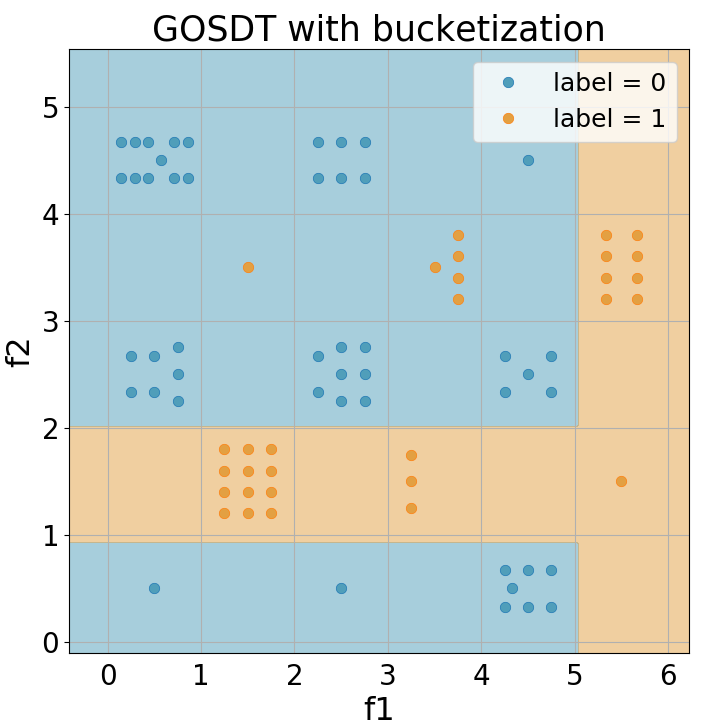}
  \caption{Proof by construction. Bucketization leads to suboptimality. The left plot's vertical split is not allowed in bucketization. 
  }\label{fig:bucket1}
\end{figure}  

Our dataset is two-dimensional. We expect the sacrifice to become worse with higher dimensions.

\section{GOSDT's DPB Algorithm} 
\label{Sec:DPB}
For optimizing non-additive loss function, we use PyGOSDT: a variant of GOSDT that is closer to OSDT \citep{HuRuSe2019}. For optimizing additive loss functions we use GOSDT which uses \textit{dynamic programming with bounds} (DPB) to provides a dramatic run time improvement. 
Like other dynamic programming approaches, we decompose a problem into smaller child problems that can be solved either recursively through a function call or in parallel by delegating work to a separate thread.
For decision tree optimization, these subproblems are the possible left and right branches of a decision node.

GOSDT maintains two primary data structures: a \emph{priority queue} to schedule problems to solve and a \emph{dependency graph} to store problems and their dependency relationships. We define a dependency relationship $dep(p_{\pi}, p_{c})$ between problems $p_{\pi}$ and $p_{c}$ if and only if the solution of $p_{\pi}$ depends on the solution of $p_{c}$. Each $p_c$ is further specified as $p_l^j$ or $p_r^j$ indicating that it is the left or right branch produced by splitting on feature $j$. 

Sections \ref{sec:support}, \ref{sec:SimSupp}, and \ref{sec:async} highlight two key differences between GOSDT and DL8.5 (since DL 8.5 also uses a form of dynamic programming):
(1) DL8.5 uniquely identifies a problem $p$ by the \textit{Boolean assertion} that is a conjunctive clause of all splitting conditions in its ancestry, while GOSDT represents a problem $p$ by the \textit{samples that satisfy the Boolean assertion}. This is described in Section \ref{sec:support}.
(2) DL8.5 uses blocking recursive invocations to execute problems, while GOSDT uses a \emph{priority queue} to schedule problems for later. This is described in Section \ref{sec:async}.
Section \ref{sec:vect} presents additional performance optimizations. Section \ref{sec:alg} presents a high-level summary of GOSDT. Note that GOSDT's DPB algorithm operates on weighted, additive, non-negative loss functions. For ease of notation, we use classification error as the loss in our exposition.

\subsection{Support Set Identification of Nodes}\label{sec:support}
GOSDT leverages the Equivalent Points Bound in Theorem \ref{thm:equiv} to save space and computational time.
To use this bound, we find the unique values of $\{x_1,...,x_N\}$, denoted by $\{z_1,...,z_U\}$, so that each $x_i$ equals one of the $z_u$'s. We store fractions $z_{u}^+$ and $z_{u}^-$, which are the fraction of positive and negative samples corresponding to each $z_u$.
\begin{eqnarray*}
    Z &=& \{z_u:z_u \in \textrm{unique}(X), 1 \leq u \leq U \leq N\}
    \\
    z_u^- &=& \frac{1}{N}\sum_{i=1}^{N}\mathds{1}[y_i = 0 \land x_i = z_u]
    \\ 
    z_u^+ &=& \frac{1}{N}\sum_{i=1}^{N}\mathds{1}[y_i = 1 \land x_i = z_u].
\end{eqnarray*}
Define $a$ as a Boolean assertion that is a conjunctive clause of conditions on the features of $z_u$ (e.g., $a$ is true if and only if the first feature of $z_u$ is 1 and the second feature of $z_u$ is 0). We define \textit{support set} $s_a$ as the set of $z_u$ that satisfy assertion $a$:
\begin{equation}
    s_a = \{ z_u:a(z_u)=\textrm{True}, 1\leq u\leq U \}.
\end{equation}
We implement each $s_a$ as a bit-vector of length $U$ and use it to \textit{uniquely identify} a \textit{problem} $p$ in the dependency graph. The bits $\{s_{au}\}_u$ of $s_a$ are defined as follows.
\begin{equation}
    s_{au} = 1 \iff z_u \in s_a.
\end{equation}
With each $p$, we track a set of values including its \textit{lower bound} ($lb$) and \textit{upper bound} ($ub$) of the optimal objective classifying its support set $s_a$.
 
In contrast, DL8.5 identifies each problem $p$ using Boolean assertion $a$ rather than $s_a$: this is an important difference between GOSDT and DL8.5 because many assertions $a$ could correspond to the same support set $s_a$. However, given the same objective and algorithm, an optimal decision tree for any problem $p$ depends only on the support set $s_a$. It does not depend on $a$ if one already knows $s_a$. That is, if two different trees both contain a leaf capturing the same set of samples $s$, the set of possible child trees for that leaf is the same for the two trees.
GOSDT solves problems identified by $s$. This way, it simultaneously solves all problems identified by any $a$ that produces $s$. 

\subsection{Connection to Similar Support Bound}
\label{sec:SimSupp}
No previous approach has been able to implement a similar support bound effectively. We provide GOSDT's new form of similar support bound, called the \textit{incremental similar support bound}. There are two reasons why this bound works where other attempts have failed:
(1) This bound works on \textit{partial trees}. One does not need to have constructed a full tree to use it.
(2) The bound takes into account the hierarchical objective lower bound, and hence leverages the way we search the space within the DBP algorithm. In brief, the bound effectively removes many similar trees from the search region by looking at only one of them (which does not need to be a full tree). 
We consider weighted, additive, non-negative loss functions: $\ell(d,\x,\y)=\sum_i \textrm{weight}_i \textrm{loss}(x_i,y_i)$. Define the maximum weighted loss: $\ell^{\max}=\max_{x,y} [\textrm{\rm weight}(x,y) \times \textrm{loss}(x,y)].$

\begin{theorem}\label{theorem:similar_support} (Incremental Similar Support Bound) 
Consider two trees $d=(d_{\fix}, d_{\splitrm}, K, H)$ and $D=(D_{\fix}, D_{\splitrm}, K, H)$ that differ only by their root node (hence they share the same $K$ and $H$ values). Further, the root nodes between the two trees are similar enough that the support going to the left and right branches differ by at most $\omega$ fraction of the observations. (That is, there are $\omega N$ observations that are captured either by the left branch of $d$ and right branch of $D$ or vice versa.)  Define $S_{\uncertain}$ as the maximum of the support within $d_{\splitrm}$ and $D_{\splitrm}$:
$
S_{\uncertain}=\max(\supp(d_{\splitrm}),\supp(D_{\splitrm})).
$
For any child tree $d' \in \sigma(d)$ grown from $d$ (grown from the nodes in $d_{\splitrm}$, that would not be excluded by the hierarchical objective lower bound) and for any child tree $D' \in \sigma(D)$ grown from $D$ (grown from nodes in $D_{\splitrm}$, not excluded by the hierarchical objective lower bound), we have:
\[
|R(d',\x,\y)-R(D',\x,\y)|\leq (\omega +2S_{\uncertain}) \ell^{\max}. 
\]
\end{theorem}

The proof is in Appendix \ref{app:SimSupp}.
 Unlike the similar support bounds in CORELS and OSDT, which require pairwise comparisons of problems, this incremental similar support bound emerges from the support set representation. The descendants $\sigma(d)$ and $\sigma(D)$ share many of the same support sets. Because of this, the shared components of their upper and lower bounds are updated simultaneously (to similar values). This bound is helpful when our data contains continuous variables: if a split at value $v$ was already visited, then splits at values close to $v$ can reuse prior computation to immediately produce tight upper and lower bounds.

\subsection{Asynchronous Bound Updates}\label{sec:async}
GOSDT computes the objective values hierarchically by defining the minimum objective $R^*(p)$ of problem $p$ as an aggregation of minimum objectives over the child problems $p_l^j$ and $p_r^j$ of $p$ for $1 \leq j \leq M$.
\begin{equation}\label{eq:async}
    R^*(p) = \min_j(R^*(p_l^j) + R^*(p_r^j)).
\end{equation}
Since DL8.5 computes each $R^*(p_l^j)$ and $R^*(p_r^j)$ in a blocking call to the the child problems $p_l^j$ and $p_r^j$, it necessarily computes $R^*(p_l^j) + R^*(p_r^j)$ after solving $p_l^j$ and $p_r^j$, which is a disadvantage. In contrast, GOSDT computes bounds over $R^*(p_l^j) + R^*(p_r^j)$ that are available before knowing the exact values of $R^*(p_l^j)$ and $R^*(p_r^j)$. The bounds over $R^*(p_l^j)$ and $R^*(p_r^j)$ are solved asynchronously and possibly in parallel.
If for some $j$ and $j'$ the bounds imply that $R^*(p_l^j) + R^*(p_r^j) > R^*(p_l^{j'}) + R^*(p_r^{j'})$, then we can conclude that $p$'s solution no longer depends on $p_l^{j}$ and $p_r^{j}$. Since GOSDT executes asynchronously, it can draw this conclusion and focus on $R^*(p_l^{j'}) + R^*(p_r^{j'})$ without fully solving $R^*(p_l^j) + R^*(p_r^j)$.

To encourage this type of bound update, GOSDT uses the priority queue to send high-priority signals to each parent $p_\pi$ of $p$ when an update is available, prompting a recalculation of $R^*(p_\pi)$ using Equation \ref{eq:async}.

\subsection{Fast Selective Vector Sums}
\label{sec:vect}
New problems (i.e., $p_{l}$ and $p_{r}$) require initial upper and lower bounds on the optimal objective. We define the initial lower bound $lb$ and upper bound $ub$ for a problem $p$ identified by support set $s$ as follows. For $1 \leq u \leq U$, define:
\begin{equation}
    z_u^{\min} = \min(z_u^-,z_u^+). 
    \end{equation}
    This is the fraction of minority class samples in equivalence class $u$. Then,
    \begin{equation}
    lb = 2\lambda + \sum_u s_u z_u^{\min}. \label{eq:lb0}
\end{equation}
    This is a basic equivalence points bound, which predicts all minority class equivalence points incorrectly. Also,
\begin{equation}
    ub = \lambda + \min\left(\sum_u s_u z_u^-,\sum_u s_u z_u^+\right). \label{eq:ub0}
\end{equation}
This upper bound comes from a baseline algorithm of predicting all one class.

We use the \textit{prefix sum trick} in order to speed up computations of sums of a subset of elements of a vector. That is,
for any vector $z^{\textrm{vec}}$ (e.g., $z^{\min}$, $z^-$, $z^+$), we want to compute a sum of a subsequence of $z^{\textrm{vec}}$:
we precompute (during preprocessing) a prefix sum vector $z^{\cumulative}$ of the vector $z^{\textrm{vec}}$ defined by the cumulative sum $z^{\cumulative}_u = \sum_{j=1}^{u}z^{\textrm{vec}}_j$.
During the algorithm, 
to sum over ranges of contiguous values in $z^{\textrm{vec}}$, over some indices $a$ through $b$, we now need only take $z^{\cumulative}[b] - z^{\cumulative}[a-1]$. This reduces a linear time calculation to constant time.
This fast sum is leveraged over calculations with the support sets of input features--for example, quickly determining the difference in support sets between two features.


\subsection{The GOSDT Algorithm}
\label{sec:alg}
Algorithm \ref{alg:gosdt_summary} constructs and optimizes problems in the \textit{dependency graph} such that, upon completion, we can extract the optimal tree by traversing the dependency graph by greedily choosing the split with the lowest objective value. This extraction algorithm and a more detailed version of the GOSDT algorithm is provided in Appendix \ref{App:algorithm}.
We present the key components of this algorithm, highlighting the differences between GOSDT and DL8.5. Note that all features have been binarized prior to executing the algorithm.


\textbf{Lines 8 to 11}: Remove an item from the queue. If its bounds are equal, no further optimization is possible and we can proceed to the next item on the queue.

\textbf{Lines 12 to 18}: Construct new subproblems, $p_l$, $p_r$ by splitting on feature $j$. Use lower and upper bounds of $p_l$ and $p_r$ to compute new bounds for $p$. \textit{We key the problems by the bit vector corresponding to their support set $s$}. Keying problems in this way avoids processing the same problem twice; other dynamic programming implementations, such as DL8.5, will process the same problem multiple times.

\textbf{Lines 19 to 23}: Update $p$ with $lb'$ and $ub'$ computed using Equation \ref{eq:async} and propagate that update to all ancestors of $p$ by enqueueing them with high priority ($p$ will have multiple parents if two different conjunctions produce the same support set). This triggers the ancestor problems to recompute their bounds using Equation \ref{eq:async}. \textit{The high priority ensures that ancestral updates occur before processing $p_l$ and $p_r$.} This scheduling is one of the key differences between GOSDT and DL8.5; by eagerly propagating bounds up the dependency tree, GOSDT prunes the search space more aggressively.

\textbf{Lines 26 to 32}: Enqueue $p_l$ and $p_r$ only if the interval between their lower and upper bounds overlaps with the interval between $p$'s lower and upper bounds. This ensures that eventually the lower and upper bounds of $p$ converge to a single value.

\textbf{Lines 34 to 42}: Define  \textbf{FIND\_OR\_CREATE\_NODE}, which constructs (or finds an existing) problem $p$ corresponding to support set $s$, initializing the lower and upper bound and checking if $p$ is eligible to be split, using the Incremental Progress Bound to Determine Splitting (Theorem \ref{thm:leaf_supp}) and the Lower Bound on Incremental Progress (Theorem \ref{thm:increm}) (these bounds are checked in subroutine \textit{fails\_bounds}). \textit{get\_lower\_bound} returns $lb$ using Equation \ref{eq:lb0}. \textit{get\_upper\_bound} returns $ub$ using Equation \ref{eq:ub0}.

An implementation of the algorithm is available at: \url{https://github.com/Jimmy-Lin/GeneralizedOptimalSparseDecisionTrees}.

\definecolor{comment}{rgb}{0.2,0.4,0.2}
\def\comment#1{\textcolor{comment}{\textit{ // #1 }}}

\begin{algorithm}[H]
\caption{GOSDT$(R, x, y, \lambda)$ \label{alg:gosdt_summary}}
\begin{tabbing}
xxx \= xx \= xx \= xx \= xx \= xx \kill
1: \> \textbf{input:} $R$, $Z$, $z^-$, $z^+$, $\lambda$ \comment{risk, samples, regularizer} \\
2: \> $Q = \emptyset$ \comment{priority queue}\\
3: \> $G=\emptyset$ \comment{dependency graph}\\
4: \> $s_0 \leftarrow \{1,...,1\}$\comment{bit-vector of 1's of length $U$} \\
5: \> $p_0 \leftarrow$ FIND\_OR\_CREATE\_NODE($G, s_0$)\comment{root}\\
6: \> $Q.{\rm push}(s_0)$\comment{add to priority queue}\\
7: \> \textbf{while} $p_0.lb \neq p_0.ub$ \textbf{do}\\
8: \> \> $s\leftarrow Q.{\rm pop}()$\comment{index of problem to work on}\\
9: \> \> $p\leftarrow G.{\rm find}(s)$\comment{find problem to work on}\\
10: \> \> \textbf{if} $p.lb=p.ub$ \textbf{then}\\
11: \> \> \> \textbf{continue}\comment{problem already solved} \\
12: \> \> $(lb', ub') \leftarrow (\infty, \infty)$\comment{loose starting bounds}\\
13: \> \> \textbf{for} each feature $j \in [1,M]$ \textbf{do}\\
14: \> \> \> $s_l, s_r \leftarrow \text{split}(s,j,Z)$ \comment{create children} \\
15: \> \> \> $p_l^j\leftarrow$FIND\_OR\_CREATE\_NODE$(G,s_l)$\\
16: \> \> \> $p_r^j\leftarrow$FIND\_OR\_CREATE\_NODE$(G,s_r)$\\
\>\>\comment{create bounds as if $j$ were chosen for splitting}
\\
17: \> \> \> $lb' \leftarrow \min(lb', p_l^j.lb + p_r^j.lb)$ \\
18: \> \> \> $ub' \leftarrow \min(ub', p_l^j.ub + p_r^j.ub)$ \\
\> \> \comment{signal the parents if an update occurred} \\
19: \> \> \textbf{if} $p.lb \neq lb'$ \textbf{or} $p.ub \neq ub'$  \textbf{then} \\
20: \> \> \> $p.ub \leftarrow \min(p.ub, ub')$\\
21: \> \> \> $p.lb \leftarrow \min(p.ub, \max(p.lb, lb'))$\\
22: \> \> \> \textbf{for} $p_{\pi} \in G.{\rm parent}(p)$ \textbf{do} \\
\> \> \> \> \comment{propagate information upwards}\\
23: \> \> \> \> $Q.{\rm push}(p_{\pi}.{\rm id}, {\rm priority}=1)$\\
24: \> \> \textbf{if} $p.lb = p.ub$ \textbf{then} \\
25: \> \> \> \textbf{continue} \comment{problem solved just now} \\
\> \> \comment{loop, enqueue all children} \\

26: \> \> \textbf{for} each feature $j \in [1,M]$ \textbf{do} \\
\> \> \comment{fetch $p_l^j$ and $p_r^j$ in case of update}\\
27: \> \> \> repeat line 14-16\\
28: \> \> \> $lb' \leftarrow p_l^j.lb + p_r^j.lb$ \\
29: \> \> \> $ub' \leftarrow p_l^j.ub + p_r^j.ub$ \\
30: \> \> \> \textbf{if} $lb' < ub'$ \textbf{and} $lb' \le p.ub$ \textbf{then} \\
31: \> \> \> \> $Q.{\rm push}(s_{l}, {\rm priority}=0)$ \\
32: \> \> \> \> $Q.{\rm push}(s_{r}, {\rm priority}=0)$ \\
33: \> \textbf{return}\\
--------------------------------------------------------------------\\
34: \> \textbf{subroutine} FIND\_OR\_CREATE\_NODE(G,s)\\
35: \> \> \textbf{if} $G.{\rm find}(s) = {\rm NULL}$\comment{$p$ not yet in graph}\\
36: \> \> \> $p.id \leftarrow s$ \comment{identify $p$ by $s$}\\
37: \> \> \> $p.lb \leftarrow {\rm get\_lower\_bound}(s,Z,z^-,z^+)$\\
38: \> \> \> $p.ub \leftarrow {\rm get\_upper\_bound}(s,Z,z^-,z^+)$\\
39: \> \> \> \textbf{if} fails\_bounds$(p)$ \textbf{then}\\
40: \> \> \> \> $p.lb=p.ub$ \comment{no more splitting allowed}\\
41: \> \> \> G.insert(p) \comment{put $p$ in dependency graph}\\
42: \> \> \textbf{return} G.find(s)
\end{tabbing}
\end{algorithm}

\section{Experiments}
We present details of our experimental setup and datasets in Appendix \ref{App:Experiments}.
GOSDT's novelty lies in its ability to optimize a large class of objective functions and its ability to efficiently handle continuous variables without sacrificing optimality. Thus, our evaluation results: 1) Demonstrate our ability to optimize over a large class of objectives (AUC in particular), 2) Show that GOSDT outperforms other approaches in producing models that are both accurate and sparse, and 3) Show how GOSDT scales in its handling of continuous variables relative to other methods.
In Appendix \ref{app:exp} we show time-to-optimality results.





\textbf{Optimizing Many Different Objectives}: 
We use the FourClass dataset \citep{libsvm} to show optimal decision trees corresponding to different objectives. Figures \ref{fig:roc_train} and \ref{fig:tree_fc} show the training ROC of decision trees generated for six different objectives and optimal trees for accuracy, AUC and partial area under ROC convex hull. Optimizing different objectives produces different trees with different $FP$ and $FN$. (No other decision tree method is designed to fully optimize any objective except accuracy, so there is no comparison to other methods.)

\begin{figure}[h]
    \centering
    \includegraphics[width=1.0\linewidth, scale=1.0]{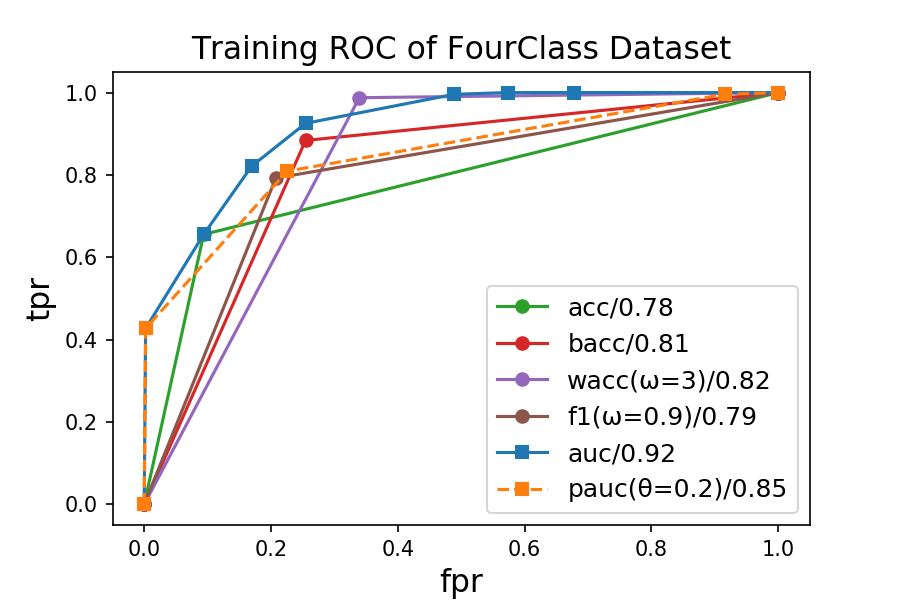}
    \caption{Training ROC of FourClass dataset ($\lambda=0.01$). A/B in the legend at the bottom right shows the objective and its parameters/area under the ROC.}
    \label{fig:roc_train}
    \vspace{0.2cm}
    \includegraphics[width=1\linewidth]{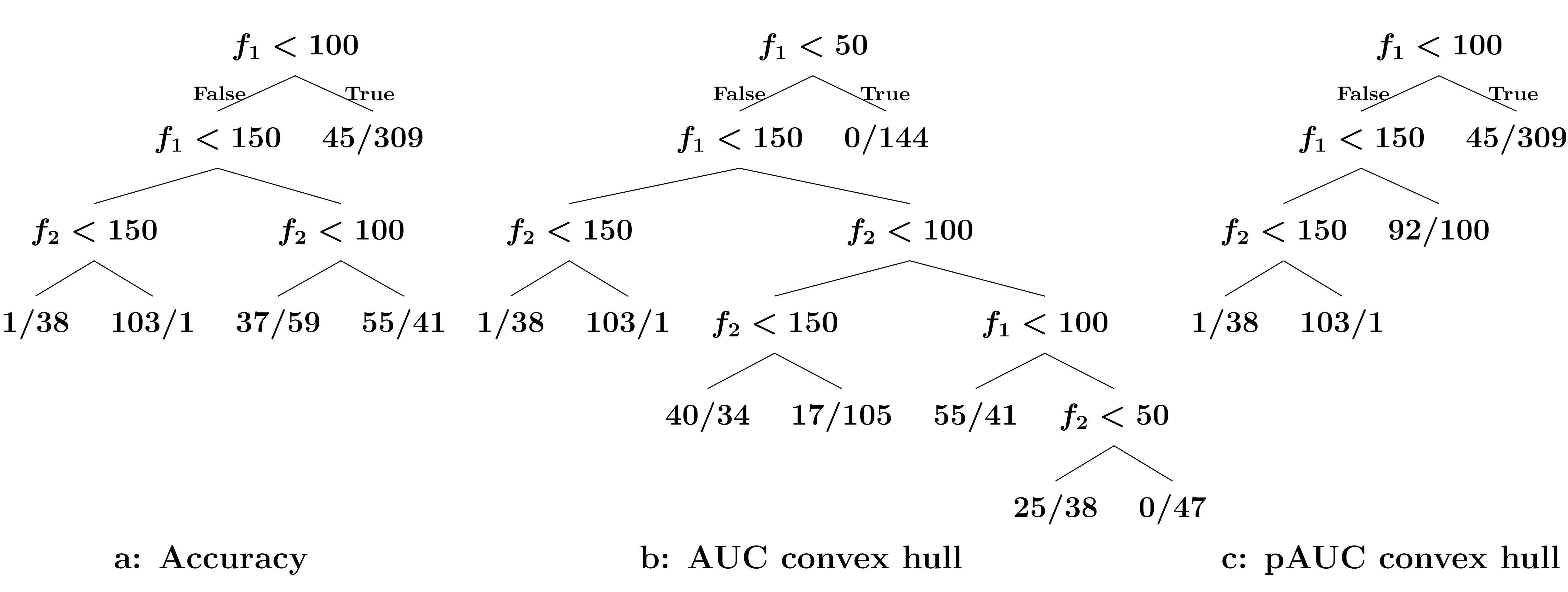}
    \caption{Trees for different objectives. A/B in the leaf node represents the number of positive/negative samples amongst the total samples captured in that leaf.}
    \label{fig:tree_fc}
\end{figure}
\begin{figure}[h]
  \centering
  \includegraphics[scale=0.4]{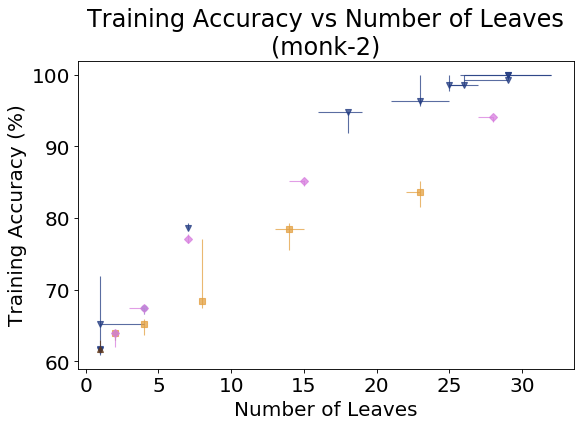}
  \includegraphics[scale=0.4]{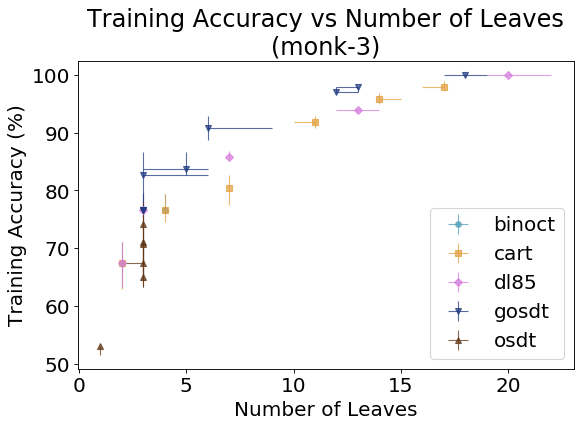}
  \caption{Training accuracy achieved by BinOCT, CART, DL8.5, GOSDT, and OSDT as a function of the number of leaves.  }
  \label{fig:train_acc}
\end{figure}
\begin{figure}[h]
  \centering
  \includegraphics[scale=0.4]{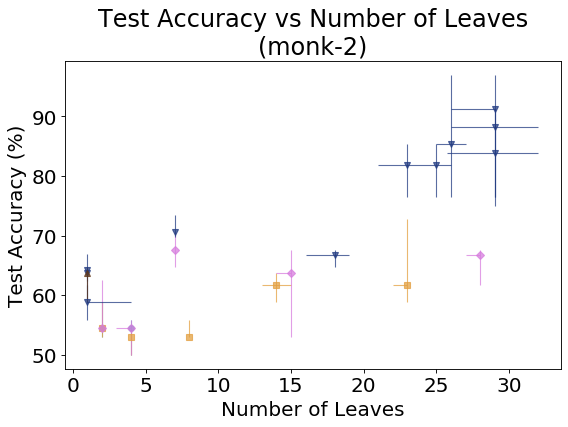}
  \includegraphics[scale=0.4]{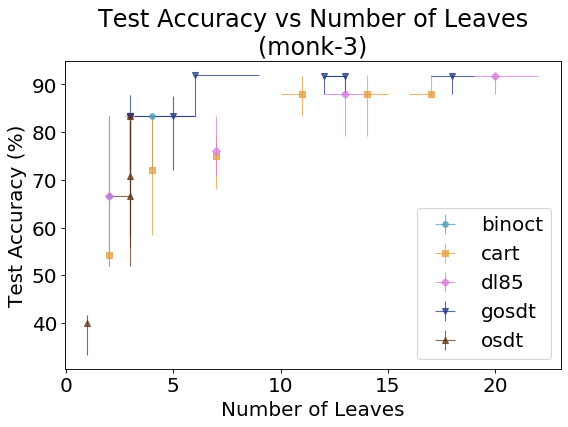}
  \caption{Test accuracy achieved by BinOCT, CART, DL8.5, GOSDT, and OSDT as a function of the number of leaves.  }
  \label{fig:test_acc}
\end{figure}
\begin{figure}[h]
  \centering
  \includegraphics[scale=0.4]{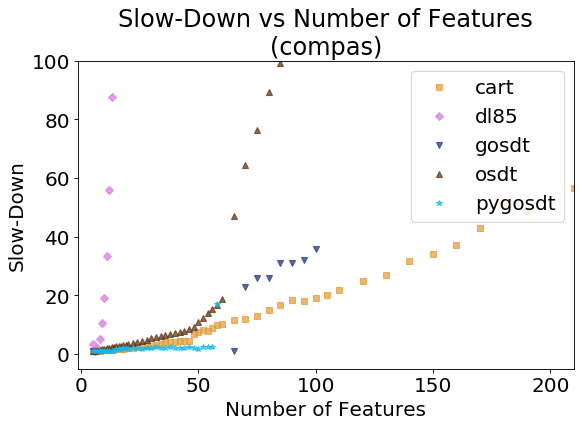}
  \includegraphics[scale=0.4]{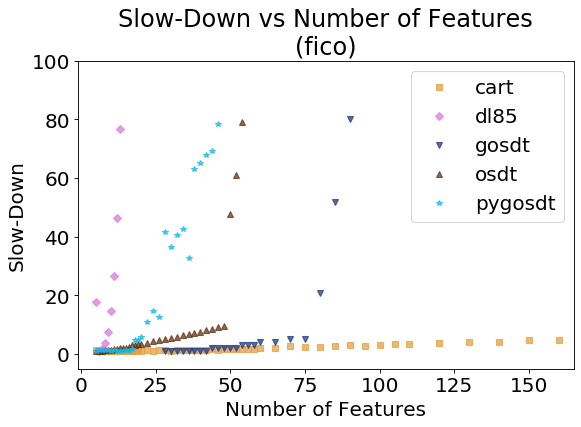}
  \caption{Training time of BinOCT, CART, DL8.5, GOSDT, and OSDT as a function of the number of binary features used to encode the continuous dataset ($\lambda$ = 0.3125 or max depth = 5).}
  \label{fig:time_feature}
\end{figure}

\textbf{Binary Datasets, Accuracy vs Sparsity:} 
We compare models produced from BinOCT \cite{verwer2019learning}, CART \citep{Breiman84}, DL8.5 \cite{nijssen2020}, OSDT \cite{HuRuSe2019}, and GOSDT. For each method, we select hyperparameters to produce trees of varying numbers of leaves and plot training accuracy against sparsity (number of leaves).
GOSDT directly optimizes the trade-off between training accuracy and number of leaves, producing points on the efficient frontier. Figure \ref{fig:train_acc} and \ref{fig:test_acc} show (1) that GOSDT typically produces excellent training and test accuracy with a reasonable number of leaves, and (2) that we can quantify how close to optimal the other methods are. Learning theory provides guarantees that training and test accuracy are close for sparse trees; 
more results are in Appendix \ref{app:exp}.

\textbf{Continuous Datasets, Slowdown vs Thresholds:} 
We preprocessed by encoding continuous variables as a set of binary variables, using all possible thresholds. We then reduced the number of binary variables by combining sets of $k$ variables for increasing values of $k$. Binary variables were ordered, firstly, in order of the indices of the continuous variable to which they correspond, and secondly, in order of their thresholds.
We measure the slow-down introduced by increasing the number of binary variables relative to each algorithm's own best time. 
Figure \ref{fig:time_feature} shows these results for CART, DL8.5, GOSDT, and OSDT. As the number of features increases (i.e., $k$ approaches 1), GOSDT typically slows down less than DL8.5 and OSDT. This relatively smaller slowdown allows GOSDT to handle more thresholds introduced by continuous variables.

Appendix \ref{App:Experiments} presents more results including training times several orders of magnitude better than the state-of-the-art.

An implementation of our experiments is available at:  \url{https://github.com/Jimmy-Lin/TreeBenchmark}.


\section{Discussion and Future Work}
GOSDT (and related methods) differs fundamentally from other types of machine learning algorithms. Unlike neural networks and typical decision tree methods, it provides a proof of optimality for highly non-convex problems. Unlike support vector machines, ensemble methods, and neural networks again, it produces sparse interpretable models without using convex proxies--it solves the exact problem of interest in an efficient and provably optimal way. As usual, statistical learning theoretic guarantees on test error are the tightest for simpler models (sparse models) with the lowest empirical error on the training set, hence the choice of accuracy, regularized by sparsity (simplicity). If GOSDT is stopped early, it reports an optimality gap, which allows the user to assess whether the tree is sufficiently close to optimal, retaining learning theoretic guarantees on test performance. GOSDT is most effective for datasets with a small or medium number of features. The algorithm scales well with the number of observations, easily handling tens of thousands.

There are many avenues for future work. Since GOSDT provides a framework to handle objectives that are monotonic in $FP$ and $FN$, one could create many more objectives than we have enumerated here. Going beyond these objectives to handle other types of monotonicity, fairness, ease-of-use, and cost-related soft and hard constraints are natural extensions. There are many avenues to speed up GOSDT related to exploration of the search space, garbage collection, and further bounds.


\bibliography{refs}

\begin{thebibliography}{35}
\providecommand{\natexlab}[1]{#1}
\providecommand{\url}[1]{\texttt{#1}}
\expandafter\ifx\csname urlstyle\endcsname\relax
  \providecommand{\doi}[1]{doi: #1}\else
  \providecommand{\doi}{doi: \begingroup \urlstyle{rm}\Url}\fi

\bibitem[{A{\"\i}vodji} et~al.(2019){A{\"\i}vodji}, {Ferry}, {Gambs}, {Huguet},
  and {Siala}]{FairCORELS}
{A{\"\i}vodji}, U., {Ferry}, J., {Gambs}, S., {Huguet}, M.-J., and {Siala}, M.
\newblock {Learning Fair Rule Lists}.
\newblock \emph{arXiv e-prints}, pp.\  arXiv:1909.03977, Sep 2019.

\bibitem[Angelino et~al.(2017)Angelino, Larus-Stone, Alabi, Seltzer, and
  Rudin]{AngelinoLaAlSeRu17-kdd}
Angelino, E., Larus-Stone, N., Alabi, D., Seltzer, M., and Rudin, C.
\newblock Learning certifiably optimal rule lists for categorical data.
\newblock In \emph{Proc. {ACM} {SIGKDD} International Conference on Knowledge
  Discovery and Data Mining {(KDD)}}, 2017.

\bibitem[Angelino et~al.(2018)Angelino, Larus-Stone, Alabi, Seltzer, and
  Rudin]{AngelinoEtAl18}
Angelino, E., Larus-Stone, N., Alabi, D., Seltzer, M., and Rudin, C.
\newblock Learning certifiably optimal rule lists for categorical data.
\newblock \emph{Journal of Machine Learning Research}, 18\penalty0
  (234):\penalty0 1--78, 2018.

\bibitem[Bennett(1992)]{Bennett92}
Bennett, K.
\newblock Decision tree construction via linear programming.
\newblock In \emph{Proceedings of the 4th Midwest Artificial Intelligence and
  Cognitive Science Society Conference, Utica, Illinois}, 1992.

\bibitem[Bennett \& Blue(1996)Bennett and Blue]{Bennett96optimaldecision}
Bennett, K.~P. and Blue, J.~A.
\newblock Optimal decision trees.
\newblock Technical report, R.P.I. Math Report No. 214, Rensselaer Polytechnic
  Institute, 1996.

\bibitem[Bertsimas \& Dunn(2017)Bertsimas and Dunn]{bertsimas2017optimal}
Bertsimas, D. and Dunn, J.
\newblock Optimal classification trees.
\newblock \emph{Machine Learning}, 106\penalty0 (7):\penalty0 1039--1082, 2017.

\bibitem[Blanquero et~al.(2018)Blanquero, Carrizosa, Molero-R{\i}o, and
  Morales]{molerooptimal}
Blanquero, R., Carrizosa, E., Molero-R{\i}o, C., and Morales, D.~R.
\newblock Optimal randomized classification trees.
\newblock August 2018.

\bibitem[Breiman et~al.(1984)Breiman, Friedman, Olshen, and Stone]{Breiman84}
Breiman, L., Friedman, J.~H., Olshen, R.~A., and Stone, C.~J.
\newblock \emph{Classification and Regression Trees}.
\newblock Wadsworth, 1984.

\bibitem[Chang \& Lin(2011)Chang and Lin]{libsvm}
Chang, C.-C. and Lin, C.-J.
\newblock Libsvm: a library for support vector machines, 2011.
\newblock software available at \url{http://www.csie.ntu.edu.tw/~cjlin/libsvm}.

\bibitem[Chen \& Rudin(2018)Chen and Rudin]{ChenRu18}
Chen, C. and Rudin, C.
\newblock An optimization approach to learning falling rule lists.
\newblock In \emph{International Conference on Artificial Intelligence and
  Statistics {(AISTATS)}}, 2018.

\bibitem[Dheeru \& Karra~Taniskidou(2017)Dheeru and Karra~Taniskidou]{Dua:2017}
Dheeru, D. and Karra~Taniskidou, E.
\newblock {UCI} machine learning repository, 2017.
\newblock URL \url{http://archive.ics.uci.edu/ml}.

\bibitem[Ferri et~al.(2002)Ferri, Flach, and Hern{\'a}ndez-Orallo]{ferri2002}
Ferri, C., Flach, P., and Hern{\'a}ndez-Orallo, J.
\newblock Learning decision trees using the area under the {ROC} curve.
\newblock In \emph{International Conference on Machine Learning {(ICML)}},
  volume~2, pp.\  139--146, 2002.

\bibitem[{FICO} et~al.(2018){FICO}, {Google}, {Imperial College London}, {MIT},
  {University of Oxford}, {UC Irvine}, and {UC Berkeley}]{competition}
{FICO}, {Google}, {Imperial College London}, {MIT}, {University of Oxford}, {UC
  Irvine}, and {UC Berkeley}.
\newblock {Explainable Machine Learning Challenge}.
\newblock
  \url{https://community.fico.com/s/explainable-machine-learning-challenge},
  2018.

\bibitem[Garey(1972)]{Garey72}
Garey, M.
\newblock Optimal binary identification procedures.
\newblock \emph{{SIAM} J. Appl. Math.}, 23\penalty0 (2):\penalty0 173--–186,
  September 1972.

\bibitem[Garofalakis et~al.(2003)Garofalakis, Hyun, Rastogi, and
  Shim]{GarofalakisHyRaSh03}
Garofalakis, M., Hyun, D., Rastogi, R., and Shim, K.
\newblock Building decision trees with constraints.
\newblock \emph{Data Mining and Knowledge Discovery}, 7:\penalty0 187--214,
  2003.

\bibitem[Hu et~al.(2019)Hu, Rudin, and Seltzer]{HuRuSe2019}
Hu, X., Rudin, C., and Seltzer, M.
\newblock Optimal sparse decision trees.
\newblock In \emph{Proceedings of Neural Information Processing Systems
  {(NeurIPS)}}, 2019.

\bibitem[Klivans \& Servedio(2006)Klivans and Servedio]{Klivans06}
Klivans, A.~R. and Servedio, R.~A.
\newblock Toward attribute efficient learning of decision lists and parities.
\newblock \emph{Journal of Machine Learning Research}, 7:\penalty0 587--602,
  2006.

\bibitem[Larson et~al.(2016)Larson, Mattu, Kirchner, and
  Angwin]{LarsonMaKiAn16}
Larson, J., Mattu, S., Kirchner, L., and Angwin, J.
\newblock How we analyzed the {COMPAS} recidivism algorithm.
\newblock \emph{ProPublica}, 2016.

\bibitem[Larus-Stone(2017)]{Larus-Stone17}
Larus-Stone, N.~L.
\newblock \textit{Learning Certifiably Optimal Rule Lists: {A} Case For
  Discrete Optimization in the 21st Century}.
\newblock 2017.
\newblock Undergraduate thesis, Harvard College.

\bibitem[Laurent \& Rivest(1976)Laurent and Rivest]{laurent1976constructing}
Laurent, H. and Rivest, R.~L.
\newblock Constructing optimal binary decision trees is np-complete.
\newblock \emph{Information processing letters}, 5\penalty0 (1):\penalty0
  15--17, 1976.

\bibitem[Loh(2014)]{Loh14}
Loh, W.-Y.
\newblock Fifty years of classification and regression trees.
\newblock \emph{International Statistical Review}, 82\penalty0 (3):\penalty0
  329--348, 2014.

\bibitem[Meisel \& Michalopoulos(1973)Meisel and Michalopoulos]{MeiselMi73}
Meisel, W.~S. and Michalopoulos, D.
\newblock A partitioning algorithm with application in pattern classification
  and the optimization of decision tree.
\newblock \emph{{IEEE} Trans. Comput., C-22}, pp.\  93--103, 1973.

\bibitem[Menickelly et~al.(2018)Menickelly, G{\"{u}}nl{\"{u}}k, Kalagnanam, and
  Scheinberg]{MenickellyGKS18}
Menickelly, M., G{\"{u}}nl{\"{u}}k, O., Kalagnanam, J., and Scheinberg, K.
\newblock Optimal decision trees for categorical data via integer programming.
\newblock \emph{Preprint at arXiv:1612.03225}, January 2018.

\bibitem[Morgan \& Sonquist(1963)Morgan and Sonquist]{MorganSo1963}
Morgan, J.~N. and Sonquist, J.~A.
\newblock Problems in the analysis of survey data, and a proposal.
\newblock \emph{J. Amer. Statist. Assoc.}, 58:\penalty0 415–434, 1963.

\bibitem[Nan et~al.(2012)Nan, Chai, Lee, and Chieu]{nan2012optimizing}
Nan, Y., Chai, K.~M., Lee, W.~S., and Chieu, H.~L.
\newblock Optimizing {F}-measure: A tale of two approaches.
\newblock \emph{arXiv preprint arXiv:1206.4625}, 2012.

\bibitem[Narodytska et~al.(2018)Narodytska, Ignatiev, Pereira, and
  Marques-Silva]{narodytska2018learning}
Narodytska, N., Ignatiev, A., Pereira, F., and Marques-Silva, J.
\newblock Learning optimal decision trees with {SAT}.
\newblock In \emph{Proc. International Joint Conferences on Artificial
  Intelligence {(IJCAI)}}, pp.\  1362--1368, 2018.

\bibitem[Nijssen \& Fromont(2007)Nijssen and Fromont]{nijssen2007mining}
Nijssen, S. and Fromont, E.
\newblock Mining optimal decision trees from itemset lattices.
\newblock In \emph{Proceedings of the ACM SIGKDD International Conference on
  Knowledge Discovery and Data Mining {(KDD)}}, pp.\  530--539. ACM, 2007.

\bibitem[Nijssen et~al.(2020)Nijssen, Schaus, et~al.]{nijssen2020}
Nijssen, S., Schaus, P., et~al.
\newblock Learning optimal decision trees using caching branch-and-bound
  search.
\newblock In \emph{Thirty-Fourth AAAI Conference on Artificial Intelligence},
  2020.

\bibitem[Payne \& Meisel(1977)Payne and Meisel]{payne1977algorithm}
Payne, H.~J. and Meisel, W.~S.
\newblock An algorithm for constructing optimal binary decision trees.
\newblock \emph{{IEEE} Transactions on Computers}, C-26\penalty0 (9):\penalty0
  905--916, 1977.

\bibitem[Quinlan(1993)]{Quinlan93}
Quinlan, J.~R.
\newblock \emph{C4.5: Programs for Machine Learning}.
\newblock Morgan Kaufmann, 1993.

\bibitem[Rudin \& Ertekin(2018)Rudin and Ertekin]{ErtekinRu18}
Rudin, C. and Ertekin, S.
\newblock Learning customized and optimized lists of rules with mathematical
  programming.
\newblock \emph{Mathematical Programming C (Computation)}, 10:\penalty0
  659--702, 2018.

\bibitem[Rudin \& Wang(2018)Rudin and Wang]{RudinWa18}
Rudin, C. and Wang, Y.
\newblock Direct learning to rank and rerank.
\newblock In \emph{International Conference on Artificial Intelligence and
  Statistics {(AISTATS)}}, pp.\  775--783, 2018.

\bibitem[Verwer \& Zhang(2019)Verwer and Zhang]{verwer2019learning}
Verwer, S. and Zhang, Y.
\newblock Learning optimal classification trees using a binary linear program
  formulation.
\newblock In \emph{Proc. Thirty-third AAAI Conference on Artificial
  Intelligence}, 2019.

\bibitem[Vilas~Boas et~al.(2019)Vilas~Boas, Santos, Merschmann, and
  Vanden~Berghe]{vilas2019}
Vilas~Boas, M.~G., Santos, H.~G., Merschmann, L. H. d.~C., and Vanden~Berghe,
  G.
\newblock Optimal decision trees for the algorithm selection problem: integer
  programming based approaches.
\newblock \emph{International Transactions in Operational Research}, 2019.

\bibitem[Wang et~al.(2017)Wang, Rudin, Doshi-Velez, Liu, Klampfl, and
  MacNeille]{WangEtAl2017}
Wang, T., Rudin, C., Doshi-Velez, F., Liu, Y., Klampfl, E., and MacNeille, P.
\newblock A {B}ayesian framework for learning rule sets for interpretable
  classification.
\newblock \emph{Journal of Machine Learning Research}, 18\penalty0
  (70):\penalty0 1--37, 2017.

\end{thebibliography}
\bibliographystyle{icml2020} 

\onecolumn
\appendix
\section{Comparison Between Decision Tree Methods}
\label{app:ComparisonTable}
\begin{table}[h]\footnotesize
    \centering
    \begin{tabular}{|m{12.5em}|m{6.0em}|m{5em}|m{5.2em}|m{5.2em}|m{4.8em}|m{4.8em}|}
    \hline
     Attributes	&	GOSDT	&	DL8	&	DL8.5 &	BinOCT	&	DTC	&	CART\\
\hline\hline
Built-in Preprocessor &	\cellcolor{green!25}Yes	& No &	No	& \cellcolor{green!25}Yes	& No & No \\ \hline
Preprocessing Strategy	&	\cellcolor{green!25}	All values & 
\cellcolor{red!25}Weka Discretization	&	\cellcolor{red!25}Bucketization	&	\cellcolor{red!25}Bucketization	&		None	&		None	\\ \hline
Preprocessing Preserves Optimality?	&		\cellcolor{green!25} Yes	&	\cellcolor{red!25}	No	&	\cellcolor{red!25}	No	&	\cellcolor{red!25}	No	&	N/A	&	N/A	\\ \hline
Optimization Strategy	&	\cellcolor{blue!25}	DPB	&	\cellcolor{blue!25}	DP	&	\cellcolor{blue!25}	DPB	&		ILP	&		Greedy Search	&		Greedy Search	\\ \hline
Optimization Preserves Optimality?	&	\cellcolor{green!25}	Yes	&	\cellcolor{green!25}	Yes	&	\cellcolor{green!25}	Yes	&	\cellcolor{green!25}	Yes	&	\cellcolor{red!25}	No	&	\cellcolor{red!25}	No	\\ \hline
Applies Hierarchical Upper bound to Reduce Search Space?	&	\cellcolor{green!25}	Yes	&	\cellcolor{red!25}	No	&	\cellcolor{green!25}	Yes	&		N/A	&		N/A	&		N/A	\\ \hline
Uses Support Set Node Identifiers?	&	\cellcolor{green!25}	Yes	&	\cellcolor{red!25}	No	&	\cellcolor{red!25}	No	&		N/A	&		N/A	&		N/A	\\ \hline
Can use Multiple Cores?	&	\cellcolor{green!25}	Yes	&	\cellcolor{red!25}	No	&	\cellcolor{red!25}	No	&	\cellcolor{green!25}	Yes	&	\cellcolor{red!25}	No	&	\cellcolor{red!25}	No	\\ \hline
Can prune using updates from partial evaluation of subproblem?	&	\cellcolor{green!25}	Yes	&	\cellcolor{red!25}	No	&	\cellcolor{red!25}	No	&		Depends on (generic) solver	&		N/A	&		N/A	\\ \hline
Strategy for preventing overfitting	&	\cellcolor{green!25}	Penalize Leaves	&	\cellcolor{red!25}	Structural Constraints	&	\cellcolor{red!25}	Structural Constraints	&	\cellcolor{red!25}	Structural Constraints	&	\cellcolor{green!25}	MDL Criteria	&	\cellcolor{red!25}	Structural Constraints	\\ \hline
Can we modify this to use regularization?	&	\cellcolor{green!25}	N/A	&	\cellcolor{green!25}	Yes	&	\cellcolor{green!25}	Yes	&	\cellcolor{green!25}	Yes	&	\cellcolor{green!25}	N/A	&	\cellcolor{red!25}	No	\\ \hline
Does it address class imbalance?	&	\cellcolor{green!25}	Yes	&		Maybe	&		Maybe	&	
Maybe &		Maybe	&		Maybe	\\ \hline
    \end{tabular}
    \caption{Comparison of GOSDT, DL8 \citep{nijssen2007mining}, DL8.5 \cite{nijssen2020}, BinOCT \citep{verwer2019learning}, DTC \citep{GarofalakisHyRaSh03}, and CART \citep{Breiman84}. \colorbox{green!25}{Green} is a comparative advantage. \colorbox{red!25}{Red} is a comparative disadvantage. \colorbox{blue!25}{Blue} highlights dynamic programming-based methods. White is neutral.}
    \label{tab:comparisontable}
\end{table} 

\section{Objectives and Their Lower Bounds for Arbitrary Monotonic Losses}\label{App:Objectives}

Before deriving the bounds for arbitrary monotonic losses, we first introduce some notation. As we know, a leaf set $d=(l_1, l_2, ..., l_{H_d})$ contains $H_d$ distinct leaves, where $l_i$ is the classification rule of the leaf $i$. If a leaf is labeled, then $y_i^{(\leaf)}$ is the label prediction for all data in leaf $i$. Therefore, a labeled partially-grown tree $d$ with the leaf set $d = (l_1, l_2, ..., l_{H_d})$ could be rewritten as $d = (d_{\fix}, \delta_{\fix}, d_{\splitrm}, \delta_{\splitrm}, K, H_d)$, where $d_{\fix} = (l_1,l_2, ..., l_K)$ is a set of $K$ fixed leaves that are not permitted to be further split, $\delta_{\fix} = (y_1^{(\leaf)}, y_2^{(\leaf)},...,y_K^{(\leaf)}) \in \{0,1\}^K$ are the predicted labels for leaves $d_{\fix}$, $d_{\splitrm} = (l_{K+1},l_{K+2}...,l_{H_d})$ is the set of $H_d-K$ leaves that can be further split, and $\delta_{\splitrm} = (y_{K+1}^{(\leaf)}, y_{K+2}^{(\leaf)},...,y_{H_d}^{(\leaf)}) \in \{0,1\}^K$ are the predicted labels for leaves $d_{\splitrm}$.  

\subsection{Hierarchical objective lower bound for arbitrary monotonic losses}
\begin{theorem}\label{thm:hierarchy}(Hierarchical objective lower bound for arbitrary monotonic losses) Let loss function $\ell(d,\x,\y)$ be monotonically increasing in $FP$ and $FN$. We now change notation of the loss to be only a function of these two quantities, written now as $\Tilde{\ell}(FP, FN)$. Let $d = (d_{\fix}, \delta_{\fix}, d_{\splitrm}, \delta_{\splitrm}, K, H)$ be a labeled tree with fixed leaves $d_{\fix}$, and let $FP_{\fix}$ and $FN_{\fix}$ be the false positives and false negatives of $d_{\fix}$. Define the lower bound to the risk $R(d,\x,\y)$ as follows (taking the lower bound of the split terms to be 0):
\[R(d,\x,\y)\geq b(d_{\fix},\x,\y) = \ell(d_{\fix},\x,\y)+\lambda H = \tilde{\ell}(FP_{\fix},FN_{\fix})+\lambda H.\]
Let $d'=(d'_{\fix}, \delta'_{\fix}, d'_{\splitrm}, \delta'_{\splitrm}, K', H') \in \sigma (d)$ be any child tree of $d$ such that its fixed leaves $d'_{\fix}$ contain $d_{\fix}$ and $K'>K$ and $H'>H$. Then, $b(d_{\fix}, \x, \y) \leqslant R(d',\x,\y)$. 
\end{theorem}
The importance of this result is that the lower bound works for all (allowed) child trees of $d$. Thus, if $d$ can be excluded via the lower bound, then all of its children can too.

\begin{proof}
Let $FP_{\fix}$ and $FN_{\fix}$ be the false positives and false negatives within leaves of $d_{\fix}$, and let $FP_{\splitrm}$ and $FN_{\splitrm}$ be the false positives and false negatives within leaves of $d_{\splitrm}$. Similarly, denote $FP_{\fix'}$ and $FN_{\fix'}$ as the false positives and false negatives of $d'_{\fix}$ and let $FP_{\splitrm'}$ and $FN_{\splitrm'}$ be the false positives and false negatives of $d'_{\splitrm}$. 
Since the leaves are mutually exclusive, $FP_d = FP_{\fix}+FP_{\splitrm}$ and $FN_d = FN_{\fix}+FN_{\splitrm}$. Moreover, since $\Tilde{\ell}(FP, FN)$ is monotonically increasing in $FP$ and $FN$, we have:
\begin{equation}
    R(d,\x,\y) = \Tilde{\ell}(FP_d, FN_d) + \lambda H \geqslant \Tilde{\ell}(FP_{\fix}, FN_{\fix}) + \lambda H = b(d_{\fix},\x,\y).
\end{equation}
Similarly, $R(d',\x,\y) \geqslant b(d'_{\fix},\x,\y)$. 
Since $d_{\fix} \subseteq d_{\fix'}$, $FP_{\fix'} \geqslant FP_{\fix}$ and $FN_{\fix'} \geqslant FN_{\fix}$, thus $\Tilde{\ell}(FP_{\fix'}, FN_{\fix'}) \geqslant \Tilde{\ell}(FP_{\fix}, FN_{\fix})$. Combined with $H' > H$, we have: 
\begin{equation}
    b(d_{\fix},\x,\y) = \Tilde{\ell}(FP_{\fix}, FN_{\fix}) + \lambda H \leqslant \Tilde{\ell}(FP_{\fix'}, FN_{\fix'}) + \lambda H' = b(d'_{\fix},\x,\y) \leqslant R(d',\x,\y).
\end{equation}
\end{proof}

Let us move to the next bound, which is the one-step lookahead. It relies on comparing the best current objective we have seen so far, denoted $R^c$, to the lower bound.
\begin{theorem}\label{thm:onestep} (Objective lower bound with one-step lookahead) Let $d$ be a $H$-leaf tree with $K$ leaves fixed and let $R^c$ be the current best objective. If $b(d_{\fix},\x,\y) + \lambda \geqslant R^c$, then for any child tree $d'\in \sigma(d)$, its fixed leaves $d'_{\fix}$ include $d_{\fix}$ and $H'>H$. It follows that $R(d',\x,\y) \geqslant R^c$. 
\end{theorem}
\begin{proof}
According to definition of the objective lower bound, 
\begin{equation}
    \begin{split}
        R(d',\x,\y)\geqslant b(d'_{\fix},\x,\y) &= \Tilde{\ell}(FP_{\fix'}, FN_{\fix'}) + \lambda H' \\
        &= \Tilde{\ell}(FP_{\fix'}, FN_{\fix'}) + \lambda H + \lambda (H'-H)\\
        &\geqslant b(d_{\fix},\x,\y) + \lambda \geqslant R^c,
    \end{split}
\end{equation}
where on the last line we used that since $H'$ and $H$ are both integers, then $H'-H\geq 1$.
\end{proof}
According to this bound, even though we might have a tree $d$ whose fixed leaves $d_{\fix}$  obeys lower bound $b(d_{\fix},\x,\y) \leqslant R^c$, its child trees may still all be guaranteed to be suboptimal: if $b(d_{\fix},\x,\y) + \lambda \geqslant R^c$, none of its child trees can ever be an optimal tree. 

\begin{theorem}\label{thm:hierarchy_subtree}(Hierarchical objective lower bound for sub-trees and additive losses) Let loss functions $\ell(d,\x,\y)$ be monotonically increasing in $FP$ and $FN$, and let the loss of a tree $d$ be the sum of the losses of the leaves. Let $R^c$ be the current best objective. Let $d$ be a tree such that the root node is split by a feature, where two sub-trees $d_{\leftrm}$ and $d_{\rightrm}$ are generated with $H_{\leftrm}$ leaves for $d_{\leftrm}$ and $H_{\rightrm}$ leaves for $d_{\rightrm}$. The data captured by the left tree is $(\x_{\leftrm},\y_{\leftrm})$ and the data captured by the right tree is $(\x_{\rightrm},\y_{\rightrm})$. Let $b(d_{\leftrm},\x_{\leftrm},\y_{\leftrm})$ and $b(d_{\rightrm},\x_{\rightrm},\y_{\rightrm})$ be the objective lower bound of the left sub-tree and right sub-tree respectively such that $b(d_{\leftrm},\x_{\leftrm},\y_{\leftrm}) \leqslant \ell(d_{\leftrm},\x_{\leftrm},\y_{\leftrm}) + \lambda H_{\leftrm}$ and $b(d_{\rightrm},\x_{\rightrm},\y_{\rightrm}) \leqslant \ell(d_{\rightrm},\x_{\rightrm},\y_{\rightrm}) + \lambda H_{\rightrm}$. If $b(d_{\leftrm},\x_{\leftrm},\y_{\leftrm}) > R^c$ or $b(d_{\rightrm},\x_{\rightrm},\y_{\rightrm}) > R^c$ or $b(d_{\leftrm},\x_{\leftrm},\y_{\leftrm})+b(d_{\rightrm},\x_{\rightrm},\y_{\rightrm}) > R^c$, then the tree $d$ is not the optimal tree.
\end{theorem}

This bound is applicable to any tree $d$, even if part of it has not been constructed yet. That is, if a partially-constructed $d$'s left lower bound or right lower bound, or the sum of left and right lower bounds, exceeds the current best risk $R^c$, then we do not need to construct $d$ since we have already proven it to be suboptimal from its partial construction. 

\begin{proof}
$R(d,\x,\y) = \ell(d,\x,\y) + \lambda H = \ell(d_{\leftrm},\x_{\leftrm},\y_{\leftrm}) + \ell(d_{\rightrm},\x_{\rightrm},\y_{\rightrm}) + \lambda H_{\leftrm}+\lambda H_{\rightrm} \geqslant b(d_{\leftrm},\x_{\leftrm},\y_{\leftrm}) + b(d_{\rightrm},\x_{\rightrm},\y_{\rightrm})$. If $b(d_{\leftrm},\x_{\leftrm},\y_{\leftrm}) > R^c$ or $b(d_{\rightrm},\x_{\rightrm},\y_{\rightrm}) > R^c$ or $b(d_{\leftrm},\x_{\leftrm},\y_{\leftrm})+b(d_{\rightrm},\x_{\rightrm}, \y_{\rightrm}) > R^c$, then $R(d,\x,\y) > R^c$. Therefore, the tree $d$ is not the optimal tree. 
\end{proof}

\subsection{Upper bound on the number of leaves}
\begin{theorem}\label{thm:ub_leaf_num}(Upper bound on the number of leaves)
For a dataset with $M$ features, consider a state space of all trees. Let $H$ be the number of leaves of tree $d$ and let $R^c$ be the current best objective. For any optimal tree $d^* \in  \argmin_d R(d,\x,\y)$, its number of leaves obeys:
\begin{equation}
    H^* \leqslant \min([R^c/\lambda], 2^M)
\end{equation}
where $\lambda$ is the regularization parameter. 
\end{theorem}
\begin{proof}
This bound adapts directly from OSDT \citep{HuRuSe2019}, where the proof can be found.
\end{proof}

\begin{theorem}\label{thm:ub_leaf_num_ps}(Parent-specific upper bound on the number of leaves)
Let $d=(d_{\fix}, \delta_{\fix}, d_{\splitrm}, \delta_{\splitrm}, K, H)$ be a tree, $d'=(d'_{\fix}, \delta'_{\fix}, d'_{\splitrm}, \delta'_{\splitrm}, K', H')\in \sigma(d)$ be any child tree such that $d_{\fix} \subseteq d'_{\fix}$, and $R^c$ be the current best objective. If $d'_{\fix}$ has lower bound $b(d'_{\fix},\x,\y) < R^c$, then
\begin{equation}
    H' < \min\bigg(H+\bigg[\frac{R^c-b(d_{\fix},\x,\y)}{\lambda}\bigg], 2^M\bigg).
\end{equation}
where $\lambda$ is the regularization parameter. 
\end{theorem}
\begin{proof}
This bound adapts directly from OSDT \citep{HuRuSe2019}, where the proof can be found.
\end{proof}

\subsection{Incremental Progress Bound to Determine Splitting and Lower Bound on Incremental Progress}
In the implementation, Theorem \ref{thm:leaf_supp} below is used to check if a leaf node within $d_{\splitrm}$ is worth splitting. If the bound is satisfied and the leaf can be further split, then we generate new leaves and Theorem \ref{thm:increm} is applied to check if this split yields new nodes or leaves that are good enough to consider in the future. Let us give an example to show how Theorem \ref{thm:leaf_supp} is easier to compute than Theorem \ref{thm:increm}. If we are evaluating a potential split on leaf $j$, Theorem \ref{thm:leaf_supp} requires $FP_j$ and $FN_j$ which are the false positives and false negatives for leaf $j$, but no extra information about the split we are going to make, whereas Theorem \ref{thm:increm} requires that additional information. Let us work with balanced accuracy as the loss function: for Theorem \ref{thm:leaf_supp} below, we would need to compute $\tau =\frac{1}{2}(\frac{FN_j}{N^+}+\frac{FP_j}{N^-})$ but for Theorem \ref{thm:increm} below we would need to calculate quantities for the new leaves we would form by splitting $j$ into child leaves $i$ and $i+1$. Namely, we would need $FN_i$, $FN_{i+1}$, $FP_i$, and $FP_{i+1}$ as well.


\begin{theorem}\label{thm:leaf_supp}(Incremental progress bound to determine splitting) 
Let $d^* = (d_{\fix}, \delta_{\fix}, d_{\splitrm}, \delta_{\splitrm}, K, H)$ be any optimal tree with objective $R^*$, i.e., $d^* \in \argmin_d R(d,\x,\y)$.
 Consider tree $d'$ derived from $d^*$ by deleting a pair of leaves $l_i$ and $l_{i+1}$ and adding their parent leaf $l_j$, $d' = (l_1, ..., l_{i-1}, l_{i+2}, ..., l_H, l_j)$. Let $\tau := \Tilde{\ell}(FP_{d'}, FN_{d'})-\Tilde{\ell}(FP_{d'}-FP_{l_j}, FN_{d'}-FN_{l_j})$. Then, $\tau$ must be at least $\lambda$. 
\end{theorem}
\begin{proof}
$\ell(d',\x,\y)=\Tilde{\ell}(FP_{d'}, FN_{d'})$ and $\ell(d^*,\x,\y)=\Tilde{\ell}(FP_{d'}+FP_{l_i}+FP_{l_{i+1}}-FP_{l_j}, FN_{d'}+FN_{l_i}+FN_{l_{i+1}}-FN_{l_j})$. The difference between $\ell(d^*,\x,\y)$ and $\ell(d',\x,\y)$ is maximized when $l_i$ and $l_{i+1}$ correctly classify all the captured data. Therefore, $\tau$ is the maximal difference between $\ell(d',\x,\y)$ and $\ell(d^*,\x,\y)$. 
Since $l(d',\x,\y) - \ell(d^*,\x,\y) \leqslant \tau$, we can get $\ell(d',\x,\y) + \lambda (H-1) \leqslant \ell(d^*,\x,\y) + \lambda (H-1) + \tau$, that is (and remember that $d^*$ is of size $H$ whereas $d'$ is of size $H-1$), $R(d',\x,\y) \leqslant R(d^*,\x,\y) - \lambda + \tau$. Since $d^*$ is optimal with respect to $R$, $0 \leqslant R(d',\x,\y) - R(d^*,\x,\y) \leqslant -\lambda + \tau$, thus, $\tau \geqslant \lambda$. \end{proof}
Hence, for a tree $d$, if any of its internal nodes contributes less than $\lambda$ in loss, even though $b(d_{\fix},\x,\y) \leqslant R^*$, it cannot be the optimal tree and none of its child tree could be the optimal tree. Thus, after evaluating tree $d$, we can prune it. 

\begin{theorem}\label{thm:increm}(Lower bound on incremental progress)
Let $d^{*}=(d_{\fix}, \delta_{\fix}, d_{\splitrm}, \delta_{\splitrm}, K, H)$ be any optimal tree with objective $R^*$, i.e., $d^* \in \argmin_d R(d,\x,\y)$. Let $d^*$ have leaves $d_{\fix} = (l_1, ..., l_H)$ and $\delta_{\fix}=(y_1^{(\leaf)}, y_2^{(\leaf)}, ..., y_H^{(\leaf)})$.
Consider tree $d'$ derived from $d^*$ by deleting a pair of leaves $l_i$ and $l_{i+1}$ with corresponding labels $y_i^{\leaf}$ and $y_{i+1}^{\leaf}$ and adding their parent leaf $l_j$ and its label $y_j^{\leaf}$. Define 
$a_i$ as the incremental objective of splitting $l_j$ to get $l_i, l_{i+1}$: $a_i := \ell(d',\x,\y) - \ell(d^*,\x,\y).$ 
In this case, $\lambda$ provides a lower bound s.t. $a_i \geqslant \lambda$. 
\end{theorem}
\begin{proof}
Let $d' = (d'_{\fix}, \delta'_{\fix}, d'_{\splitrm}, \delta'_{\splitrm}, K', H')$ be the tree derived from $d^*$ by deleting a pair of leaves $l_i$ and $l_{i+1}$, and adding their parent leaf $l_j$. Then, 
\begin{equation}
    \begin{split}
        R(d',\x,\y) &= \ell(d',\x,\y) + \lambda (H-1) = a_i + l(d^*,\x,\y) + \lambda(H-1)\\
    &= a_i + R(d^*,\x,\y) - \lambda.\\
    \end{split}
\end{equation}
Since $0 \leqslant R(d',\x,\y)-R(d^*,\x,\y)$, then $a_i \geqslant \lambda$.
\end{proof}

In the implementation, we apply both Theorem \ref{thm:leaf_supp} and Theorem \ref{thm:increm}. If Theorem \ref{thm:leaf_supp} is not satisfied, even though $b(d_{\fix},\x,\y) \leqslant R^*$, it cannot be an optimal tree and none of its child trees could be an optimal tree. In this case, $d$ can be pruned, as we showed before. However, if Theorem \ref{thm:leaf_supp} is satisfied, we check Theorem \ref{thm:increm}. If Theorem \ref{thm:increm} is not satisfied, then we would need to further split at least one of the two child leaves--either of the new leaves $i$ or $i+1$--in order to obtain a potentially optimal tree. 

\subsection{Permutation Bound}
\begin{theorem}\label{thm:perm}(Leaf Permutation bound) Let $\pi$ be any permutation of $\{1, ..., H\}$. Let $d=(d_{\fix}, d_{\splitrm}, K, H)$ and $D=(D_{\fix},D_{\splitrm}, K, H)$ be trees with leaves $(l_1, ..., l_H)$ and $(l_{\pi(1)}, ..., l_{\pi(H)})$ respectively, i.e., the leaves in $D$ correspond to a permutation of the leaves in d. Then the objective lower bounds of $d$ and $D$ are the same and their child trees correspond to permutations of each other. 
\end{theorem}
\begin{proof}
This bound adapts directly from OSDT \citep{HuRuSe2019}, where the proof can be found.
\end{proof}
Therefore, if two trees have the same leaves, up to a permutation, according to Theorem \ref{thm:perm}, one of them can be pruned. This bound is capable of reducing the search space by all future symmetries of trees we have already seen. 

\subsection{Equivalent Points Bound}
As we know, for a tree $d=(d_{\fix}, \delta_{\fix}, d_{\splitrm}, \delta_{\splitrm}, K, H)$, the objective of this tree (and that of its children) is minimized when there are no errors in the split leaves: $FP_{\splitrm}=0$ and $FN_{\splitrm}=0$. In that case, the risk is equal to $b(d_{\fix},\x,\y)$. However, if multiple observations captured by a leaf in $d_{\splitrm}$ have the same features but different labels, then no tree, including those that extend $d_{\splitrm}$, can correctly classify all of these observations, that is $FP_{\splitrm}$ and $FN_{\splitrm}$ cannot be zero. In this case, we can apply the equivalent points bound to give a tighter lower bound on the objective.

Let $\Omega$ be a set of leaves. \textit{Capture} is an indicator function that equals 1 if $x_i$ falls into one of the leaves in $\Omega$, and 0 otherwise, in which case we say that $\mathrm{cap}(x_i,\Omega)=1$. We define a set of samples to be equivalent if they have exactly the same feature values. Let $e_u$ be a set of equivalent points and let $q_u$ be the minority class label that minimizes the loss among points in $e_u$. Note that a dataset consists of multiple sets of equivalent points. Let $\{e_u\}_{u=1}^U$ enumerate these sets. 

\begin{theorem}\label{thm:equiv}(Equivalent points bound) Let $d = (d_{\fix}, \delta_{\fix}, d_{\splitrm}, \delta_{\splitrm}, K, H)$ be a tree such that $l_k \in d_{\fix}$ for $k \in \{1,...,K\}$ and $l_k \in d_{\splitrm}$ for $k \in \{K+1,...,H\}$. For any tree $d' \in \sigma(d)$, 
\begin{equation}
    \ell(d',\x,\y) \geqslant \Tilde{\ell}(FP_{\fix} + FP_e, FN_{\fix} + FN_e), \ where
\end{equation}
\begin{equation}
\begin{split}
    FP_e &= \sum_{i=1}^N \sum_{u=1}^U \sum_{k=K+1}^H {\rm cap}(x_i, l_k) \wedge \mathds{1}[y_i=0]\wedge \mathds{1}[x_i \in e_u]\mathds{1}[y_i = q_u]\\
    FN_e &= \sum_{i=1}^N \sum_{u=1}^U \sum_{k=K+1}^H {\rm cap}(x_i, l_k) \wedge \mathds{1}[y_i=1] \wedge  \mathds{1}[x_i \in e_u]\mathds{1}[y_i = q_u].
\end{split}
\end{equation}
\end{theorem}
\begin{proof}
Since $d' \in \sigma(d)$, $d' = (l'_1, ..., l'_K, l'_{K+1}, ...,l'_{K'},..., l_{H'})$ where we have both $l'_k \in d'_{\fix}$ for $k \in \{1,...,K'\}$, which are the fixed leaves and also $l'_k \in d'_{\splitrm}$ for $k \in \{K'+1,...,H'\}$. Note that for $k \in \{1,...,K\}$, $l'_k = l_k$. 

Let $\Delta = d'_{\fix}\backslash d_{\fix}$ which are the leaves in $d'_{\fix}$ that are not in $d_{\fix}$. Then $\ell(d',\x,\y) = \Tilde{\ell}(FP_{d'}, FN_{d'}) = \Tilde{\ell}(FP_{\fix}+FP_{\Delta}+FP_{\splitrm'}, FN_{\fix}+FN_{\Delta}+FN_{\splitrm'})$, where $FP_{\Delta}$ and $FN_{\Delta}$ are false positives and false negatives in $d'_{\fix}$ but not $d_{\fix}$ and $FP_{\splitrm'}$ and $FN_{\splitrm'}$ are false positives and false negatives in $d'_{\splitrm}$. For tree $d'$, its leaves in $\Delta$ are those indexed from $K$ to $K'$. Thus, the sum over leaves of $d'$ from $K$ to $H'$ includes leaves from $\Delta$ and leaves from $d_{\splitrm}'$. 
\begin{equation}
\begin{split}
    FP_{\Delta}+FP_{\splitrm'} &= \sum_{i=1}^N \sum_{u=1}^U \sum_{k=K+1}^{H'} \cap(x_i, l'_k) \wedge \mathds{1}[y_i \neq \hat{y}_k^{(\leaf)}] \wedge \mathds{1}[y_i = 0] \wedge \mathds{1}[x_i \in e_u] \\
    &\geqslant \sum_{i=1}^N \sum_{u=1}^U \sum_{k=K+1}^{H'} \cap(x_i, l'_k) \wedge \mathds{1}[y_i=0]\wedge \mathds{1}[x_i \in e_u]\mathds{1}[y_i = q_u] \\
    FN_{\Delta}+FN_{\splitrm'} &= \sum_{i=1}^N \sum_{u=1}^U \sum_{k=K+1}^{H'} \cap(x_i, l'_k) \wedge \mathds{1}[y_i \neq \hat{y}_k^{(\leaf)}] \wedge \mathds{1}[y_i = 1] \wedge \mathds{1}[x_i \in e_u] \\
    &\geqslant \sum_{i=1}^N \sum_{u=1}^U \sum_{k=K+1}^{H'} \cap(x_i, l'_k) \wedge \mathds{1}[y_i=1]\wedge \mathds{1}[x_i \in e_u]\mathds{1}[y_i = q_u].
\end{split}
\end{equation}
For $i \in \{1, ..., N\}$, the samples in $d_{\splitrm}$ are the same ones captured by either $\Delta$ or $d_{\splitrm}'$, that is $\sum_{k=K+1}^{H}\cap(x_i, l_k) = \sum_{k=K+1}^{H'}\cap(x_i, l'_k)$. Then 
\begin{equation}
    FP_{\Delta}+FP_{\splitrm'} \geqslant \sum_{i=1}^N \sum_{u=1}^U \sum_{k=K+1}^{H} \cap(x_i, l_k) \wedge \mathds{1}[y_i=0]\wedge \mathds{1}[x_i \in e_u]\mathds{1}[y_i = q_u] = FP_e.
\end{equation}
Similarly, $FN_{\Delta}+FN_{\splitrm'} \geqslant FN_e$. Therefore, 
\begin{equation}
\begin{split}
    \ell(d',\x,\y) = \Tilde{\ell}(FP_{\fix}+FP_{\Delta}+FP_{\splitrm'}, FN_{\fix}+FN_{\Delta}+FN_{\splitrm'}) \geqslant \Tilde{\ell}(FP_{\fix} + FP_e, FN_{\fix} + FN_e).
\end{split}
\end{equation}
\end{proof}

\subsection{Similar Support Bound}
Given two trees that are exactly the same except for one internal node split by different features $f_1$ and $f_2$, we can use the similar support bound for pruning. 
\begin{theorem}\label{thm:similar}(Similar support bound) Define $d = (d_{\fix}, \delta_{\fix}, d_{\splitrm}, \delta_{\splitrm}, K, H)$ and $D = (D_{\fix}, \Delta_{\fix}, D_{\splitrm}, \Delta_{\splitrm}, K, H)$ to be two trees that are exactly the same except for one internal node split by different features. Let $f_1$ and $f_2$ be the features used to split that node in $d$ and $D$ respectively. Let $t_1, t_2$ be the left and right sub-trees under the node $f_1$ in $d$ and let $T_1, T_2$ be the left and right sub-trees under the node $f_2$ in $D$.  Let $\omega$ be the  observations captured by only one of $t_1$ or $T_1$, i.e.,
\begin{equation}
    \omega \coloneqq 
    \{i:[{\rm cap}(x_i,t_1) \wedge \neg {\rm cap}(x_i, T_1) + \neg {\rm cap}(x_i, t_1) \wedge {\rm cap}(x_i, T_1)]\}.
\end{equation}
Let $FP_{-\omega}$ and $FN_{-\omega}$ be the false positives and false negatives of samples except $\omega$. The difference between the two trees' objectives is bounded as follows: 
\begin{equation}
    |R(d,\x,\y) - R(D,\x,\y)|\leqslant \gamma,\  where
\end{equation}
\begin{equation}
    \gamma := \max_{a\in \{0,...,|\omega|\}} [\Tilde{\ell}(FP_{-\omega}+a, FN_{-\omega}+|\omega|-a)] - \Tilde{\ell}(FP_{-\omega}, FN_{-\omega}).
\end{equation}
Then we have
\begin{equation}
   \left|\min_{d^+ \in \sigma{(d)}} R(d^+,\x,\y) - \min_{D^+ \in \sigma{(D)}} R(D^+,\x,\y)\right| \leqslant \gamma.
\end{equation}
\end{theorem}
\begin{proof}
The difference between the objectives of $d$ and $D$ is largest when one of them correctly classifies all the data in $\omega$ but the other misclassifies all of them. If $d$ classifies all the data corresponding to $\omega$ correctly while $D$ misclassifies them, 
\begin{equation} \label{eq:similar}
    R(d,\x,\y) - R(D,\x,\y) \geqslant \Tilde{\ell}(FP_{-\omega}, FN_{-\omega}) - \max_{a\in \{0,...,|\omega|\}} [\Tilde{\ell}(FP_{-\omega}+a, FN_{-\omega}+|\omega|-a)] = -\gamma.
\end{equation}
We can get $R(d,\x,\y)-R(D,\x,\y) \leqslant \gamma$ in the same way. Therefore, $\gamma \geqslant R(d,\x,\y)-R(D,\x,\y) \geqslant -\gamma$.

Let $d^*$ be the best child tree of $d$, i.e., $R(d^*,\x,\y) = \min_{d^+ \in \sigma{(d)}} R(d^+,\x,\y)$. Let $D' \in \sigma(D)$ be its counterpart which is exactly the same except for one internal node split by a different feature. Then, using Equation \ref{eq:similar},
\begin{equation}
    \min_{d^+ \in \sigma{(d)}} R(d^+,\x,\y) = R(d^*,\x,\y) \geqslant R(D',\x,\y) - \gamma \geqslant \min_{D^+ \in \sigma{(D)}} R(D^+,\x,\y) - \gamma.
\end{equation}
Similarly, using the symmetric counterpart to Equation \ref{eq:similar} and the same logic, $\min\limits_{D^+ \in \sigma{(D)}} R(D^+,\x,\y) + \gamma \geqslant \min\limits_{d^+ \in \sigma{(d)}} R(d^+,\x,\y)$.
\end{proof}

\section{Objectives and Their Lower Bounds for Rank Statistics}\label{App: Obj_rank}
In this appendix, we provide proofs for Theorem \ref{thm: lb_auc} and Theorem \ref{thm: hierarchy_auc}, and adapt the Incremental Progress Bound to Determine Splitting and the Equivalent Points Bound for the objective $\AUC_{\ch}$. The Upper Bound on the Number of Leaves, Parent-Specific Upper Bound on the Number of Leaves, Lower Bound of Incremental Progress, and Permutation Bound are the same as the bounds in Appendix \ref{App:Objectives}. We omit these duplicated proofs here. At the end of this appendix, we define the objective $\rm p\AUC_{\ch}$ and how we implement the derived bounds for this objective. As a reminder, we use notation $d=(d_{\fix}, r_{\fix}, d_{\splitrm}, r_{\splitrm}, K, H_d)$ to represent tree $d$. 

\begin{lemma}\label{lemma:auc}
Let $d=(d_{\fix},r_{\fix}, d_{\splitrm}, r_{\splitrm}, K, H_d)$ be a tree. The AUC convex hull does not decrease when an impure leaf is split. 
\end{lemma}
\begin{proof}
Let $l_i$ be the impure leaf that we intend to split, where $i \in \{1,...,H_d\}$. Let $n_i^+$ be the positive samples in $l_i$ and $n_i^-$ negative samples. Suppose $l_i$ is ranked in position ``pos.'' If the leaf is split once, it will generate two leaves $l_{i_1}$ and $l_{i_2}$ such that $r_{i_1} \geq r_{i} \geq r_{i_2}$ without loss of generality. Let $d'$ be the tree that consists of the leaf set $(l_1,...,l_{i-1},  l_{i+1},...,l_{H_d}, l_{i_1}, l_{i_2})$. If $r_{i_1} = r_{i} = r_{i_2}$, then the rank order of leaves (according to the $r_i$'s) will not change, so $\AUC_{\ch}$ will be unchanged after the split. Otherwise (if the rank order of leaves changes when introducing a child) we can reorder the $H+1$ leaves, leading to the following four cases. 
For the new leaf set $(l_1,...,l_{i-1},  l_{i+1},...,l_{H_d}, l_{i_1}, l_{i_2})$, either:
\begin{enumerate}[leftmargin=*, topsep=0pt, noitemsep]
    \item The rank of $l_{i_1}$ is smaller than $\rm pos$ and the rank of $l_{i_2}$ equals $\rm pos+1$ (requires $r_{i_1}>r_{i}$ and $r_{i} \geq r_{i_2}$);
    \item The rank of $l_{i_1}$ is smaller than $\rm pos$ and the rank of $l_{i_2}$ is larger than $\rm pos+1$ (requires $r_{i_1}>r_i$ and $r_i>r_{i_2}$); 
    \item The rank of $l_{i_1}$ is equal to $\rm pos$ and the rank of $l_{i_2}$ is equal to $\rm pos+1$ (requires $r_{i_1}\geq r_i$ and $r_i \geq r_{i_2}$); 
    \item The rank of $l_{i_1}$ is $\rm pos$ and the rank of $l_{i_2}$ is larger than $\rm pos+1$ (requires $r_{i_1}\geq r_i$ and $r_i > r_{i_2}$).
\end{enumerate}
\begin{figure}[]
    \centering
    \includegraphics[scale=0.4]{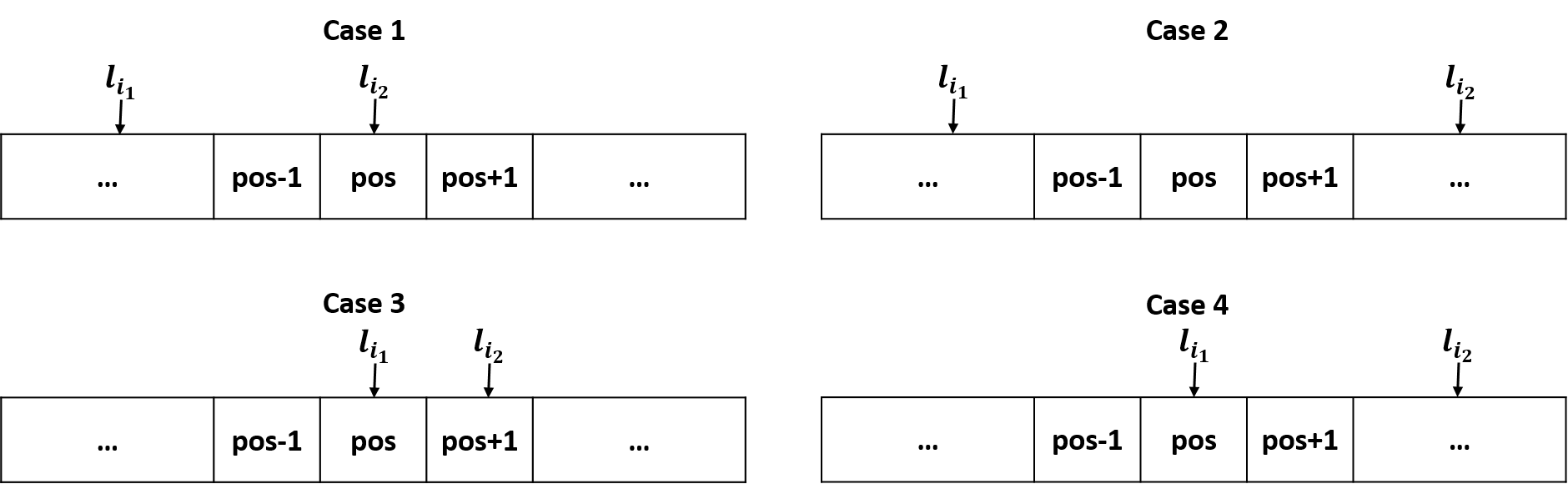}
    \caption{Four cases of positions of $l_{i_1}$ and $l_{i_2}$}
    \label{fig:pos}
\end{figure}
\begin{figure}[h]
    \centering
    \includegraphics[width=0.78\linewidth]{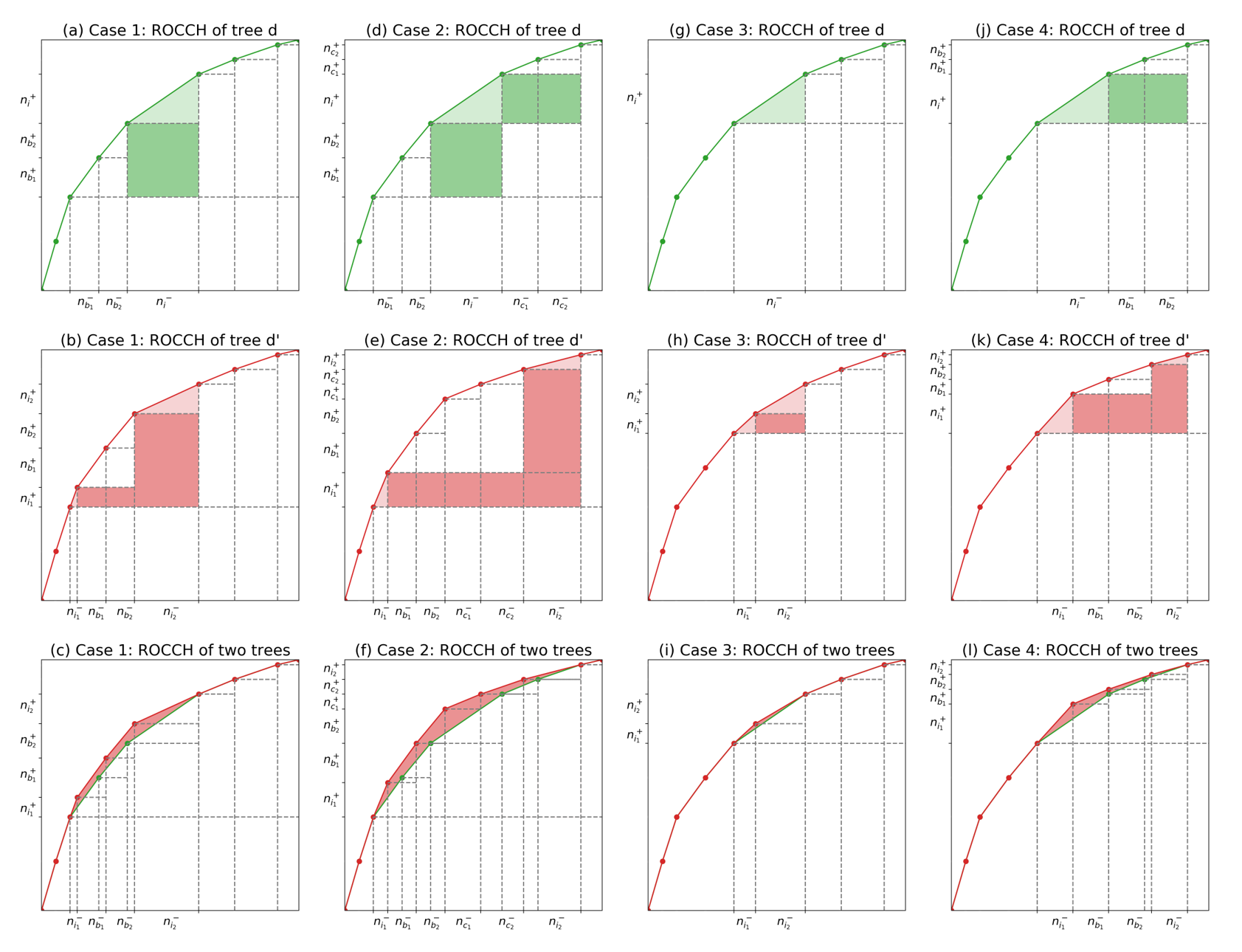}
    \caption{Four different cases of changing the rank orders after introducing a child. Each column represents a case. Subplots in the first row show the ROCCH of tree $d$ and subplots in the second row indicate the ROCCH of tree $d'$ corresponding to the different positions of two child leaves. Subplots in the third row present the change of $\AUC_{\ch}$ after introducing the child leaves.}
    \label{fig:c1}
\end{figure}
Figure \ref{fig:pos} shows four cases of the positions of $l_{i_1}$ and $l_{i_2}$. 

Let us go through these cases in more detail.

For the new leaf set after splitting $l_i$, namely $(l_1,...,l_{i-1}, l_{i+1},...,l_{H_d}, l_{i_1}, l_{i_2})$, we have that:
\begin{enumerate}[leftmargin=*]
    \item $l_{i_1}$ has rank smaller than $\rm pos$ (which requires $r_{i_1}>r_{i}$) and $l_{i_2}$ has rank $\rm pos+1$ (which requires $r_{i} \geq r_{i_2}$).
    
    Let $A = \{l_{a_1}, l_{a_2},...l_{a_U}\}$ be a collection of leaves ranked before $l_{i_1}$ and let $B = \{l_{b_1}, l_{b_2},...l_{b_V}\}$ be a collection of leaves ranked after $l_{i_1}$ but before $\rm pos+1$. In this case, recalling Equation (\ref{eq:auc_loss}), a change in the $\AUC_{\ch}$ after splitting on leaf $i$ is due only to a subset of leaves, namely $l_{i_1}, l_{b_1}, ..., l_{b_V}, l_{i_2}$. Then we can compute the change in the $\AUC_{\ch}$ as follows: 
    \begin{equation}\label{eq:c1_1}
        \Delta_{\AUC_{\ch}} = \frac{1}{N^+N^-}\bigg( \frac{n_{i_1}^-n_{i_1}^+}{2} +
        \bigg(\sum_{v=1}^V n_{b_v}^-\bigg) n_{i_1}^+  
        +n_{i_2}^-\bigg[n_{i_1}^++\bigg(\sum_{v=1}^V n_{b_v}^+\bigg) +\frac{n_{i_2}^+}{2}\bigg]
        - n_i^-\bigg[\bigg(\sum_{v=1}^V n_{b_v}^+\bigg) +\frac{n_{i}^+}{2}\bigg]
        \bigg)
    \end{equation}
To derive the expression for $\Delta_{\AUC_{\ch}}$, we first sum shaded areas of rectangles and triangles under the ROC curves' convex hull for both tree $d$ and its child tree $d'$, and then calculate the difference between the two shaded areas, as indicated in Figure \ref{fig:c1} (a-c). This figure shows where each of the terms arises within the $\Delta_{\AUC_{\ch}}$: terms $n_{b_v}^-n_{i_1}^+$ and $n_{i_2}^-[n_{i_1}^++(\sum_{v=1}^V n_{b_v}^+)]$ come from the area of rectangles colored in dark pink in Figure \ref{fig:c1} (b). Terms $\frac{n_{i_1}^-n_{i_1}^+}{2}$ and $\frac{n_{i_2}^-n_{i_2}^+}{2}$ handle the top triangles colored in light pink. Term $n_i^-(\sum_{v=1}^V n_{b_v}^+)$ represents the rectangles colored in dark green in Figure \ref{fig:c1} (a) and term $\frac{n_i^-n_{i}^+}{2}$ deals with the triangles colored in light green. Subtracting green shaded areas from red shaded areas, we get $\Delta_{\AUC_{\ch}}$, which is represented by the (remaining) pink area in Figure \ref{fig:c1} (c).  
    
Simplifying Equation (\ref{eq:c1_1}), we get \begin{equation}
        \Delta_{\AUC_{\ch}} = \frac{1}{N^+N^-}\bigg(n_{i_1}^+ \bigg(\sum_{v=1}^V n_{b_v}^- \bigg) +n_{i_2}^- \bigg(\sum_{v=1}^V n_{b_v}^+ \bigg)- n_i^- \bigg(\sum_{v=1}^V n_{b_v}^+\bigg)  + \frac{n_{i_1}^- n_{i_1}^+ + 2 n_{i_2}^- n_{i_1}^+ + n_{i_2}^-n_{i_2}^+ - n_i^-n_i^+}{2}
        \bigg). \end{equation}
Recall that $n_i^- = n_{i_1}^-+n_{i_2}^-$ and $n_i^+ = n_{i_1}^+ + n_{i_2}^+$. Then, simplifying,         
        \begin{equation}
        \Delta_{\AUC_{\ch}}
        = \frac{1}{N^+N^-}\bigg(n_{i_1}^+\bigg(\sum_{v=1}^V n_{b_v}^-\bigg)-n_{i_1}^-\bigg(\sum_{v=1}^V n_{b_v}^+\bigg) + \frac{n_{i_1}^+n_{i_2}^- - n_{i_1}^-n_{i_2}^+}{2} 
        \bigg).
    \end{equation}
    Since $r_{i_1} > r_{b_1} > r_{b_2} > ... > r_{b_V}$, $\forall v \in \{1, 2, ..., V\}, \frac{n_{i_1}^+}{n_{i_1}^+ + n_{i_1}^-} > \frac{n_{b_v}^+}{n_{b_v}^++n_{b_v}^-}$. Then we can get $n^+_{i_1}n^-_{b_v} > n^+_{b_v}n^-_{i_1}$. Hence $n^+_{i_1}(\sum_{v=1}^V n_{b_v}^-)-n_{i_1}^-(\sum_{v=1}^V n_{b_v}^+)> 0$. Similarly, because $r_{i_1}>r_{i_2}$, $\frac{n_{i_1}^+}{n_{i_1}^++n_{i_1}^-} > \frac{n_{i_2}^+}{n_{i_2}^++n_{i_2}^-}$. Then $n_{i_1}^+n_{i_2}^- > n_{i_1}^-n_{i_2}^+$. Therefore, $\Delta_{\AUC_{\ch}} > 0$.

    \item $l_{i_1}$ has a ranking smaller than $\rm pos$ (which requires $r_{i_1} > r_i$) and $l_{i_2}$ has a ranking larger than $\rm pos+1$ (which requires $r_i > r_{i_2}$).
    
    Let $A = \{l_{a_1},...l_{a_U}\}$ be a collection of leaves that ranked before $l_{i_1}$ and $B = \{l_{b_1},...l_{b_V}\}$ be a collection of leaves that ranked after $l_{i_1}$ but before $\rm pos+1$, and $C=\{l_{c_1}, ..., l_{c_W}\}$ be a collection of leaves that ranked after $\rm pos+1$ but before the rank of $l_{i_2}$. In this case, the change is caused by $l_{i_1}, l_{b_1}, ..., l_{b_V}, l_{c_1}, ..., l_{c_W}, l_{i_2}$. Then we can compute the change in the $\AUC_{\ch}$ as follows: 
\begin{equation}\label{eq:c1_2}
    \begin{aligned}
    \Delta_{\AUC_{\ch}} &= \frac{1}{N^+N^-} \bigg( \frac{n_{i_1}^-n_{i_1}^+}{2} 
    +\bigg(\sum_{v=1}^V n_{b_v}^-\bigg)n_{i_1}^+ 
    +\bigg(\sum_{w=1}^W n_{c_w}^-\bigg)n_{i_1}^+
    +n_{i_2}^-\bigg[n_{i_1}^+ + \bigg(\sum_{v=1}^{V}n_{b_v}^+\bigg)+ \bigg(\sum_{w=1}^{W}n_{c_w}^+\bigg) +\frac{n_{i_2}^+}{2} \bigg]\\
    &\ \ \ \ - n_i^-\bigg[\bigg(\sum_{v=1}^{V}n_{b_v}^+\bigg) +\frac{n_i^+}{2} \bigg] - \bigg(\sum_{w=1}^W n_{c_w}^-\bigg)n_i^+\bigg).
    \end{aligned}
\end{equation}
Similar to the derivation proposed in case 1, we first sum shaded areas of rectangles and triangles under the ROC curves' convex hull for both tree $d$ and its child tree $d'$ and then calculate the difference between two shaded areas as indicated in Figure \ref{fig:c1} (d-f). These three subfigures show where each of the terms arises within the $\Delta_{\AUC_{\ch}}$: terms $n_{b_v}^-n_{i_1}^+$, $n_{c_w}^-n_{i_1}^+$, and $n_{i_2}^-[n_{i_1}^++(\sum_{v=1}^V n_{b_v}^+)+(\sum_{w=1}^W n_{c_w}^+)]$ come from the area of rectangles colored in dark pink in Figure \ref{fig:c1} (e). Terms $\frac{n_{i_1}^-n_{i_1}^+}{2}$ and $\frac{n_{i_2}^-n_{i_2}^+}{2}$ handle the top triangles colored in light pink. Terms $n_i^-n_{b_v}^+$ and $n_i^+n_{c_w}^-$ represent the rectangles colored in dark green in Figure \ref{fig:c1} (d) and term $\frac{n_i^-n_{i}^+}{2}$ deals with the triangles colored in light green. Subtracting green shaded areas from red shaded areas, we can get $\Delta_{\AUC_{\ch}}$, which is represented by the light red area in Figure \ref{fig:c1} (f). 

Recall that $n_i^- = n_{i_1}^-+n_{i_2}^-$ and $n_i^+ = n_{i_1}^+ + n_{i_2}^+$. Then, simplifying Equation (\ref{eq:c1_2}), we get 
    \begin{equation}
    \begin{aligned}
        \Delta_{\AUC_{\ch}} &= \frac{1}{N^+N^-}\bigg(\bigg(\sum_{v=1}^V n_{b_v}^-\bigg)n_{i_1}^+ 
    +\bigg(\sum_{w=1}^W n_{c_w}^-\bigg)n_{i_1}^+ +  \bigg(\sum_{v=1}^{V}n_{b_v}^+\bigg)n_{i_2}^-+  \bigg(\sum_{w=1}^{W}n_{c_w}^+\bigg)n_{i_2}^- -  \bigg(\sum_{v=1}^{V}n_{b_v}^+\bigg)n_i^-\\
    &\ \ \ \ -\bigg(\sum_{w=1}^W n_{c_w}^-\bigg)n_i^+  +\frac{n_{i_1}^-n_{i_1}^+ + 2n_{i_2}^- n_{i_1}^++n_{i_2}^-n_{i_2}^+ - n_{i}^-n_{i}^+}{2}\bigg)\\
    & = \frac{1}{N^+N^-}\bigg( \bigg(\sum_{v=1}^{V}n_{b_v}^-\bigg)n_{i_1}^+ - \bigg(\sum_{v=1}^{V}n_{b_v}^+\bigg) n_{i_1}^- +  \bigg(\sum_{w=1}^{W}n_{c_w}^+\bigg)n_{i_2}^- -\bigg(\sum_{w=1}^{W}n_{c_w}^-\bigg)n_{i_2}^+ + \frac{n_{i_2}^-n_{i_1}^+ - n_{i_1}^-n_{i_2}^+}{2}
        \bigg).
    \end{aligned}
    \end{equation}
    Since $r_{i_1} > r_{b_1}>...>r_{b_V}$, $\forall v \in \{1, ..., V\}$, $\frac{n_{i_1}^+}{n_{i_1}^++n_{i_1}^-} > \frac{n_{b_v}^+}{n_{b_v}^+n_{b_v}^-}$. Then we get $n^+_{i_1}n^-_{b_v} > n^+_{b_v}n^-_{i_1}$. Thus, $(\sum_{v=1}^{V}n_{b_v}^-)n_{i_1}^+ - (\sum_{v=1}^{V}n_{b_v}^+) n_{i_1}^->0$. Similarly, since $r_{c_1} >...>r_{c_W} > r_{i_2}$, $\forall w \in \{1, ..., W\}, n^+_{c_w}n^-_{i_2} > n^+_{i_2}n^-_{c_w}$. Thus, $(\sum_{w=1}^{W}n_{c_w}^+)n_{i_2}^- -(\sum_{w=1}^{W}n_{c_w}^-)n_{i_2}^+ >0$. Moreover, because $r_{i_1}>r_{i_2}$, $n_{i_2}^-n_{i_1}^+ > n_{i_1}^-n_{i_2}^+$.
    Hence, $\Delta_{\AUC_{\ch}} > 0$.
    
    \item $l_{i_1}$ has a ranking same as $\rm pos$ (which requires $r_{i_1} \leq r_i$) and $l_{i_2}$ has a ranking $\rm pos+1$ (which requires $r_i \leq r_{i_2}$). 
    
    In this case, the change of $\AUC_{\ch}$ is caused by $l_{i_1}$ and $l_{i_2}$. Then we compute the change in the $\AUC_{\ch}$ as follows:  
    \begin{equation}\label{eq:c1_3}
        \Delta_{\AUC_{\ch}} = \frac{1}{N^+N^-} \bigg(\frac{n_{i_1}^-n_{i_1}^+}{2} + n_{i_2}^-n_{i_1}^+ + \frac{n_{i_2}^-n_{i_2}^+}{2} - \frac{n_i^-n_i^+}{2}\bigg).
    \end{equation}
We derive the expression in the similar way as case 1 and case 2. Term $n_{i_2}^-n_{i_1}^+$ comes from the area of rectangle colored in dark pink in Figure \ref{fig:c1} (h) and terms $\frac{n_{i_1}^-n_{i_1}^+}{2}$ and $\frac{n_{i_2}^-n_{i_2}^+}{2}$ handle the top triangles colored in light pink. Term $\frac{n_{i}^-n_{i}^+}{2}$ deals with the top triangle colored in light green in Figure \ref{fig:c1} (g). Subtracting green shaded areas from pink shaded areas, we get $\Delta_{\AUC_{\ch}}$, which is represented by the (remaining) pink area in Figure \ref{fig:c1} (i). 

Recall $n_i^- = n_{i_1}^- + n_{i_2}^-$ and $n_i^+ = n_{i_1}^++n_{i_2}^+$. Simplifying Equation (\ref{eq:c1_3}), we get 
\begin{equation}
    \Delta_{\AUC_{\ch}} = \frac{1}{N^+N^-}\bigg(\frac{n^+_{i_1}n^-_{i_2}-n^+_{i_2}n^-_{i_1}}{2} \bigg)
\end{equation}
    Since $r_{i_1} > r_{i_2}$, $n^+_{i_1}n^-_{i_2}-n^+_{i_2}n^-_{i_1}$. Therefore, $\Delta_{\AUC_{\ch}} >0$. 
    
    \item $l_{i_1}$ has a ranking same as $\rm pos$ (which requires $r_{i_1}\leq r_i$) and $l_{i_2}$ has a ranking larger than $\rm pos+1$ (which requires $r_i > r_{i_2}$). 
    
    Let $A = \{l_{a_1},...,l_{a_U}\}$ be a collection of leaves that ranked before $l_{i_1}$ and $B = \{l_{b_1}, ...,l_{b_V}\}$ be a collection of leaves that ranked after $l_{i_1}$ but before $l_{i_2}$. In this case the change of $\AUC_{\ch}$ is caused by $l_{i_1}, l_{b_1}, ..., l_{b_V}, l_{i_2}$. Then we can compute the change as follows: 
    \begin{equation}\label{eq:c1_4}
        \Delta_{\AUC_{\ch}} = \frac{1}{N^+N^-} \bigg(\frac{n_{i_1}^-n_{i_1}^+}{2} + \bigg(\sum_{v=1}^Vn_{b_v}^-\bigg)n_{i_1}^+ + n_{i_2}^-\bigg[n_{i_1}^+ + \bigg(\sum_{v=1}^Vn_{b_v}^+\bigg) + \frac{n_{i_2}^+}{2}\bigg] - \frac{n_i^-n_i^+}{2} -\bigg(\sum_{v=1}^Vn_{b_v}^-\bigg)n_i^+ \bigg)
    \end{equation}
The Figure \ref{fig:c1} (j-l) show where each of the terms arises within the $\Delta_{\AUC_{\ch}}$: terms $n_{b_v}^-n_{i_1}^+$ and $n_{i_2}^-[n_{i_1}^+ + (\sum_{v=1}^V n_{b_v}^+)]$ come from the area of rectangles colored in dark pink in Figure \ref{fig:c1} (k) and terms $\frac{n_{i_1}^-n_{i_1}^+}{2}$ and $\frac{n_{i_2}^-n_{i_2}^+}{2}$ handle triangles colored in light pink. Term $n_{b_v}^-n_i^+$ represents the rectangle colored in dark green in Figure \ref{fig:c1} (j) and term $\frac{n_i^-n_i^+}{2}$ deals with the triangle colored in light green. Subtracting green shaded areas from pink shaded areas, we get $\Delta_{\AUC_{\ch}}$, which is represented by the (remaining) pink area in Figure \ref{fig:c1} (l). 

Recall $n_i^- = n_{i_1}^- + n_{i_2}^-$ and $n_i^+ = n_{i_1}^++n_{i_2}^+$. Simplifying Equation (\ref{eq:c1_4}), we get 
\begin{equation}
\begin{aligned}
    \Delta_{\AUC_{\ch}} &= \frac{1}{N^+N^-} \bigg( \bigg(\sum_{v=1}^Vn_{b_v}^-\bigg)n_{i_1}^++ \bigg(\sum_{v=1}^Vn_{b_v}^+\bigg) n_{i_2}^- - \bigg(\sum_{v=1}^Vn_{b_v}^-\bigg)n_i^+
    +\frac{n_{i_1}^-n_{i_1}^+ + 2n_{i_2}^-n_{i_1}^+ + n_{i_2}^-n_{i_2}^+ - n_i^-n_i^+}{2}
    \bigg)\\
    &=\frac{1}{N^+N^-}\bigg( \bigg(\sum_{v=1}^Vn_{b_v}^+\bigg) n_{i_2}^- - \bigg(\sum_{v=1}^Vn_{b_v}^-\bigg)n_{i_2}^+ + \frac{n_{i_2}^-n_{i_1}^+ - n_{i_1}^-n_{i_2}^+}{2}
    \bigg)
\end{aligned}
\end{equation}
Since $r_{b_1}>...>r_{b_v} > r_{i_2}$, $\forall v \in \{1, ..., V\}, n^+_{b_v}n^-_{i_2} > n^+_{i_2}n^-_{b_v}$. Thus, $(\sum_{v=1}^Vn_{b_v}^+) n_{i_2}^- > (\sum_{v=1}^Vn_{b_v}^-)n_{i_2}^+$. Since $r_{i_1} > r_{i_2}$, $n^+_{i_1}n^-_{i_2}-n^+_{i_2}n^-_{i_1}$. Therefore, $\Delta_{\AUC_{\ch}} >0$.
\end{enumerate}
Therefore, once an impure leaf is split, the $\AUC_{\ch}$ doesn't decrease. If the split is leading to the change of the rank order of leaves, then $\AUC_{\ch}$ increases.  
\end{proof}

\subsection{Proof of Theorem \ref{thm: lb_auc}}
\begin{proof}
Let tree $d = (d_{\fix}, r_{\fix}, d_{\splitrm}, r_{\splitrm}, K, H_d)$ and leaves of tree $d$ are mutually exclusive. According to Lemma \ref{lemma:auc}, for leaves that can be further split, we can do 
splits to increase the $\AUC_{\ch}$.
In the best possible hypothetical case, all new generated leaves are pure (contain only positive or negative samples). In this case, $\AUC_{\ch}$ is maximized for $d_{\splitrm}$. In this case, we will show that $b(d_{\fix},\x,\y) =1-\frac{1}{N^+N^-}\bigg (\sum\limits_{i=1}^K n_i^-\bigg[N_{\splitrm}^+ + \bigg(\sum\limits_{j=1}^{i-1} n_j^+\bigg) +\frac{1}{2}n_i^+\bigg]+N^+N_{\splitrm}^- \bigg) + \lambda H_d$ as defined by Equation (\ref{eq:auc_lb}). Then $b(d_{\fix},\x,\y) \leq R(d,\x,\y)$. 

To derive the expression for $b(d_{\fix},\x,\y)$, we sum areas of rectangles and triangles under the ROC curve's convex hull. Figure \ref{fig:aucchderive} shows where each of the terms arises within this sum: the first term in the sum, which is $n_i^-N_{\splitrm}^+$, comes from the area of the lower rectangle of the ROC curve's convex hull, colored in green. This rectangle arises from the block of $N_{\splitrm}^+$ positives at the top of the ranked list (within the split leaves, whose hypothetical predictions are 1). The $n_i^-\bigg(\sum\limits_{j=1}^{i-1} n_j^+\bigg)$ term handles the areas of the growing rectangles, colored in blue in Figure \ref{fig:aucchderive}. The areas of the triangles comprising the top of the ROCCH account for the third term $\frac{1}{2}n_i^-n_i^+$. The final term $N^+N_{\splitrm}^-$ comes from the rectangle on the right, colored red, stemming from the split observations within a hypothetical purely negative leaf.

\begin{figure}[]
    \centering
    \includegraphics[scale=0.3]{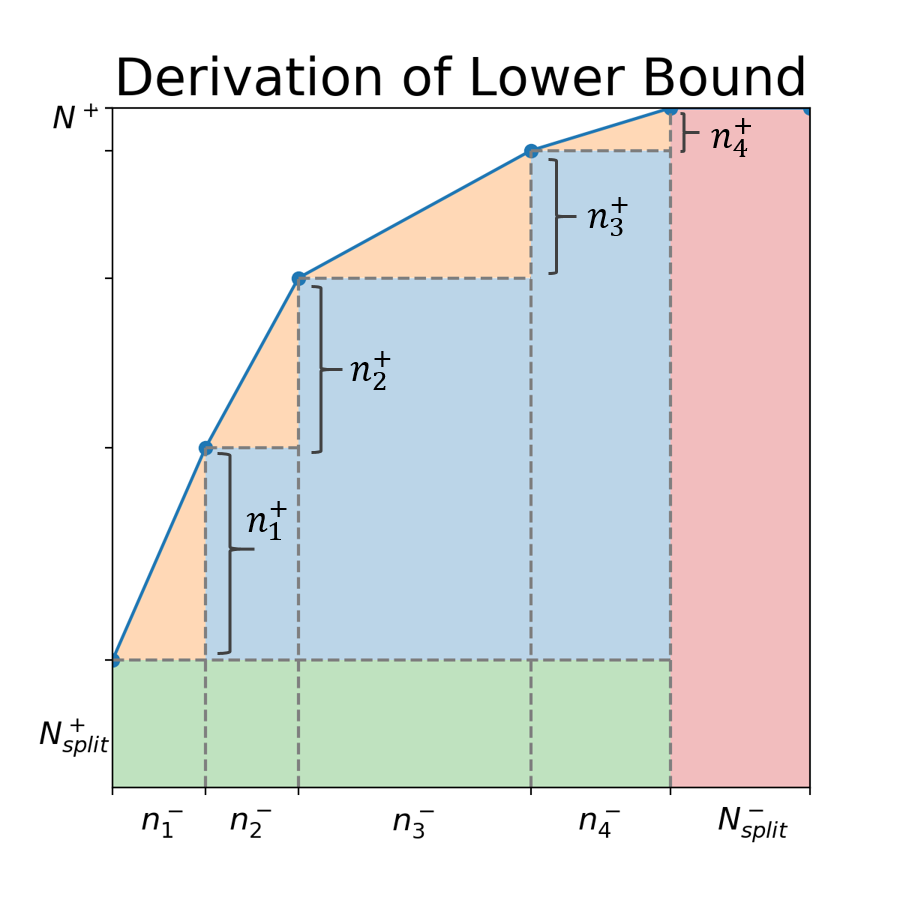}
    \caption{Example derivation of the hierarchical objective lower bound for rank statistics. $N^+$ at top left is the number of positive samples in the dataset. $N^+_{\splitrm}$ at bottom left is the number of positive samples ($N^-_{\splitrm}$ is the number of negative samples) captured by leaves that can be further split.}
    \label{fig:aucchderive}
\end{figure}

\end{proof}

\subsection{Proof of Theorem \ref{thm: hierarchy_auc}}
\begin{proof}
According to Theorem \ref{thm: lb_auc}, $R(d',\x,\y) \geqslant b(d'_{\fix},\x,\y)$ and $R(d,\x,\y) \geqslant b(d_{\fix},\x, \y)$. Since $d_{\fix} \subseteq d'_{\fix}$, leaves in $d_{\splitrm}$ but not in $d'_{\splitrm}$ can be further split to generate pure leaves for tree $d$ but fixed for tree $d'$. Based on Lemma \ref{lemma:auc} and $H_d \leqslant H_{d'}$, $b(d_{\fix},\x,\y) \leqslant b(d'_{\fix},\x,\y)$. Therefore, $b(d_{\fix},\x,\y) \leqslant R(d',\x,\y)$. 
\end{proof}

\subsection{Incremental Progress Bound to Determine Splitting for Rank Statistics}
\begin{theorem}\label{thm:leaf_supp_rank}(Incremental progress bound to determine splitting for rank statistics)
Let $d^* = (d_{\fix}, r_{\fix}, d_{\splitrm}, r_{\splitrm}, K, H_d)$ be any optimal tree with objective $R^*$, i.e., $d^* \in \argmin_d R(d,\x,\y)$. Consider tree $d'$ derived from $d^*$ by deleting a pair of leaves $l_i$ and $l_{i+1}$ and adding their parent leaf $l_j$, $d'=(l_1,...,l_{i-1}, l_{i+2},...,l_{H_d}, l_j)$. Let $n^+_j$ be the number of positive samples ($n^-_j$ be the number of negative samples) in leaf $j$. Calculate $\ell(d',\x,\y)$ as in Equation (\ref{eq:auc_loss}). Define $d'_{-j}$ as the tree $d'$ after dropping leaf $j$, adding two hypothetical pure leaves (i.e., one has all positives and the other has all negatives), and reordering the remaining $H_d-2$ leaves based on the fraction of positives. 
Then, we can calculate the loss of the tree $d'_{-j}$ as
\begin{equation}
    \ell(d'_{-j},\x,\y) = 1-\frac{1}{N^+N^-}\bigg(\sum_{i=1}^{H_d-2} n_i^-\bigg[n_j^+ + \bigg(\sum_{v=1}^{i-1} n_v^+\bigg)+\frac{1}{2}n_i^+\bigg]+N^+n_j^-\bigg).
\end{equation}
Let $\tau := \ell(d',\x,\y)-\ell(d'_{-j},\x,\y)$. Then $\tau$ must be at least $\lambda$. 
\end{theorem}
\begin{proof}
By the way $\ell(d'_{-j},\x,\y)$ is defined (using the same counting argument we used in the proof of Theorem \ref{thm: lb_auc}), it is a lower bound for $\ell(d^*,\x,\y)$. These two quantities are equal when the split leaves are all pure. So, we have $\ell(d'_{-j},\x,\y) \leq \ell(d^*,\x,\y)$.
Since  $\ell(d',\x,\y) - \ell(d^*,\x,\y) \leqslant \ell(d',\x,\y)-\ell(d'_{-j},\x,\y) = \tau$, we can get $\ell(d',\x,\y) + \lambda (H_d-1) \leqslant \ell(d^*,\x,\y) + \lambda (H_d-1) + \tau$, that is (and remember that $d^*$ is of size $H_d$ whereas $d'$ is of size $H_{d}-1$), $R(d',\x,\y) \leq R(d^*,\x,\y)-\lambda +\tau$. Since $d^*$ is optimal with respect to $R$, then $0 \leqslant R(d',\x,\y) - R(d^*,\x,\y) \leqslant -\lambda + \tau$, thus $\tau \geqslant \lambda$.
\end{proof}
Similar to the Incremental Progress Bound to Determine Splitting for arbitrary monotonic losses, for a tree $d$, if any of its internal node contributes less than $\lambda$ in loss, it is not the optimal tree. Thus, after evaluating tree $d$, we can prune it. 

\subsection{Equivalent points bound for rank statistics}
Similar to the equivalent points bound for arbitrary monotonic losses, for a tree $d=(d_{\fix}, r_{\fix}, d_{\splitrm}, r_{\splitrm}, K, H_d)$, the objective of this tree and its children is minimized when leaves that can be further split to generate pure leaves. In the case when it is possible to split the data into pure leaves, the risk could be equal to $b(d_{\fix},\x,\y)$. However, if multiple observations captured by a leaf in $d_{\splitrm}$ have the same features but different labels, then no tree, including those that extend $d_{\splitrm}$, can correctly classify all of these observations; that is, leaves in $d_{\splitrm}$ can not generate pure leaves. In this case, we can leverage these equivalent points to give a tighter lower bound for the objective. We use the same notation for capture and set of equivalent points as in Appendix \ref{App:Objectives}, and minority class label is simply the
label with fewer samples. 

Let $d=(d_{\fix}, r_{\fix}, d_{\splitrm}, r_{\splitrm}, K, H_d)$ be a tree. Data in leaves from $d_{\splitrm}$ can be separated into $U$ equivalence classes. 
For $u=1,...,U$, let $N_{u} = \sum_{i=1}^N \cap(x_i, d_{\splitrm}) \wedge \mathds{1}[x_i \in e_u]$, that is, the number of samples captured by $d_\splitrm$ belonging to equivalence set $u$, and let $\delta_{u} = \sum_{i=1}^N \cap(x_i, d_{\splitrm}) \wedge \mathds{1}[x_i \in e_u] \mathds{1}[y_i = q_u]$ be the number of minority-labeled samples captured by $d_{\splitrm}$ in equivalence point set $u$. Then we define $\tilde{r}_u$ as the classification rule we would make on each equivalence class separately (if we were permitted):
\[\tilde{r}_u = \left\{ \begin{array}{rcl}
\frac{\delta_u}{N_u} & \mbox{if}
& \delta_u \geq 0 \text{ and } q_u = 1 \\ \frac{N_u-\delta_u}{N_u} & \mbox{if} & \delta_u \geq 0 \text{ and } q_u = 0.
\end{array}\right.\]

Let us combine this with $d_{\fix}$ to get a bound.
We order the combination of leaves in $d_{\fix}$ and equivalence classes $u \in \{1,...,U\}$ by the fraction of positives in each, Sort$(r_1,...,r_i,...,r_K, \tilde{r}_1,...,\tilde{r}_u,..., \tilde{r}_U)$ from highest to lowest. Let us re-index these sorted $K+U$ elements by index $\tilde{i}$.   
Denote $n^+_{\tilde{i}}$ as the number of positive samples in either the leaf or equivalence class corresponding to $\tilde{i}$. (Define
$n^-_i$ to be the number of negative samples analogously). We define our new, tighter, lower bound as: 
\begin{equation} \label{eq:lb_equiv}
    b(d_{\rm equiv},\x,\y) = 1-\frac{1}{N^+N^-}\sum_{{\tilde{i}}=1}^{K+U}n^-_{\tilde{i}}\bigg[\bigg(\sum_{j=1}^{{\tilde{i}}-1} n^+_{j}\bigg) +\frac{1}{2}n^+_{\tilde{i}}\bigg] + \lambda H_d.
\end{equation}
\begin{theorem}\label{thm:equiv_rank}(Equivalent points bound for rank statistics)
For a tree $d=(d_{\fix}, r_{\fix}, d_{\splitrm}, r_{\splitrm},K,H_d)$, 
Let tree $d'=(d'_{\fix}, r'_{\fix}, d'_{\splitrm}, r'_{\splitrm}, K', H_{d'}) \in \sigma(d)$ be any child tree such that its fixed leaves $d'_{\fix}$ contain $d_{\fix}$ and $H_{d'} \geqslant H_d$. Then $b(d_{\rm equiv},\x,\y) \leqslant R(d',\x,\y)$, where $b(d_{\rm equiv},\x,\y)$ is defined in Equation (\ref{eq:lb_equiv}). 
\end{theorem}
\begin{proof}
The proof is similar to the proof of Theorem \ref{thm: lb_auc} and Theorem \ref{thm: hierarchy_auc}.
\end{proof}



\subsection{Partial area under the ROCCH}
We discuss the partial area under the ROC convex hull in this section. The ROCCH for a decision tree is defined in Section \ref{Sec:NewObjectives}, where leaves are rank-ordered by the fraction of positives in the leaves. 
Given a threshold $\theta$, the partial area under the ROCCH ($\rm pAUC_{\ch}$) looks at only the leftmost part of the ROCCH, that is focusing on the top ranked leaves. This measure is important for  applications such as information retrieval and maintenance \citep[see, e.g., ][]{RudinWa18}. In our implementation, all bounds derived for the objective $\AUC_{\ch}$ can be adapted directly for $\rm p\AUC_{\ch}$, where all terms are calculated only for false positive rates smaller or equal to $\theta$. 

In our code, we implement all of the rank statistics bounds, with one exception for the partial AUCCH -- the equivalent points bound. We do not implement the equivalent points bound for partial AUCCH since the $\rm p\AUC_{\ch}$ statistic is heavily impacted by the leaves with high fraction of positives, which means that the leaves being repeatedly calculated for the objective tend not to be impure, and thus the equivalence points bound is less effective.

\section{Optimizing F-score with Decision Trees} \label{app:Fscore}
For a labeled tree $d = (d_{\fix}, \delta_{\fix}, d_{\splitrm}, \delta_{\splitrm}, K, H)$, the F-score loss is defined as
\begin{equation}\label{eq:f-score}
    \ell(d,\x,\y)= \frac{FP+FN}{2N^++FP-FN}. 
\end{equation}
For objectives like accuracy, balanced accuracy and weighted accuracy, the loss of a tree is the sum of loss in each leaf. For F-score loss, however, $FP$ and $FN$ appear in both numerator and denominator, thus the loss no longer can be calculated using a sum over the leaves. 
\begin{lemma}\label{lemma:f1_label}
The label of a single leaf depends on the labels of other leaves when optimizing F-score loss. 
\end{lemma}
\begin{proof}
Let $l_1,...,l_{H_d}$ be the leaves of tree $d$. Suppose $H_d-1$ leaves are labeled. Let $FP_{H_d-1}$ and $FN_{H_d-1}$ be the number of false positives and false negatives of these $H_d-1$ leaves, respectively. Let $A=FP_{H_d-1}+FN_{H_d-1}$ and $B=2N^++FP_{H_d-1}-FN_{H_d-1}$, and by these definitions, we will always have $A \leq B$. Let $n^+_{H_d}$ be the number of positive samples in leaf $H_d$ and $n^-_{H_d}$ be the number of negative samples. The leaf's predicted label can be either positive or negative. The loss of the tree depends on this predicted label as follows: 
\begin{eqnarray}
    \text{If}\  \hat{y}^{(\leaf)}_{H_d} = 1, \textrm{there can be only false positives, thus} \ \  \label{eq:f_1} \ell(d,\x,\y) &=& \frac{A+n^-_{H_d}}{B+n^-_{H_d}}. \\
   \text{If} \ \hat{y}^{(\leaf)}_{H_d}=0, \textrm{there can be only false negatives, thus} \ \ \ell(d,\x,\y) &=& \frac{A+n^+_{H_d}}{B-n^+_{H_d}}. \label{eq:f_0}
\end{eqnarray}
Calculating loss (\ref{eq:f_0}) minus loss (\ref{eq:f_1}):  
\begin{equation} \label{eq:minus}
\frac{A+n^+_{H_d}}{B-n^+_{H_d}} - \frac{A+n^-_{H_d}}{B+n^-_{H_d}} = \frac{(A+n_{H_d}^+)\times(B+n_{H_d}^-) - (A+n_{H_d}^-)\times (B-n_{H_d}^+)}{(B-n^+_{H_d})(B+n^-_{H_d})}.
\end{equation}
Denote $\Delta$ as the numerator of (\ref{eq:minus}), that is 
$$(A+n_{H_d}^+)\times(B+n_{H_d}^-) - (A+n_{H_d}^-)\times (B-n_{H_d}^+).$$
Then we can get 
\begin{eqnarray}
\Delta&=&AB+An_{H_d}^- + Bn_{H_d}^+ + n_{H_d}^+n_{H_d}^- - AB+An_{H_d}^+ - Bn_{H_d}^- + n_{H_d}^-n_{H_d}^+ \\
&=& 2n^+_{H_d}n^-_{H_d}+An^+_{H_d}+Bn^+_{H_d}+An^-_{H_d}-Bn^-_{H_d}.
\end{eqnarray}
The value of $\Delta$ depends on $A$, $B$, $n_{H_d}^+$ and $n_{H_d}^-$. Hence, in order to minimize the loss, the predicted label of leaf $H_d$ is 0 if $\Delta \leq 0$ and 1 otherwise. Therefore, the predicted label of a single leaf depends on $A$ and $B$, which depend on the labels of the other samples, as well as the positive and negative samples captured by that leaf. 
\end{proof}

\begin{theorem}\label{thm:f1_label} (Optimizing F-score Poses a Unusual Challenge) Let $l_1,...,l_{H_d}$ be the leaves of tree $d$ and let $N^+$ be the number of positive samples in the dataset. Let $\Gamma_1$ and $\Gamma_2$ be two predicted labelings for the first $H_d-1$ leaves. Leaf $H_d$ has a fixed predicted label. Suppose the loss for the F-score (Equation (\ref{eq:f-score})) of the first $H_d-1$ leaves based on labeling method $\Gamma_1$  is smaller than the loss based on labeling method $\Gamma_2$ (where in both cases, leaf $H_d$ has the same predicted label). It is not guaranteed that the $\rm F_1$ loss of the tree $d$ based on the first labeling $\Gamma_1$ is always smaller than the loss based on the second labeling $\Gamma_2$. 

\end{theorem}
\begin{proof}
Let $FP_{H_d-1}^{(1)}$ and $FN_{H_d-1}^{(1)}$ be the number of false positives and number of false negatives for the first $H_d-1$ leaves from the labeling method $\Gamma_1$ and similarly define $FP_{H_d-1}^{(2)}$ and $FN_{H_d-1}^{(2)}$ for labeling method $\Gamma_2$. Denote $A_1 = FP_{H_d-1}^{(1)}+FN_{H_d-1}^{(1)}$ and $B_1 = 2N^+ + FP_{H_d-1}^{(1)}-FN_{H_d-1}^{(1)}$. Similarly, denote $A_2 = FP_{H_d-1}^{(2)}+FN_{H_d-1}^{(2)}$ and $B_2 = 2N^+ + FP_{H_d-1}^{(2)}-FN_{H_d-1}^{(2)}$. As we know from the assumptions of the theorem, $\frac{A_1}{B_1} \leq \frac{A_2}{B_2}$.

Denote $FP_{H_d}$ and $FN_{H_d}$ be the number of false positives and number of false negatives of the last leaf $l_{H_d}$. 

Let $\ell^{(1)}(d,\x,\y)$ and $\ell^{(2)}(d,\x,\y)$ be the loss of the tree $d$ based on two different predicted labelings of the leaves. 

Suppose the predicted label of leaf $l_{H_d}$ is 1. (An analogous result holds when the predicted label of leaf $l_{H_d}$ is 0.) Then $\ell^{(1)}(d,\x,\y) = \frac{A_1+FP_{H_d}}{B_1+FP_{H_d}}$ and $\ell^{(2)}(d,\x,\y)= \frac{A_2+FP_{H_d}}{B_2+FP_{H_d}}$. Let $\Delta$ be the numerator of $\ell^{(2)}(d,\x,\y) - \ell^{(1)}(d,\x,\y)$. 
\begin{eqnarray}\nonumber
    \Delta &=& (A_2+FP_{H_d})(B_1+FP_{H_d})-(A_1+FP_{H_d})(B_2+FP_{H_d})\\
    &=& A_2B_1-A_1B_2 + (A_2-B_2+B_1-A_1)FP_{H_d}.
\end{eqnarray}
Since $\frac{A_1}{B_1}\leq \frac{A_2}{B_2}$, $A_2B_1 \geq A_1B_2$, that is, the first two terms together are nonnegative. Meanwhile, $A_1 \leq B_1$ and $A_2 \leq B_2$. Thus, $\Delta$ could be negative or positive. Therefore, even though the labeling method $\Gamma_1$ leads to smaller loss for the first $H_d-1$ leaves and the label of the last leaf depends on the label of previous $H_d-1$ leaves, it is not guaranteed that the loss of the tree is smaller than that based on $\Gamma_2$. It is easy to construct examples of $A_1$, $A_2$, $B_1$, $B_2$, and $FP_{H_d}$ where the result is either positive or negative, as desired.
\end{proof}

Lemma \ref{lemma:f1_label} and Theorem \ref{thm:f1_label} indicate that optimizing F-score loss is much harder than other arbitrary monotonic losses such as balanced accuracy and weighted accuracy. Thus, we simplify the labeling step by incorporating a parameter $\omega$ at each leaf node s.t. $l_i$ is labeled as 0 if $\omega n^+_i \leqslant n^-_i$ and 1 otherwise $\forall i \in \{1,...,H_d\}$. 

\section{Dynamic Programming Formulation} 
\label{app:DynProgFormulation}
Note that this section describes standard dynamic programming, where possible splits describe subproblems. The more interesting aspects are the bounds and how they interact with the dynamic programming. 

We will work only with the weighted misclassification loss for the following theorem, so that the loss is additive over the data:
\[
\ell(d,\x,\y)=\sum_i \textrm{weight}_i \textrm{loss}(x_i,y_i).
\]

We denote $(\x, \y)$ as a data set of features $\x$ and binary labels $\y$ containing a total of $N$ samples and $M$ features. 

\textbf{Initial Problem}: We define a tree optimization problem as a minimization of the regularized risk $R(d, \x, \y)$ over the domain $\sigma(D)$, where $D$ is a tree consisting of a single split leaf
\[ D = (D_{\rm fix},r_{\rm fix},D_{\rm split},r_{\rm split},K,H) = (\emptyset,\emptyset,D_{\rm split},r_{\rm split},0,1) \]
\begin{equation}\label{eq:root}
d^* \in \argmin_{d \in \sigma(D)}R(d, \x, \y).
\end{equation}
Since all trees are descendants of a tree that is a single split leaf, this setup applies to tree optimization of any arbitrary data set $\x$, $\y$. We can rewrite the optimization problem as simply:
\begin{equation}\label{eq:rootfree}
d^* \in \argmin_{d}R(d, \x, \y).
\end{equation}

We partition the domain $d \in \sigma(D)$ into $M + 1$ cases: One Leaf Case and $M$ Tree Cases. We solve each case independently, then optimize over the solutions of each case:
\[ \sigma(D) = \text{Leaf} \cup \text{Tree}_1 \cup \text{Tree}_2 \cup \dots \cup \text{Tree}_M \]
\[ d^*_{\rm Leaf} \in \argmin_{d \in \sigma(\rm Leaf)}R(d, \x, \y) \]
\[ d^*_{\textrm{Tree}_i} \in \argmin_{d \in \sigma(\textrm{Tree}_i)}R(d, \x, \y) \]
\[ \argmin_{d \in \sigma(D)} \in \argmin_{d \in \{ d^*_{\rm Leaf}, d^*_{\rm Tree_1}, d^*_{\rm Tree_2}, \dots, d^*_{\rm Tree_M} \}}. \]
The Leaf Case forms a base case in a recursion, while each Tree Case is a recursive case that further decomposes into two instances of tree optimization of the form described in (\ref{eq:root}).

\textbf{Leaf Case}: In this case, $d^*_{\rm Leaf}$ is a tree consisting of a single fixed leaf. This tree's only prediction $r_{\rm fix}^*$ is a choice of two possible classes $\{0, 1\}$.
\[ r_{\rm fix}^* \in \argmin_{r_{\rm fix} \in \{\textrm{true}, \textrm{false}\}} R((d_{\rm fix},r_{\rm fix},\emptyset,\emptyset,1, 1), \x, \y), \]
\begin{equation}\label{eq:leafcase}
d^*_{\rm Leaf} = (d_{\rm fix},r_{\rm fix}^*,\emptyset,\emptyset,1, 1)
\end{equation}
where a tie would be broken randomly.

\textbf{Tree Case}: For every possible $i$ in the set feature indices $\{1, 2, 3, ..., M\}$ we designate an $i^{\rm th}$ Tree Case and an $d^*_{{\rm Tree}_i}$ as the optimal descendent of a tree $D^i$. We define $D^i$ as a tree consisting of a root split on feature $i$ and two resulting split leaves $d^{\rm Left}$ and $d^{\rm Right}$ so that:
\[ \text{Tree}_i = \sigma(D^i) \;\;\;\;\textrm{(the children of $D^i$)}\]
\[ D^i = (\emptyset,\emptyset,d_{\rm split},r_{\rm split},0, 2) = (\emptyset,\emptyset,\{d^{\rm Left}, d^{i}\},\{r^{-i}, r^{i}\},0, 2) \]
\begin{equation}\label{eq:treeicase}
d^*_{\textrm{Tree}_i} \in \mathrm{argmin}_{d \in \sigma(D^i)}R(d, \x, \y).
\end{equation}
Instead of directly solving (\ref{eq:treeicase}), we further decompose this into two smaller tree optimization problems that match the format of (\ref{eq:root}). Since we are working with the weighted misclassification loss, we can optimize subtrees extending from from $d^{-i}$ and $d^{i}$ independently. We define data within the support set of $d^{-i}$ as $\x^{-i}$, $\y^{-i}$. We define data within the support set of $d^{i}$ as $\x^{i}$, $\y^{i}$. For each $D^i$, we define an optimization over the extensions of the left split leaf $d^{-i}$ as:
\[ {\rm Left}^i = (\emptyset,\emptyset,\{d^{-i}\},r_{\rm split},0,1). \]
\begin{equation}\label{eq:treeleft}
d^{*}_{{\rm Left}^i} \in \mathrm{argmin}_{d \in \sigma({\rm Left}^i)}R(d, \x^{-i}, \y^{-i}).
\end{equation}
By symmetry, we define an optimization over the extensions of the right split leaf $d^{i}$:
\[ {\rm Right}^i = (\emptyset,\emptyset,\{d^{i}\},r_{\rm split},0,1) \]
\begin{equation}\label{eq:treeright}
d^{*}_{{\rm Right}^i} \in \mathrm{argmin}_{d \in \sigma({\rm Right}^i)}R(d, \x^{i}, \y^{i}).
\end{equation}
$d^{*}_{{\rm Left}^i}$ is the optimal subtree that classifies $\x^{-i}$ and $d^{*}_{{\rm Right}^i}$ is the optimal subtree that classifies $\x^{i}$. Together, with a root node splitting on the $i^{\rm th}$ feature, $d^{*}_{{\rm Left}^i}$ combines with $d^{*}_{{\rm Right}^i}$ to form $d^*_{\textrm{Tree}_i}$. Thus, we can solve (\ref{eq:treeleft}) and (\ref{eq:treeright}) to get the solution of (\ref{eq:treeicase}).

In (\ref{eq:root}) we defined a decomposition of an optimization problem over the domain $\sigma(D)$, where $D$ is a tree consisting of a single split leaf. Both expressions (\ref{eq:treeleft}) and (\ref{eq:treeright}) are also optimizations over the domain of children of a tree consisting of a single split leaf. Recall that the descendants of any tree consisting of only a single split leaf covers the space of all possible trees, therefore the trees are optimized over an unconstrained domain. We can thus rewrite (\ref{eq:treeleft}) as:
\begin{equation}\label{eq:treeleftfree}
d^{*}_{{\rm Left}^i} \in \mathrm{argmin}_{d}R(d, \x^{-i}, \y^{-i}).
\end{equation}
Symmetrically we can also rewrite (\ref{eq:treeright}) as:
\begin{equation}\label{eq:treerightfree}
d^{*}_{{\rm Right}^i} \in \mathrm{argmin}_{d}R(d, \x^{i}, \y^{i}).
\end{equation}
Observe that (\ref{eq:treeleftfree}) and (\ref{eq:treerightfree}) are simply tree optimization problems over a specific set of data (in this case $\x^{-i}, \y^{-i}$ and $\x^{i}, \y^{i}$). 
Hence,
these tree optimizations form a recursion, and each can be solved as though they were (\ref{eq:rootfree}).


\textbf{Termination}: To ensure this recursion terminates, we consider only splits where $\x^{-i}$ and $\x^{i}$ are strict subsets of $\x$. This ensures that the support strictly decreases until a minimum support is reached, which prunes all of $\text{Tree}_1 \cup \text{Tree}_2 \cup \dots \cup \text{Tree}_M$ leaving only the leaf case described in (\ref{eq:leafcase}).

\textbf{Identifying Reusable Work}: As we perform this decomposition, we identify each problem using its data set $\x$, $\y$ by storing a bit vector to indicate it as a subset of the initial data set. At each recursive step, we check to see if a problem has already been visited by looking for an existing copy of this bit vector.

Figure \ref{fig:dpgraph} shows a graphical representation of the algorithm. Note that we use the following shortened notations in the figure:
\begin{equation}\label{eq:dataset_positive}
(\x,\y)^{k} = (\x^{k},\y^{k})
\end{equation}
\begin{equation}\label{eq:dataset_negative}
(\x,\y)^{-k} = (\x^{-k},\y^{-k})
\end{equation}
\begin{equation}\label{eq:dataset_composed}
(\x,\y)^{k,l} = (\x^{k,l},\y^{k,l}).
\end{equation}
Equation (\ref{eq:dataset_positive}) denotes a data set $(\x,\y)$ filtered by the constraint that samples must respond positive to feature $k$.
Equation (\ref{eq:dataset_negative}) denotes a data set $(\x,\y)$ filtered by the constraint that samples must respond negative to feature $k$. 
Equation (\ref{eq:dataset_composed}) denotes a data set $(\x,\y)$ filtered by the constraint that samples must respond positive to both feature $k$ and feature $l$.

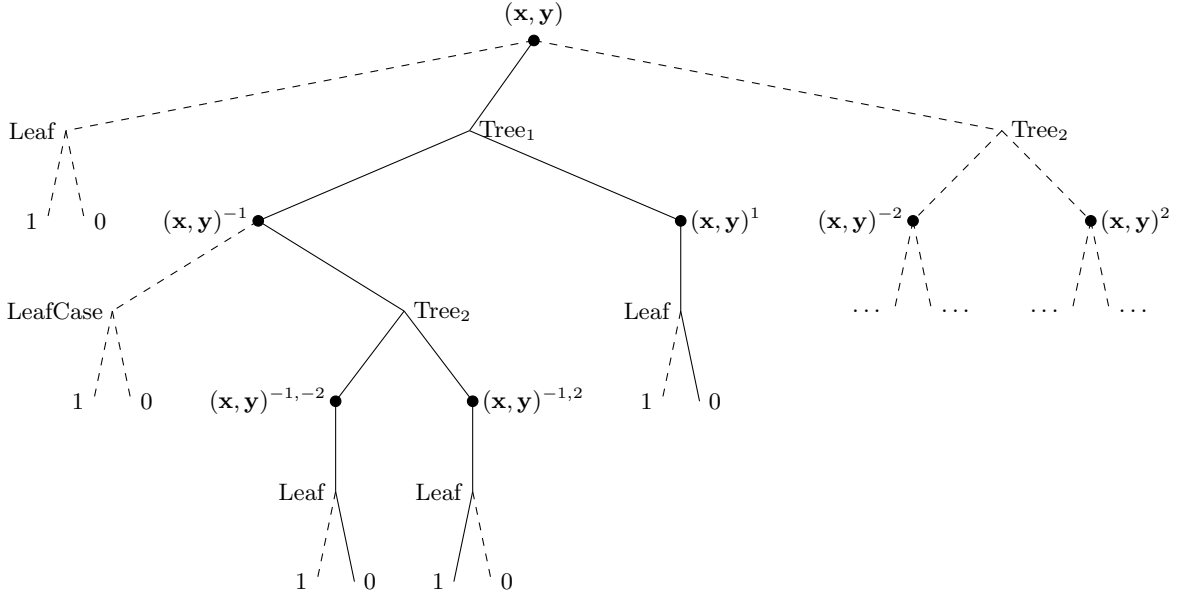
\begin{figure}[h]
  \centering

\begin{forest}
  anchors/.style={anchor=#1,child anchor=#1,parent anchor=#1},
  /tikz/every node/.append style={font=\footnotesize},
  for tree={
    s sep=5mm,l=12mm,
    if n=0{anchors=east}{
    if n=1{anchors=east}{anchors=west}},
     content format={$\forestoption{content}$},
  },
  anchors=south, outer sep=2pt,
  nodot/.style={content format={},draw=none},
  dot/.style={tikz+={\draw[#1](.anchor)circle[radius=2pt];}},
  if content={}{}{dot={fill}}
  [
    {(\x,\y)},dot=fill
    [
        {\rm Leaf},edge=dashed
        [{\rm 1},edge=dashed]
        [{\rm 0},edge=dashed]
      ]
    [
        {\rm Tree_1}
        [
            {(\x,\y)^{-1}},dot=fill
            [
                {\rm Leaf Case},edge=dashed
                [\rm 1,edge=dashed]
                [\rm 0,edge=dashed]
              ]
            [
                {\rm Tree_2}
                [
                    {(\x,\y)^{-1,-2}},dot=fill
                    [
                        {\rm Leaf}
                        [{\rm 1},edge=dashed]
                        [{\rm 0}]
                      ]
                  ]
                [
                    {(\x,\y)^{-1,2}},dot=fill
                    [
                        {\rm Leaf}
                        [{\rm 1}]
                        [{\rm 0},edge=dashed]
                      ]
                  ]
              ]
          ]
        [
            {(\x,\y)^{1}},dot=fill
            [
                {\rm Leaf}
                [{\rm 1},edge=dashed]
                [{\rm 0}]
              ]
          ]
      ]
    [
        {\rm Tree_2},edge=dashed
        [
            {(\x,\y)^{-2}},dot=fill,edge=dashed
            [{\dots},edge=dashed]
            [{\dots},edge=dashed]
          ]
        [
            {(\x,\y)^{2}},dot=fill,edge=dashed
            [{\dots},edge=dashed]
            [{\dots},edge=dashed]
          ]
      ]
  ]
\end{forest}

  \caption{Graphical representation of dependency graph Produced by this algorithm. Filled vertices show problems identified by a support set. Solid edges show one possible tree that can be extracted from the graph.}
  \label{fig:dpgraph}
\end{figure}

\section{Incremental Similar Support Bound Proof} 
\label{app:SimSupp}
We will work only with weighted misclassification loss for the following theorem, so that the loss is additive over the data:
\[
\ell(d,\x,\y)=\sum_i \textrm{weight}_i \textrm{loss}(x_i,y_i).
\]
Define the maximum possible weighted loss:
\[
\ell^{\max}=\max_{x,y} [\textrm{\rm weight}(x,y) \times \textrm{loss}(x,y)].
\]
The following bound is our important incremental similar support bound, which we leverage in order to effectively remove many similar trees from the search region by looking at only one of them.

\begin{theorem}\label{theorem:incrementalsimsupp} (Incremental Similar Support Bound) 
Consider two trees $d=(d_{\fix}, d_{\splitrm}, K, H)$ and $D=(D_{\fix}, D_{\splitrm}, K, H)$ that differ only by their root node (hence they share the same $K$ and $H$ values). Further, the root nodes between the two trees are similar enough that the support going to the left and right branches differ by at most $\omega$ fraction of the observations. (That is, there are $\omega N$ observations that are captured either by the left branch of $d$ and right branch of $D$ or vice versa.)  Define $S_{\uncertain}$ as the maximum of the support within $d_{\splitrm}$ and $D_{\splitrm}$:
\[
S_{\uncertain}=\max(\supp(d_{\splitrm}),\supp(D_{\splitrm})).
\]
For any child tree $d'$ grown from $d$ (grown from the nodes in $d_{\splitrm}$, that would not be excluded by the hierarchical objective lower bound) and for any child tree $D'$ grown from $D$ (grown from the nodes in $D_{\splitrm}$, that would not be excluded by the hierarchical objective lower bound), we have:
\[
|R(d',\x,\y)-R(D',\x,\y)|\leq (\omega +2S_{\uncertain}) \ell^{\max}. 
\]
\end{theorem}
This theorem tells us that any two child trees of $d$ and $D$ that we will ever generate during the algorithm will have similar objective values. The similarity depends on $\omega$, which is how many points are adjusted by changing the top split, and the other term involving $S_{\uncertain}$ is determined by how much of the tree is fixed. If most of the tree is fixed, then there can be little change in loss among the children of either $d$ or $D'$, leading to a tighter bound. In standard classification tasks, the value of $\ell^{\max}$ is usually 1, corresponding to a classification error for an observation.

\begin{proof}
We will proceed in three steps. The first step is to show that 
\[R(d,\x,\y)-R(D,\x,\y)\leq \omega \ell^{\max}.\]
The second step is to show:
\[ 
R(d,\x,\y)\leq R(d',\x,\y)+S_{\uncertain}\ell^{\max}, 
\]
for all feasible children $d'$ of $d$. The same bound will hold for $D$ and any of its children $D'$.
The third step is to show 
\[ 
R(d',\x,\y) \leq R(d,\x,\y)+S_{\uncertain}\ell^{\max}, 
\]
which requires different logic than the proof of Step 2. Together, Steps 2 and 3 give 
\[|R(d',\x,\y)-R(d,\x,\y)|\leq S_{\uncertain}\ell^{\max}.\]
From here, we use the triangle inequality and the bounds from the three steps to obtain the final bound:
\begin{eqnarray*}
\lefteqn{|R(d',\x,\y)-R(D',\x,\y)|}\\
&=& |R(d',\x,\y)-R(d,\x,\y)+R(d,\x,\y)-R(D,\x,\y)+R(D,\x,\y)-R(D',\x,\y)|\\
&\leq& |R(d',\x,\y)-R(d,\x,\y)|+|R(d,\x,\y)-R(D,\x,\y)|+|R(D,\x,\y)-R(D',\x,\y)|\\
&\leq& S_{\uncertain}\ell^{\max}  +\omega \ell^{\max} + S_{\uncertain}\ell^{\max},
\end{eqnarray*}
which is the statement of the theorem.
Let us now go through the three steps.

\textbf{First step}:
Define ``$\move$'' as the set of indices of the observations that either go down the left branch of the root of $d$ and the right of $D$, or that go down the right of $d$ and the left of $D$. The remaining data will be denoted ``$/\move$.'' These remaining data points will be classified the same way by both $d$ and $D$. The expression above follows from the additive form of the objective $R$:
\begin{eqnarray*}
R(d,\x,\y) &=& \ell(d,\x^{\move},\y^{\move})+\ell(d,\x^{/\move},\y^{/\move}) +\lambda H,\\
R(D,\x,\y) &=& \ell(D,\x^{\move},\y^{\move})+\ell(D,\x^{/\move},\y^{/\move}) +\lambda H,
\end{eqnarray*}
and since $\ell(d,\x^{/\move},\y^{/\move}) = \ell(D,\x^{/\move},\y^{/\move})$ since this just considers overlapping leaves, we have:
\[
|R(d,\x,\y) - R(D,\x,\y)|\leq |\ell(d,\x^{\move},\y^{\move}) - \ell(D,\x^{\move},\y^{\move})| \leq \omega \ell^{\max}.
\]
(For the last inequality, the maximum is attained when one of $\ell(d,\x^{\move},\y^{\move})$ and $\ell(D,\x^{\move},\y^{\move})$ is zero and the other attains its maximum possible value.)

\textbf{Second step}:
Recall that $d'$ is a child of $d$ so that $d'_{\fix}$ contains $d_{\fix}$. Let us denote the leaves in $d'_{\fix}$ that are not in $d_{\fix}$ by $d'_{\fix}/d_{\fix}$. Then,
\begin{eqnarray*}
R(d',\x,\y)&=& 
\ell(d'_{\fix},\x,\y) 
+ \ell(d'_{\splitrm},\x,\y)+\lambda H_{d'}\\
&=&\ell(d_{\fix},\x,\y) +\ell(d'_{\fix}/d_{\fix},\x,\y) + \ell(d'_{\splitrm},\x,\y) + \lambda H_{d'}. 
\end{eqnarray*}
Adding $\ell(d_{\splitrm},\x,\y)$ to both sides,
\begin{eqnarray*}
R(d',\x,\y) + \ell(d_{\splitrm},\x,\y) &=& \ell(d_{\fix},\x,\y) + \ell(d_{\splitrm},\x,\y)+\ell(d'_{\fix}/d_{\fix},\x,\y) + \ell(d'_{\splitrm},\x,\y) + \lambda H_{d'}\\
&=& R(d,\x,\y)+\ell(d'_{\fix}/d_{\fix},\x,\y) + \ell(d'_{\splitrm},\x,\y) + \lambda(H_{d'}-H)\\
&\geq& R(d,\x,\y),
\end{eqnarray*}
since the terms we removed were all nonnegative. Now,
\begin{eqnarray*}
R(d,\x,\y)&\leq& R(d',\x,\y)+ \ell(d_{\splitrm},\x,\y)\\
&\leq& R(d',\x,\y)+ \supp(d_{\splitrm})\ell^{\max}\\
&\leq& R(d',\x,\y)+ S_{\uncertain}\ell^{\max}.
\end{eqnarray*}

\textbf{Third step}: Here we will use the hierarchical objective lower bound. We start by noting that since we have seen $d$, we have calculated its objective $R(d,\x,\y)$, and it must be as good or worse than than the current best value that we have seen so far (or else it would have replaced the current best). So $R^c \leq R(d,\x,\y)$.
The hierarchical objective lower bound (Theorem \ref{thm:hierarchy}) would be violated if the following holds. This expression states that $b(d',\x,\y)$ is worse than $R(d,\x,\y)$ (which is worse than the current best), which means we would have already excluded $d'$ from consideration:
\[
R^c \leq R(d,\x,\y)
< b(d',\x,\y) =R(d',\x,\y) - \ell(d'_{\splitrm},\x,\y).
\]
This would be a contradiction. Thus, the converse holds:
\[
R(d',\x,\y) - \ell(d'_{\splitrm},\x,\y) =
b(d',\x,\y) \leq R(d,\x,\y). 
\]
Thus,
\[
R(d',\x,\y) \leq R(d,\x,\y) + \ell(d'_{\splitrm},\x,\y). 
\]
Now to create an upper bound for $\ell(d'_{\splitrm},\x,\y)$ as $\ell^{\max}$ times the support of $d'_{\splitrm}$. Since $d'$ is a child of $d$, its support on the split leaves is less than or equal to that of $d$, $\supp(d'_{\splitrm})\leq \supp(d_{\splitrm})$. Thus, $\ell(d'_{\splitrm},\x,\y)\leq \ell^{\max} \supp(d_{\splitrm})\leq \ell^{\max} S_{\uncertain}$.
Hence, we have the result for the final step of the proof, namely:
\[ 
R(d',\x,\y) \leq R(d,\x,\y)+S_{\uncertain}\ell^{\max}. 
\]

\end{proof}

\section{Subset Bound Proof} \label{app:subset_bound}
We will work with the loss that is additive over the data for the following theorem. The following bound is our subset bound, which we leverage in order to effectively remove the thresholds introduced by the continuous variables, thus pruning the search space. 

\begin{theorem}(Subset Bound).
Define $d=(d_{\fix}, \delta_{\fix}, d_{\splitrm}, \delta_{\splitrm}, K, H)$ and $D=(D_{\fix}, \Delta_{\fix}, D_{\splitrm},\Delta_{\splitrm})$ to be two trees with same root node. Let $f_1$ and $f_2$ be the features used to split the root node. Let $t_1$, $t_2$ be the left and right sub-trees under the root node split by $f_1$ in $d$ and let $(\x_{t_1}, \y_{t_1})$ and $(\x_{t_2}, \y_{t_2})$ be the samples captured by $t_1$ and $t_2$ respectively. Similarly, let $T_1$, $T_2$ be the left and right sub-trees under the root node split by $f_2$ in $D$ and let $(\x_{T_1}, \y_{T_1})$ and $(\x_{T_2}, \y_{T_2})$ be the samples captured by $T_1$ and $T_2$ respectively. Suppose $t_1$, $t_2$ are the optimal trees for $(\x_{t_1}, \y_{t_1})$ and $(\x_{t_2}, \y_{t_2})$ respectively, and $T_1$, $T_2$ are the optimal trees for corresponding $(\x_{T_1}, \y_{T_1})$ and $(\x_{T_2}, \y_{T_2})$. If $R(t_1, \x_{t_1}, \y_{t_1}) \leq R(T_1, \x_{T_1}, \y_{T_1})$ and $(\x_{t_2}, \y_{t_2}) \subseteq (\x_{T_2}, \y_{T_2})$, then $R(d, \x, \y) \leq R(D, \x, \y)$. 
\end{theorem} 

This theorem tells us that when $f_1$ and $f_2$ are from different thresholds of a continuous variable, for example $\text{age} \leq 20$ and $\text{age} \leq 18$, $(\x_{t_2}, \y_{t_2}) \subseteq (\x_{T_2}, \y_{T_2})$ is always true. In this case, we need only develop and compare the two left sub-trees. 


\begin{proof}
Since $(\x_{t_2}, \y_{t_2}) \subseteq (\x_{T_2}, \y_{T_2})$ and $t_2$ and $T_2$ are optimal trees for $(\x_{t_2}, \y_{t_2})$ and $(\x_{T_2}, \y_{T_2})$ respectively, 
$$R(t_2, \x_{t_2}, \y_{t_2}) \leq R(T_2, \x_{t_2}, \y_{t_2}) \leq R(T_2, \x_{T_2}, \y_{T_2}).$$
Given $R(t_1, \x_{t_1}, \y_{t_1}) \leq R(T_1, \x_{T_1}, \y_{T_1})$,
\begin{equation}
    \begin{aligned}
    R(t_1, \x_{t_1}, \y_{t_1}) + R(t_2, \x_{t_2}, \y_{t_2}) &\leq  R(T_1, \x_{T_1}, \y_{T_1}) + R(T_2, \x_{T_2}, \y_{T_2})\\
    R(d, \x, \y) &\leq R(D, \x, \y).\\
    \end{aligned}
\end{equation}
\end{proof}

\section{Complexity of Decision Tree Optimization}
In this section, we show the complexity of decision tree optimization. 
\begin{theorem}\label{thm:complexity}
Let $f(M)$ be the number of binary decision trees that can be constructed for the dataset with $N$ samples, $M$ binary features and $K$ label classes. The complexity of $f(M)$ is $O(M!)$.
\end{theorem}

\begin{proof}
We show the proof by induction. 

\textbf{Base Case:}

When $M=0$, the dataset cannot be split and a predicted label is assigned to all samples. Therefore, the number of binary decision trees for the given dataset is equal to the number of classes $K$. Hence, $f(0) = K$. 

\textbf{Inductive Step:}

When $M > 0$, the whole dataset can be split into two subsets by any of the $M$ features. With $M$ possible ways to do the first split, $2M$ subsets of data are created. Since binary features only have two different values, 0 and 1, a binary feature used to produce the subsets cannot be used again. Therefore, each of the $2M$ subsets can only be separated using trees constructed from at most $M-1$ binary features. This produces a recursive definition of $f(M)$. That is $f(M) = 2M  f(M-1)$.

\textbf{Combining Step:}

Each inductive step reduces $M$ by one, guaranteeing its arrival at the base case. By combining the base case with the inductive step, a non-recursive definition of $f(M)$ is produced. 
\begin{equation}
    \begin{aligned}
    f(0) &= K\\
    f(1) &= 2\times 1 \times f(0) = 2^1\times 1 \times K\\ 
    f(2) &= 2\times 2\times f(1) = 2^2\times2 \times 1 \times K\\
    f(3) &= 2\times 3\times f(2) = 2^3\times 3\times 2\times 1\times K\\
    \vdots &\\
    f(M) &=2\times M\times f(M-1) = 2^M\times M! \times K
    \end{aligned}
\end{equation}
Since the term with the highest complexity in $(2^M)(M!)(K)$ is factorial, the complexity of $f(M)$ is $O(M!)$
\end{proof}

\section{Experiments}\label{app:exp}
\label{App:Experiments}
In this section, we elaborate on the experimental setup, data sets, pre-processing and post-processing. Additionally, we present extra experimental results that were omitted from the main paper due to space constraints.

\subsection{Data Sets}
We used a total of 11 data sets: Seven of them are from the UCI Machine Learning Repository \citep{Dua:2017}, \textit{(monk1, monk2, monk3, tic-tac-toe, balance scale, car evaluation, iris)}, one from LIBSVM \citep{libsvm}, \textit{(FourClass)}. The other three datasets are the ProPublica recidivism dataset \citep{LarsonMaKiAn16} \textit{(COMPAS)}, the Fair Isaac credit risk data sets \citep{competition} \textit{(FICO)}, and the mobile advertisement data sets \citep{WangEtAl2017} \textit{(coupon)}. We predict which individuals are arrested within two years of release (N = 5,020) on the recidivism data set, whether an individual will default on a loan for the FICO dataset, and whether a customer is going to accept a coupon for a bar considering demographic and contextual attributes. 

\subsection{Preprocessing}

\textbf{Missing Values}: We exclude all observations with missing values.

\textbf{monk 1}, \textbf{monk 2}, \textbf{monk 3}, \textbf{tic-tac-toe}, \textbf{balance scale}, and \textbf{car evaluation}: We preprocessed these data sets, which contain only categorical features, by using a binary feature to encode every observable categorical value.

\textbf{iris}: We encode a binary feature to represent every threshold between adjacent values for all four continuous features (\textit{sepal length}, \textit{sepal width}, \textit{petal length}, \textit{petal width}). From the 3-class classification problem, we form 3 separate binary classification problems. Each binary classification is 1 of the 3 classes against the remaining classes. These three problems are referred to as \textbf{iris-setosa}, \textbf{iris-versicolor}, and \textbf{iris-virginica}.

\textbf{Four Class}: This dataset contains simulated points in a two-dimensional, bounded space with two classes that have irregular spreads over the space (241 positive samples and 448 negative samples). We split two continuous features into six categories (e.g. ${\rm feature 1}\leq 50$ and ${\rm feature 1}\leq 100$) and the value for each column is either 0 or 1.  

\textbf{ProPublica Recidivism (COMPAS)}: We discretized each continuous variable by using a binary feature to encode a threshold between each pair of adjacent values.

\textbf{ProPublica Recidivism (COMPAS-2016)}: We selected features \textit{age, count of juvenile felony, count of juvenile misdemeanor, count of juvenile other crimes, count of prior crimes}, and the target \textit{recidivism within two years}. We replace \textit{count of juvenile felony, count of juvenile misdemeanor, count of juvenile other crimes} with a single count called \textit{count of juvenile crimes} which is the sum of \textit{count of juvenile felony, count of juvenile misdemeanor, count of juvenile other crimes}.
We discretized the \textit{count of prior crimes} into four ranges \textit{count of prior crimes = 0}, \textit{count of prior crimes = 1}, \textit{count of prior crimes between 2 to 3}, and \textit{count of prior crimes $>$ 3}. These four ranges are each encoded as binary features. Therefore, after preprocessing, these data contain 2 continuous features, 4 binary features, and one target. 

\textbf{ProPublica Recidivism (COMPAS-binary)}: We use the same discretized binary features of \textbf{compas} produced in \cite{HuRuSe2019} which are the following: \textit{sex = Female, age $<$ 21, age $<$ 23, age $<$ 26, age $<$ 46, juvenile felonies = 0, juvenile misdemeanors = 0, juvenile crimes = 0, priors = 0, priors = 1, priors = 2 to 3, priors $>$ 3.}

\textbf{Fair Isaac Credit Risk (FICO)}: We discretized each continuous variable by using a binary feature to encode a threshold between each pair of adjacent values.

\textbf{Fair Isaac Credit Risk (FICO-binary)}: We use the same discretized binary features of \textbf{compas} produced in \cite{HuRuSe2019} which are the following: \textit{External Risk Estimate $<$ 0.49 , External Risk Estimate $<$ 0.65, External Risk Estimate $<$ 0.80, Number of Satisfactory Trades $<$ 0.5, Trade Open Time $<$ 0.6, Trade Open Time $<$ 0.85, Trade Frequency $<$ 0.45, Trade Frequency $<$ 0.6, Delinquency $<$ 0.55, Delinquency $<$ 0.75, Installment $<$ 0.5, Installment $<$ 0.7, Inquiry $<$ 0.75, Revolving Balance $<$ 0.4, Revolving Balance $<$ 0.6, Utilization $<$ 0.6, Trade W. Balance $<$ 0.33.}

\textbf{Mobile Advertisement (coupon)}: We discretized each continuous variable by using a binary feature to encode a threshold between each pair of adjacent values. We discretized each categorical variable by using a binary feature to encode each observable categorical value.

\textbf{Mobile Advertisement (bar-7)}: In order to predict whether a customer is going to accept a coupon for a bar, we selected features \textit{age, passengers, bar, restaurant20to50, direction same}, and target. For features \textit{age, bar} and \textit{direction} we used the same encoding as we used for \textbf{coupon}.
We modified \textit{passengers} so that it is replaced with the binary feature \textit{passengers $>$ 0}. We modified \textit{restaurant20to50} so that it is replaced with the binary feature \textit{restaurant20to50=0}, which is 0 if the number of times that they eat at a restaurant with average expense less than \$20 per person is less than 4 times per month and 1 otherwise.

\textbf{Data Set Summary}: Table \ref{table:datasetsummary} presents a summary of the datasets.\\
\begin{table}[h]
\centering
\begin{tabular}{|c c c c c|} 
 \hline
 \textbf{Data Set} & \textbf{Samples} & \textbf{Categorical Features} & \textbf{Continuous Features} & \textbf{Encoded Binary Features} \\
 \hline\hline
 monk 1 & 124 & 11 & 0 & 11 \\ 
 \hline
 monk 2 & 169 & 11 & 0 & 11 \\
 \hline
 monk 3 & 122 & 11 & 0 & 11 \\
 \hline
 car evaluation & 1729 & 15 & 0 & 15 \\
 \hline
 balance scale & 625 & 16 & 0 & 16 \\
 \hline
 tic-tac-toe & 958 & 18 & 0 & 18 \\
 \hline
 iris & 151 & 0 & 4 & 123 \\
 \hline
 iris-setosa & 151 & 0 & 4 & 123 \\
 \hline
 iris-versicolor & 151 & 0 & 4 & 123 \\
 \hline
 iris-virginica & 151 & 0 & 4 & 123 \\
 \hline
 coupon & 12684 & 26 & 0 & 129 \\
 \hline
 bar-7 & 1913 & 5 & 0 & 14 \\
 \hline
 Four Class & 862 & 0 & 2 & 345 \\
 \hline
 COMPAS & 12382 & 2 & 20 & 647 \\
 \hline
 COMPAS-2016 & 5020 & 4 & 2 & 85 \\
 \hline
 COMPAS-binary & 6907 & 12 & 0 & 12 \\
 \hline
 FICO & 1000 & 0 & 23 & 1407 \\
 \hline
 FICO-binary & 10459 & 17 & 0 & 17 \\
 \hline

\end{tabular}
\caption{Comparison of different data sets (and their preprocessed derivatives). \textbf{Categorical Features} denote the number of features that are inherently categorical or represent discrete ranges of a continuous feature. These features generally have a low cardinality. \textbf{Continuous Features} denote the number of features that take on many integer or real number values. These features generally have a high cardinality. \textbf{Encoded Binary Features} denote the number of binary features required to encode all categorical and continuous values observed in the dataset. This is value is approximately the total cardinality of all categorical and continuous features. \label{table:datasetsummary}}
\end{table}

\subsection{Optimization Algorithms}

\textbf{CART}: We run CART as a reference point for what is achievable with a greedy algorithm that makes no optimality guarantee. The algorithm is run using the Python implementation from Sci-Kit Learn. The depth and leaf count are constrained in order to adjust the resulting tree size.

\textbf{BinOCT}: BinOCT is modified to run using only a single thread to make comparison across algorithms fair. This algorithm runs on the academic version of CPLEX. The depth is constrained to adjust the resulting tree size.

\textbf{DL8.5}: DL8.5 is implemented in C++ and is run as a native extension of the Python interpreter. The depth is constrained to adjust the resulting tree size.

Because BinOCT and DL8.5 have hard constraints, rather than OSDT or GOSDT's soft constraints, GOSDT and OSDT's optimization problem is substantially harder than that of BinOCT and DL8.5. GOSDT and OSDT effectively consider a large number of possible tree sizes whereas the other algorithms consider only full trees of a given depth.

\textbf{OSDT}: OSDT is implemented in Python and run directly. The regularization coefficient is varied to adjust the resulting tree size. 

\textbf{PyGOSDT}: PyGOSDT is an early Python implementation of GOSDT. The regularization coefficient is varied to adjust the resulting tree size.

\textbf{GOSDT}: GOSDT is implemented in C++ and run as a native executable. The regularization coefficient is varied to adjust the resulting tree size.

\subsection{Computing Infrastructure}
The experiments for optimizing rank statistics were run on a personal laptop with a 2.6GHz i5-7300U processor and 8GB of RAM.

All other experiments were run on a 2.30 GHz (30 MB L3 cache) Intel Xeon E7-4870 v2 processor with 60 cores across 4 NUMA nodes. We disabled hyper-threading. The server has 512 GB RAM uniformly distributed (128 GB each) across the four NUMA nodes. The host OS is Ubuntu 16.04.6 LTS. We set a 5-minute time limit on all experiments, unless otherwise stated. All algorithms that support multi-threading are modified to run sequentially.

\subsection{Experiments: Rank Statistics}
\textbf{Collection and Setup}:
We ran this experiment on the data set \textbf{FourClass}. We train models to optimize various objectives with 30 minute time limits. When the time limit is reached, our algorithm returns the current best tree considering the objectives. 

\textbf{Results}: Figure \ref{fig:fc_rank_statistics} shows the training ROC and test ROC of decision trees generated for six different objectives. Optimizing different objectives produces different trees with different $FP$ and $FN$. Some interesting observations are that the pAUC$_{\ch}$ model performs as well as the AUC$_{\ch}$ model on the left part of the ROC curve, but then sacrifices some area under the middle and right of the curve (which is not as relevant to its objective) to obtain a sparser model (sparsity is relevant to the objective). The  pAUC$_{\ch}$ and AUC$_{\ch}$ results illustrate how the objective allows us to trade off parts of the ROC curve (that are not important for some applications) with overall sparsity. Another interesting observation is that some of the models are extremely sparse: recall that each leaf is a single diagonal line on the ROC curve, so one can count the number of leaves by looking at the number of diagonal lines. In some cases, a well-chosen single split point can lead to a model with an excellent TPR/FPR tradeoff somewhere along the ROC curve.

\begin{figure}[h]
    \centering
    \includegraphics[scale=0.5]{figure/roc_train.png}
    \includegraphics[scale=0.5]{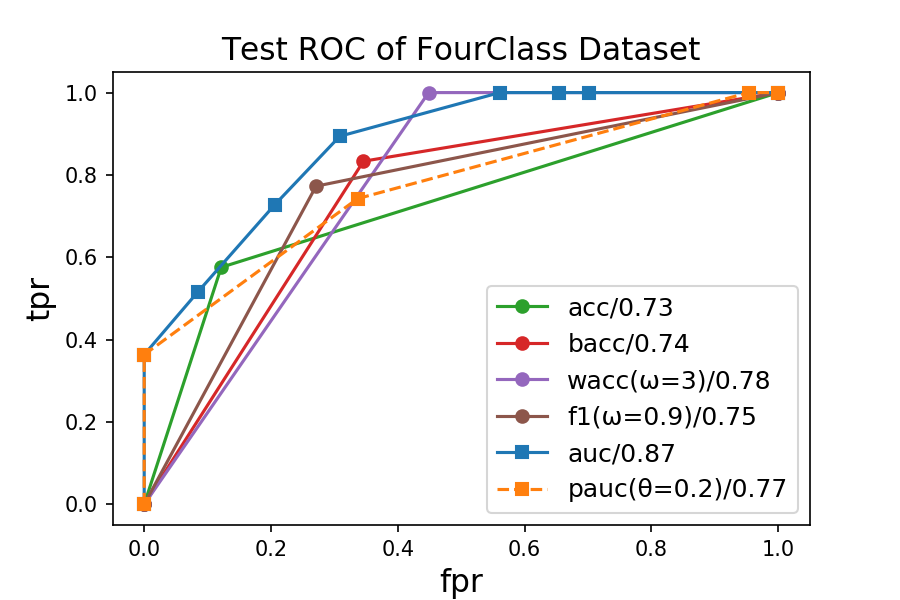}
    \caption{Training ROC and test ROC of FourClass dataset ($\lambda=0.01$). A/B in the legend at the bottom right of each subfigures shows the objective and its parameter/area under the ROC. }
    \label{fig:fc_rank_statistics}
\end{figure}

\subsection{Experiment: Accuracy vs Sparsity}
\label{exp:accspar}
\textbf{Collection and Setup}:
We ran this experiment on the 6 data sets \textbf{car evaluations, COMPAS-binary, tic-tac-toe, monk 1, monk 2}, and \textbf{monk 3}. For the monk data sets, we used only the samples from the training set. For each data set we train models using varying configurations (described in the following sections) to produce models with varying number of leaves. For any single configuration, we perform a 5-fold cross validation to measure training accuracy and test accuracy for each fold. All runs that exceed the time limit of 5 minutes are discarded.

\textit{We omit PyGOSDT since it differs only from GOSDT in program speed, and would provide no additional information for this experiment.}

Below are the configurations used for each algorithm:
\begin{itemize}
    \item \textit{CART} Configurations: We ran this algorithm with 6 different configurations: depth limits ranging from 1 to 6, and a corresponding maximum leaf limit of 2, 4, 8, 16, and 64.
    \item \textit{BinOCT} and \textit{DL8.5} Configurations: We ran these algorithms with 6 different configurations: depth limits ranging from 1 to 6.
    \item \textit{OSDT} and \textit{GOSDT} Configurations: We ran these algorithms with 29 different regularization coefficients: 0.2, 0.1, 0.09, 0.08, 0.07, 0.06, 0.05, 0.04, 0.03, 0.02, 0.01, 0.009, 0.008, 0.007, 0.006, 0.005, 0.004, 0.003, 0.002, 0.001, 0.0009, 0.0008, 0.0007, 0.0006, 0.0005, 0.0004, 0.0003, 0.0002, and 0.0001.
\end{itemize}

\textbf{Calculations}:
For each combination of data set, algorithm, and configuration, we produce a set of up to 5 models, depending on how many runs exceeded the time limit. We summarize the measurements (e.g., training accuracy, test accuracy, and number of leaves) across the set of up to 5 models by plotting the median. We compute the $25^{th}$ and $75^{th}$ percentile and show them as lower and upper error values respectively.

\textbf{Results}:
Figure \ref{fig:train_vs_leaf} shows that the objective optimized by GOSDT (same as OSDT) reliably produces a more efficient frontier between training accuracy and number of leaves.
Figure \ref{fig:test_vs_leaf} shows the same plots with test accuracy and number of leaves. The difference between frontiers sometimes becomes insignificant due to error introduced from generalization, particularly when the training accuracies between algorithms were close together. That is, if CART achieves a solution that is almost optimal, then it tends to achieve high test accuracy as well. \textit{Without methods like GOSDT or OSDT, one would not be able to determine whether CART's training solution is close to optimal for a given number of leaves.} Further, if the training accuracies of the different algorithms are different (e.g., as in the monk2 data), this difference is reflected in an improved test accuracy for OSDT or GOSDT.

\begin{figure}[h]
  \centering
  \includegraphics[scale=0.291]{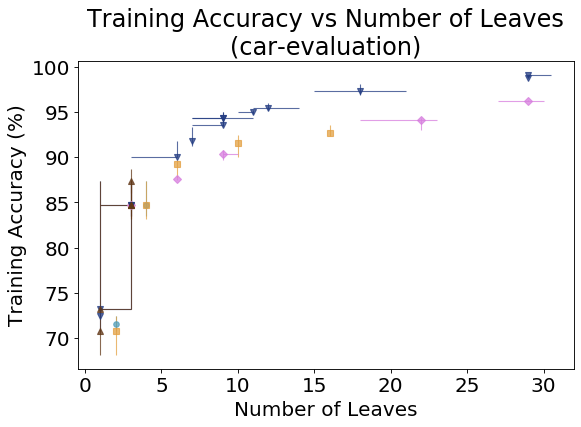}
  \includegraphics[scale=0.291]{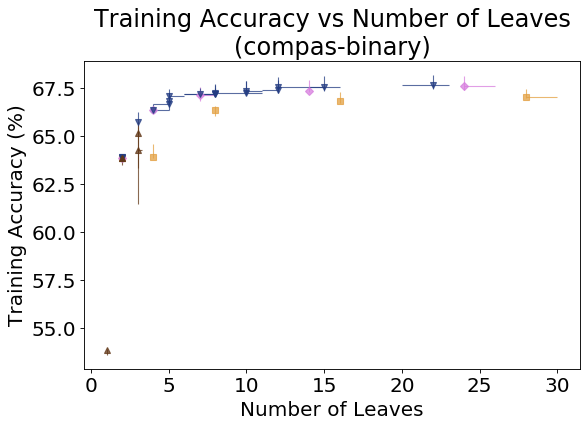}
  \includegraphics[scale=0.291]{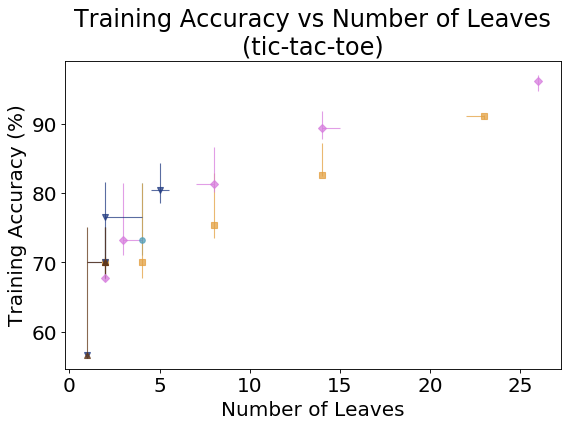}
  \includegraphics[scale=0.291]{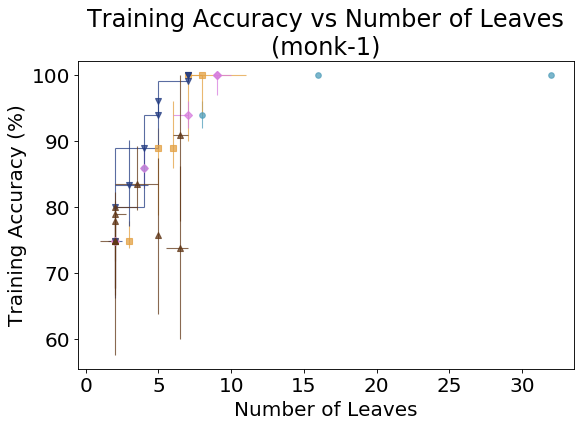}
  \includegraphics[scale=0.291]{figure/experiments/tradeoff/training_accuracy_vs_leaves_monk-2.png}
  \includegraphics[scale=0.291]{figure/experiments/tradeoff/training_accuracy_vs_leaves_monk-3.png}
  \caption{Training accuracy achieved by BinOCT, CART, DL8.5, GOSDT, and OSDT as a function of the number of leaves.  }
  \label{fig:train_vs_leaf}
\end{figure}

\begin{figure}[h]
  \centering
  \includegraphics[scale=0.291]{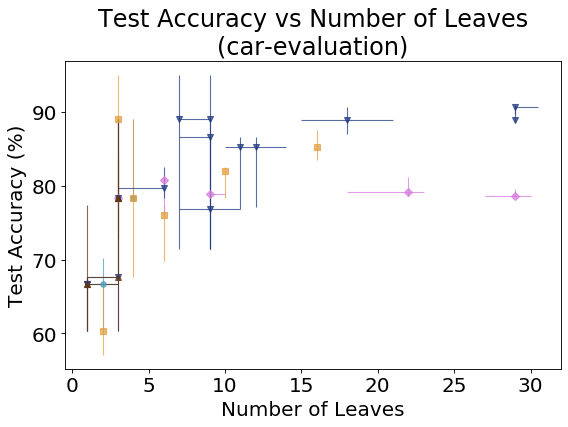}
  \includegraphics[scale=0.291]{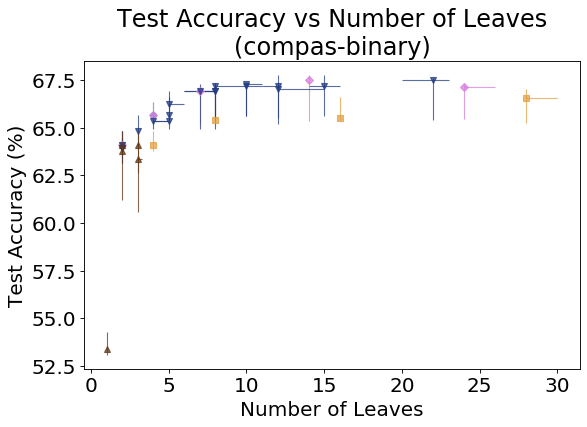}
  \includegraphics[scale=0.291]{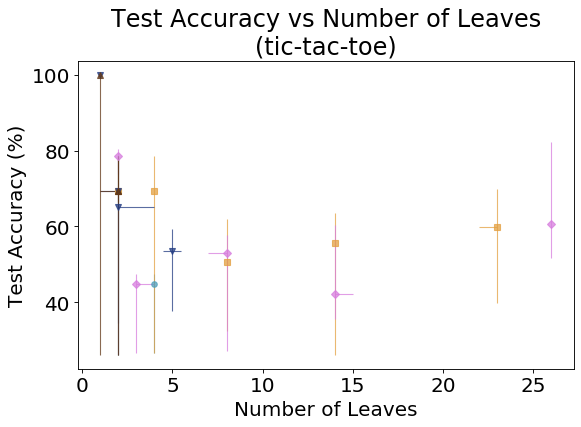}
  \includegraphics[scale=0.291]{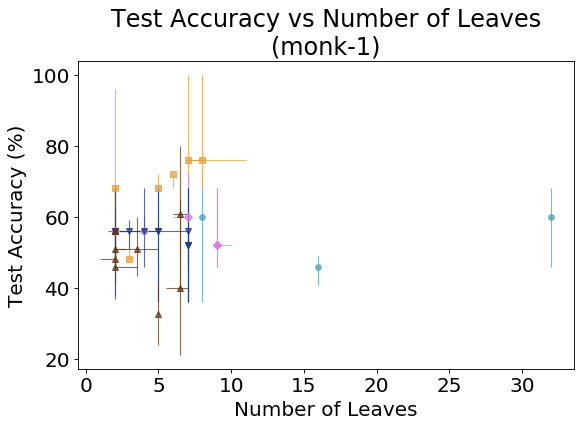}
  \includegraphics[scale=0.291]{figure/experiments/tradeoff/test_accuracy_vs_leaves_monk-2.png}
  \includegraphics[scale=0.291]{figure/experiments/tradeoff/test_accuracy_vs_leaves_monk-3.png}
  \caption{Test accuracy achieved by BinOCT, CART, DL8.5, GOSDT, and OSDT as a function of the number of leaves.
  \label{fig:test_vs_leaf}}
\end{figure}

\subsection{Experiment: Scalability}

\textbf{Collection and Setup}:
We ran this experiment on 4 data sets: \textbf{bar-7, compas-2016, compas}, and \textbf{fico}. The four data sets vary in the number of binary features required to fully represent their information. The number of binary features are, respectively, 14, 85, 647, and 1407. For each data set we show runtime as a function of both sample size and number of binary features used. 

Each data set is preprocessed so that categorical features produce one binary feature per unique value, and continuous features produce one binary feature per pair of adjacent values. The samples are then randomly shuffled. We measure run time on increasingly larger subsets of this data (with all binary features included), this is our measure of run time as a function of sample size. We also measure run time on increasingly larger numbers of binary features (with all samples included), this is our measure of run time as a function of binary features. For all experiments we continue increasing the difficulty until either the difficulty covers the full data set or a time limit of 5 minutes has been exceeded 3 times by the same algorithm.

Note that when varying the number of binary features, we include all samples. This means that adding a feature to a large data set (e.g., COMPAS and FICO) generally increases the difficulty more than adding a feature to a small data set (e.g., bar-7 and COMPAS-2016). Likewise, when varying the number of samples, we include all binary features. This means that adding a feature to a high-dimensional data set (e.g., COMPAS and FICO) generally increases the difficulty more than adding a sample to a low-dimensional data set (e.g., bar-7 and COMPAS-2016). As a result, the sample size is not a good measure of difficulty when comparing across different data sets of completely different features. The number of binary features is a more robust measure when comparing across different data sets.

Below are the configurations used for each algorithm tested:
\begin{itemize}
    \item \textit{CART} is configured to have a maximum depth of $\log_2(32)$ and a maximum leaf count of 32.
    \item \textit{BinOCT} and \textit{DL8.5} are configured to have a maximum depth of $\log_2(32)$.
    \item \textit{OSDT} and \textit{GOSDT} are configured with a regularization coefficient of $\frac{1}{32}$.
\end{itemize}

While we initially attempted to include BinOCT in this experiment, we were unable to find an instance where BinOCT reached completion with a maximum depth of $\log_2(32)$ and a time limit of 5 minutes. Consequently, BinOCT was not included in this experiment.

\textbf{Calculations}:
We provide two measures of speed. Training time measures the number of seconds required for an algorithm to complete with a certificate of optimality. Slow-down measures the ratio of the algorithm's training time against its fastest training time over values of problem difficulty.
We vary and measure problem difficulty in two separate ways. ``Number of binary features'' indicates how many of the binary features generated by our binary encoding were included for training. ``Number of samples'' indicates how many samples were included for training.

\textbf{Results}:
Figure \ref{fig:time_vs_feature} shows how each algorithm's training time varies as additional binary features are included. Figure \ref{fig:time_vs_sample} shows how each algorithm's training time varies as additional samples are included.

For \textit{bar-7} and \textit{compas-2016}, we observe a logarithmic time complexity when increasing sample size. These problems are sufficiently represented and solvable at a small sample size. As a result, additional samples contribute diminishing increase in the difficulty of the problem. Under these circumstances, GOSDT, PyGOSDT, and OSDT have a significant performance advantage over DL8.5.

For all data sets we observe an approximately factorial time complexity when increasing the number of features. This is consistent with the theoretical worst-case time complexity of full tree optimization (see Theorem \ref{thm:complexity}). The sharp increase in run time results in a limit on the size of problems solvable in practice by each algorithm. We observe that while all full tree optimization algorithms have such a limit, GOSDT usually has a higher limit than other algorithms.

Figure \ref{fig:slow_vs_feature} show how each algorithm's relative slow-down varies with additional binary features. Figure \ref{fig:slow_vs_sample} show how each algorithm's relative slow-down varies with additional samples. This reduces the effects of constant overhead, showing the asymptotic behavior of each algorithm. Our observations from Figure \ref{fig:time_vs_feature} and Figure \ref{fig:time_vs_sample} still hold under this analysis. Additionally, we observe that the slow-down of GOSDT and PyGOSDT under the \textit{bar-7} data set appears to become approximately constant; this is likely a result of additional samples belonging to already-present equivalence classes (the set of equivalence classes saturates). Recall that both PyGOSDT and GOSDT reduce the data set size to only the equivalence classes that are present in the data set, and thus scale in this quantity rather than the number of samples.

Overall, we observe that both GOSDT, PyGOSDT and OSDT have an advantage over \textit{DL8.5} which becomes increasingly clearer as we test on data sets of greater difficulty. GOSDT and OSDT appears to perform better than PyGOSDT, with GOSDT having a slight advantage over OSDT on larger data sets.

Previous comparisons do not account for differences in implementation language. We observe that that GOSDT is several orders of magnitude faster and more scalable than DL8.5, both of which are implemented in C++. However, PyGOSDT is not quite as performant as OSDT, both of which are implemented in Python. This suggests, for data sets similar to the ones in this experiment, there are advantageous characteristics of OSDT that are worth further exploration for extensions of GOSDT.

\begin{figure}[h]
\centering
\subfloat[Training Time vs Number of Features (Full Scale)]{
    \begin{minipage}[b]{0.99\linewidth}
        \centering
        \scalebox{1.0}{
            \includegraphics[scale=0.22]{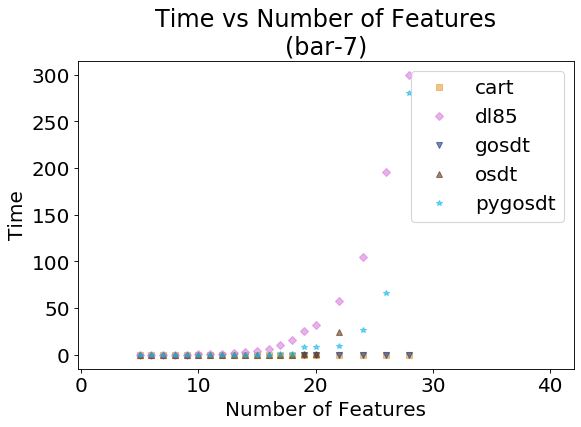}
            \includegraphics[scale=0.22]{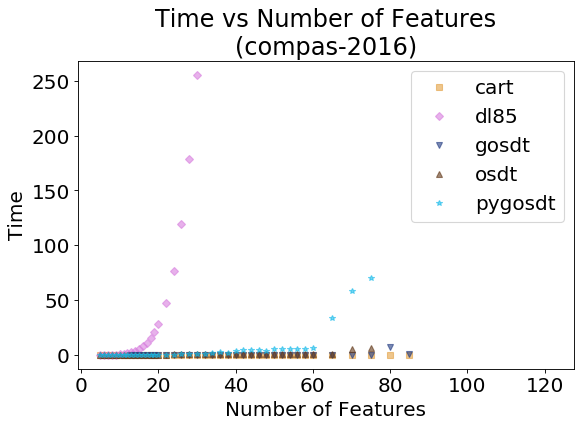}
            \includegraphics[scale=0.22]{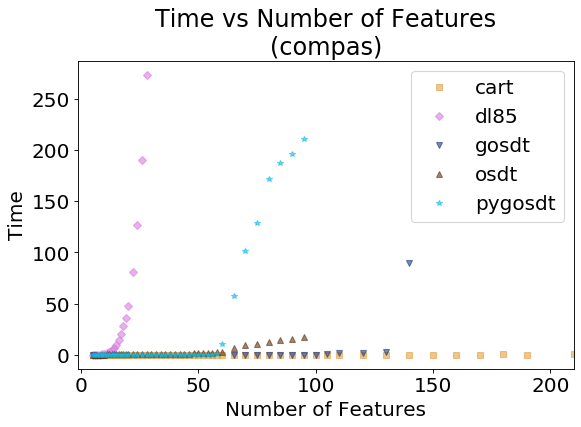}
            \includegraphics[scale=0.22]{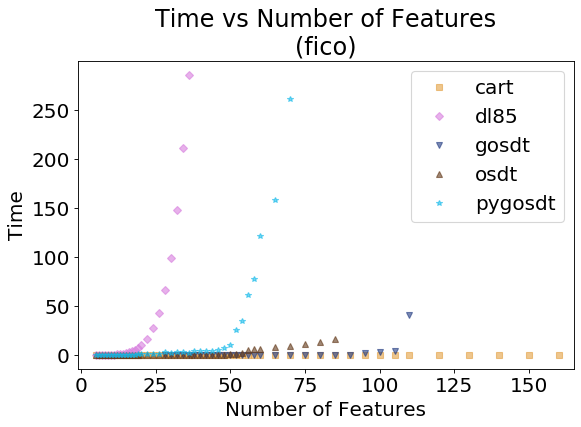}
        }
    \end{minipage}
}
\quad
\subfloat[Training Time vs Number of Features (Zoomed In)]{
    \begin{minipage}[b]{0.99\linewidth}
    \centering
        \scalebox{1.0}{
            \includegraphics[scale=0.22]{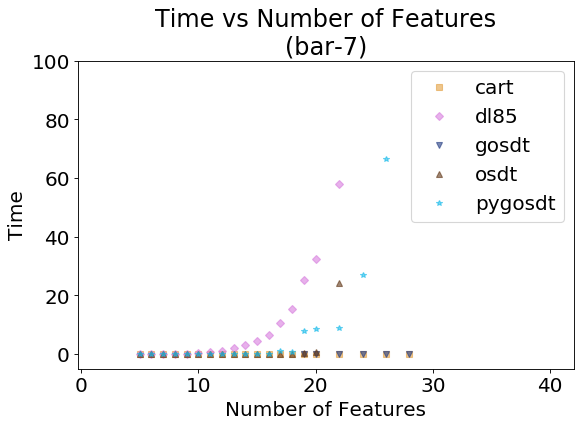}
            \includegraphics[scale=0.22]{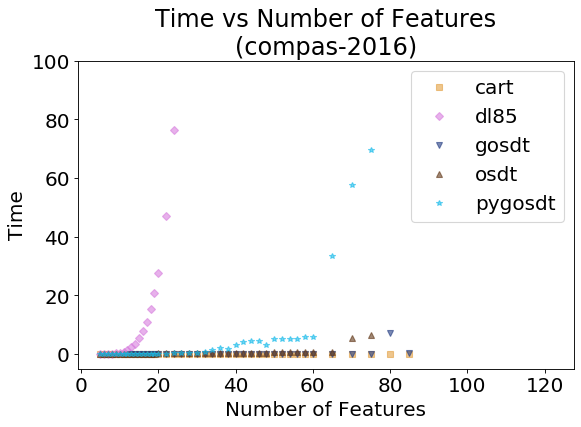}
            \includegraphics[scale=0.22]{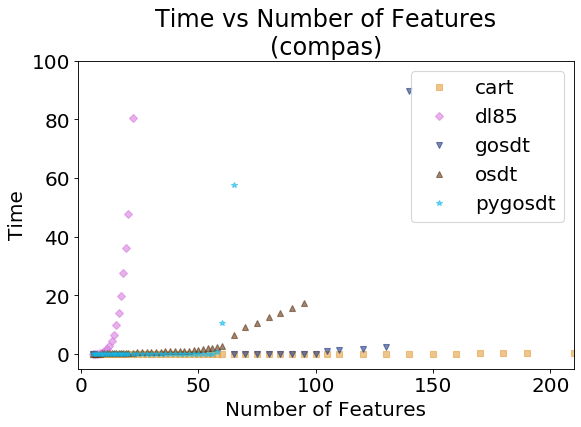}
            \includegraphics[scale=0.22]{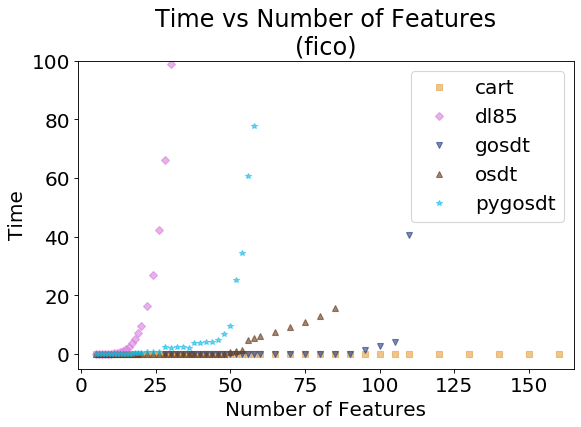}
        }
    \end{minipage}
}
\caption{
    Time required to reach optimality  (or to finish tree construction for non-optimal methods) for BinOCT, CART, DL8.5, GOSDT (C++), PyGOSDT (Python) and OSDT as a function of the number of binary features used to encode the continuous dataset ($\lambda$ = 0.3125 or max depth = 5).
}
\label{fig:time_vs_feature}
\end{figure}

\begin{figure}[h]
\centering
\subfloat[Training Time vs Number of Samples (Full Scale)]{
    \begin{minipage}[b]{0.99\linewidth}
        \centering
        \scalebox{1.0}{
            \includegraphics[scale=0.22]{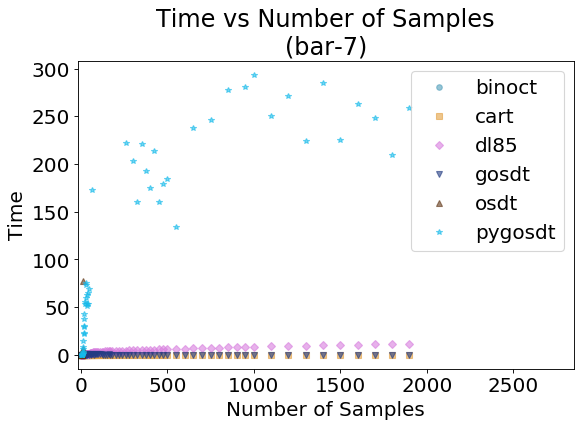}
            \includegraphics[scale=0.22]{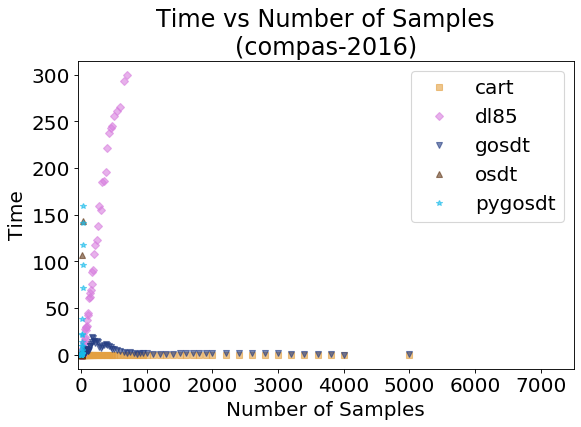}
            \includegraphics[scale=0.22]{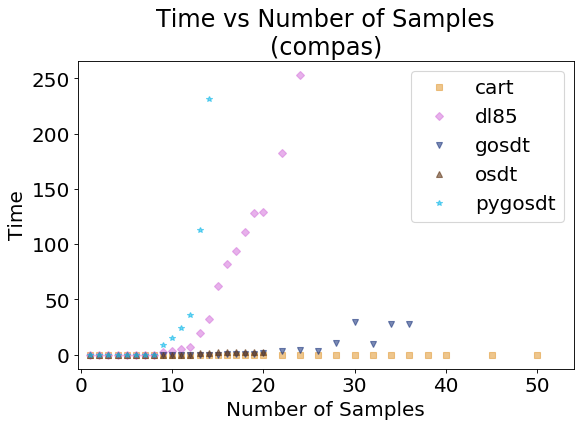}
            \includegraphics[scale=0.22]{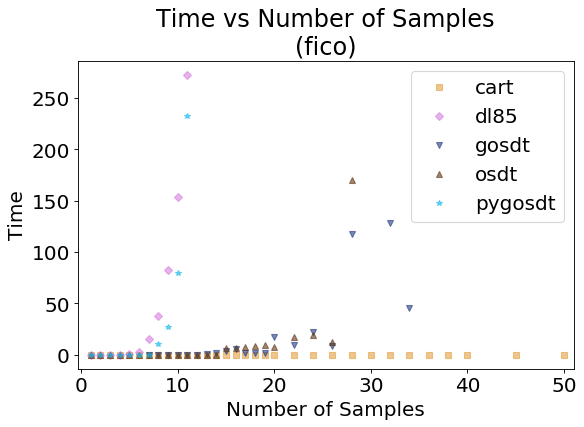}
        }
    \end{minipage}
}
\quad
\subfloat[Training Time vs Number of Samples (Zoomed In)]{
    \begin{minipage}[b]{0.99\linewidth}
    \centering
        \scalebox{1.0}{
            \includegraphics[scale=0.22]{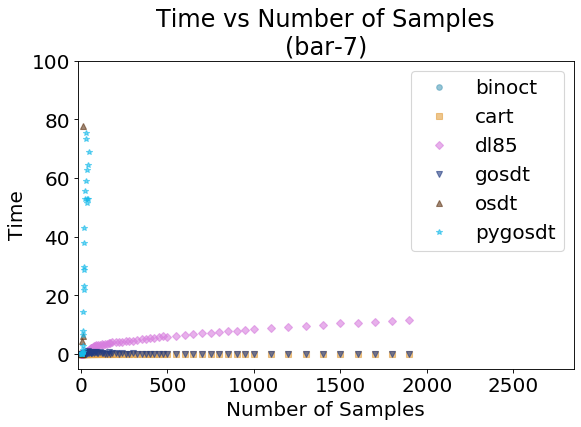}
            \includegraphics[scale=0.22]{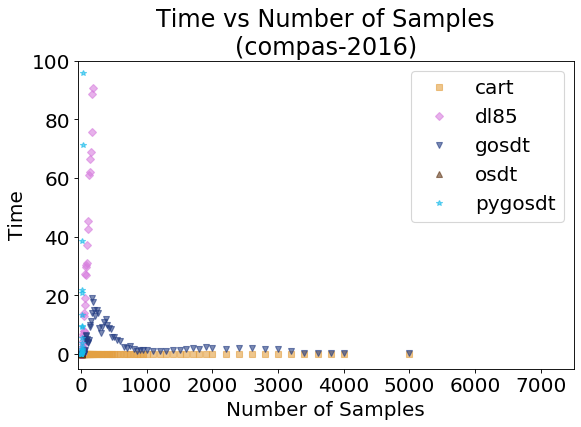}
            \includegraphics[scale=0.22]{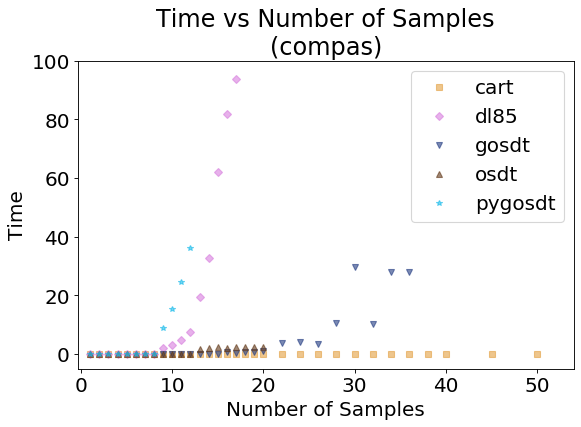}
            \includegraphics[scale=0.22]{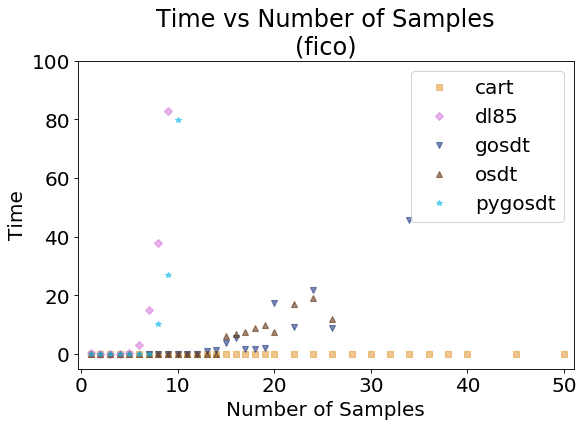}
        }
    \end{minipage}
}
\caption{
    Time required to reach optimality (or to finish tree construction for non-optimal methods) for BinOCT, CART, DL8.5, GOSDT (C++), PyGOSDT (Python) and OSDT as a function of the number of samples taken from the continuous dataset ($\lambda$ = 0.3125 or max depth = 5).
}
\label{fig:time_vs_sample}
\end{figure}

\begin{figure}[h]
\centering
\subfloat[Slow-Down vs Number of Features (Full Scale)]{
    \begin{minipage}[b]{0.99\linewidth}
        \centering
        \scalebox{1.0}{
            \includegraphics[scale=0.21]{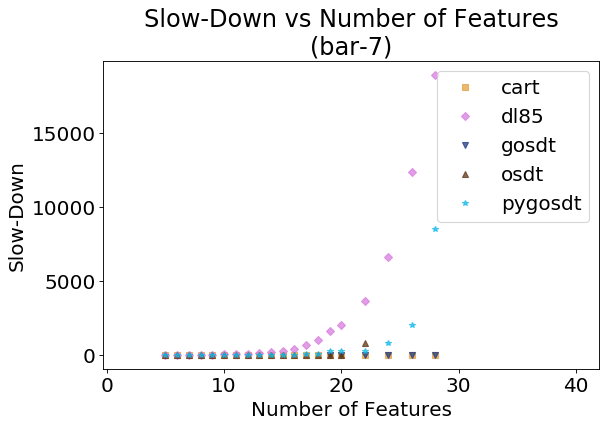}
            \includegraphics[scale=0.21]{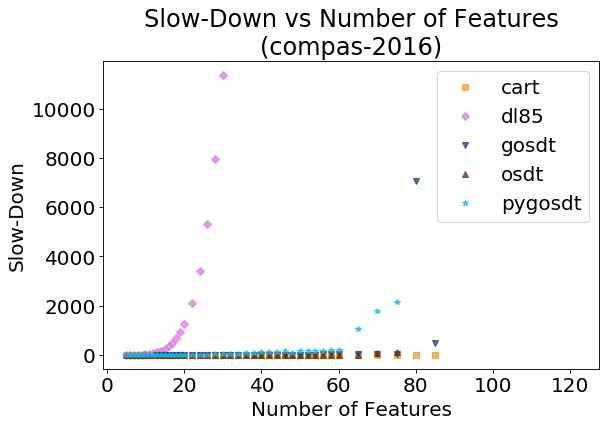}
            \includegraphics[scale=0.21]{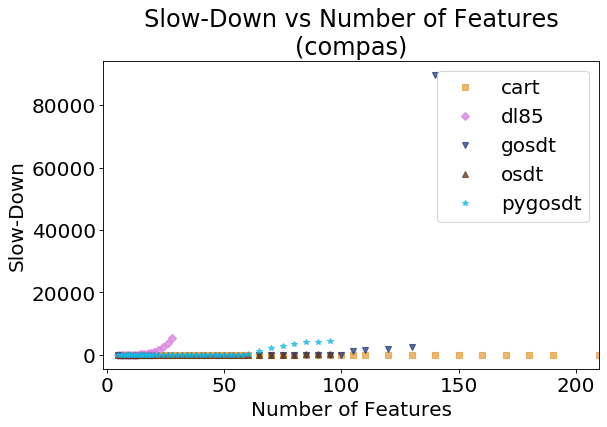}
            \includegraphics[scale=0.21]{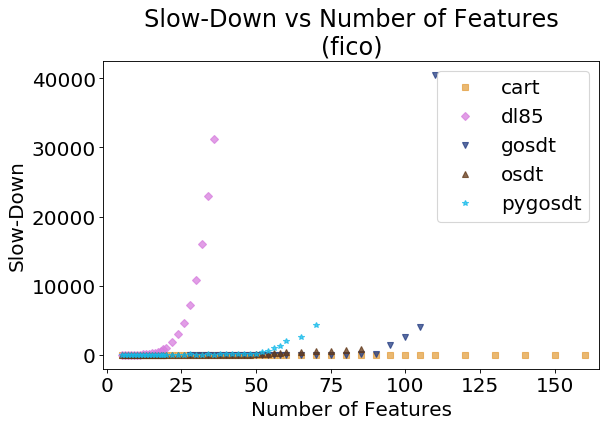}
        }
    \end{minipage}
}
\quad
\subfloat[Slow-Down vs Number of Features (Zoomed In)]{
    \begin{minipage}[b]{0.99\linewidth}
    \centering
        \scalebox{1.0}{
            \includegraphics[scale=0.22]{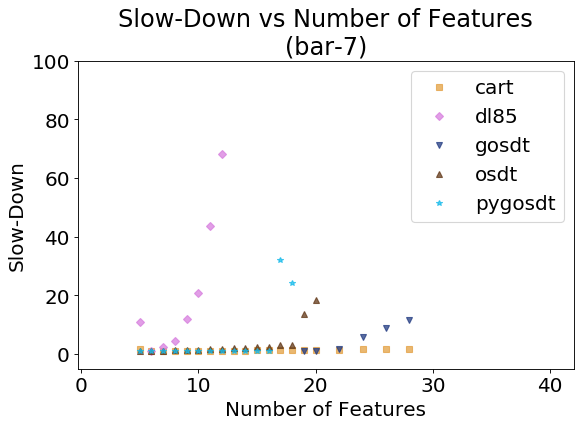}
            \includegraphics[scale=0.22]{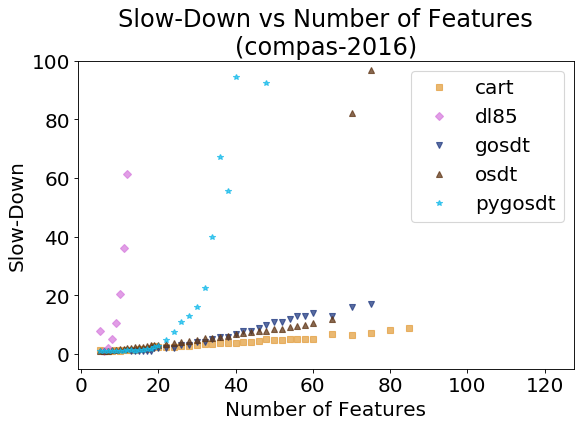}
            \includegraphics[scale=0.22]{figure/experiments/scalability_features/slow-down_number_of_features_compas_small.png}
            \includegraphics[scale=0.22]{figure/experiments/scalability_features/slow-down_number_of_features_fico_small.png}
        }
    \end{minipage}
}
\caption{
    Slow-down experienced by BinOCT, CART, DL8.5, GOSDT (C++), PyGOSDT (Python) and OSDT as a function of the number of binary features used to encode the continuous dataset ($\lambda$ = 0.3125 or max depth = 5).
}
\label{fig:slow_vs_feature}
\end{figure}

\begin{figure}[h]
\centering
\subfloat[Slow-Down vs Number of Samples (Full Scale)]{
    \begin{minipage}[b]{0.99\linewidth}
        \centering
        \scalebox{1.0}{
            \includegraphics[scale=0.21]{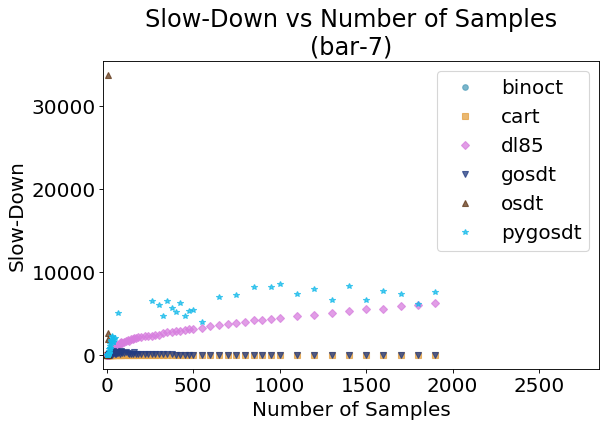}
            \includegraphics[scale=0.21]{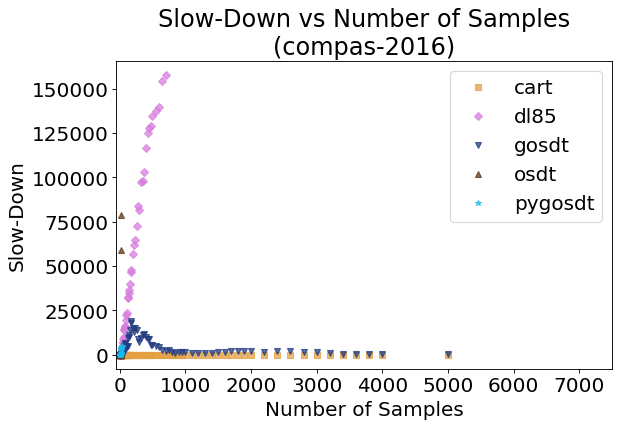}
            \includegraphics[scale=0.21]{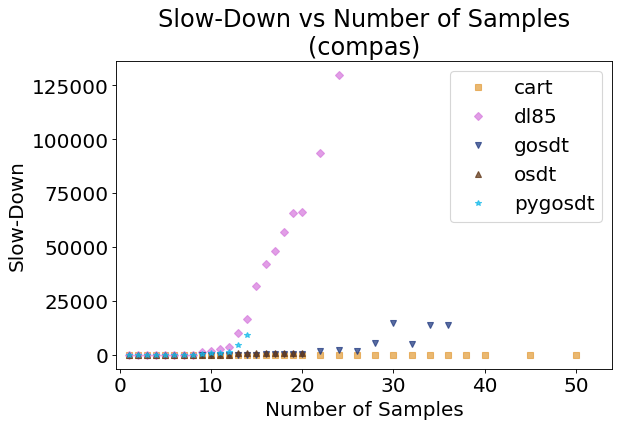}
            \includegraphics[scale=0.21]{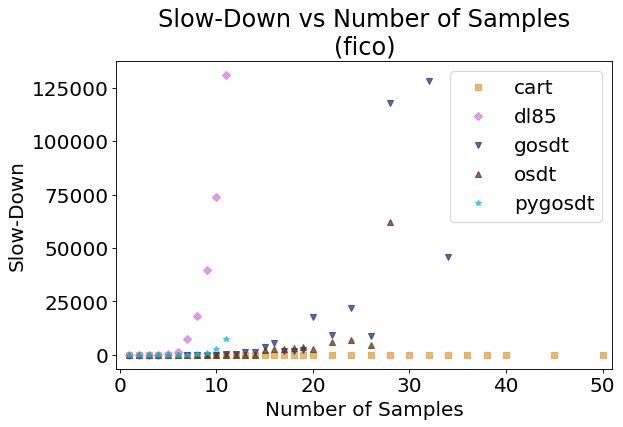}
        }
    \end{minipage}
}
\quad
\subfloat[Slow-Down vs Number of Samples (Zoomed In)]{
    \begin{minipage}[b]{0.99\linewidth}
    \centering
        \scalebox{1.0}{
            \includegraphics[scale=0.22]{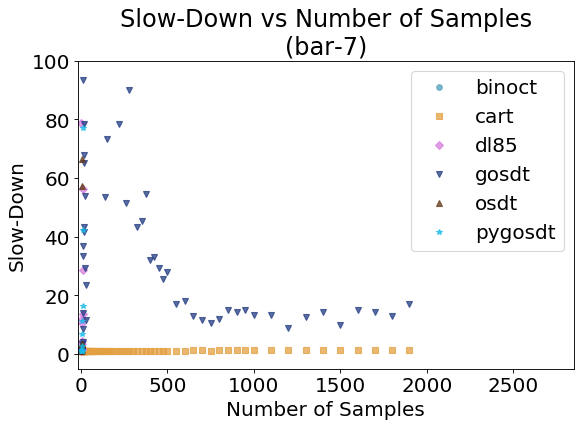}
            \includegraphics[scale=0.22]{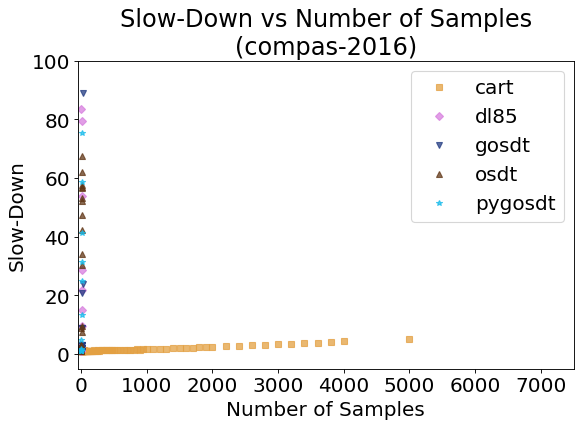}
            \includegraphics[scale=0.22]{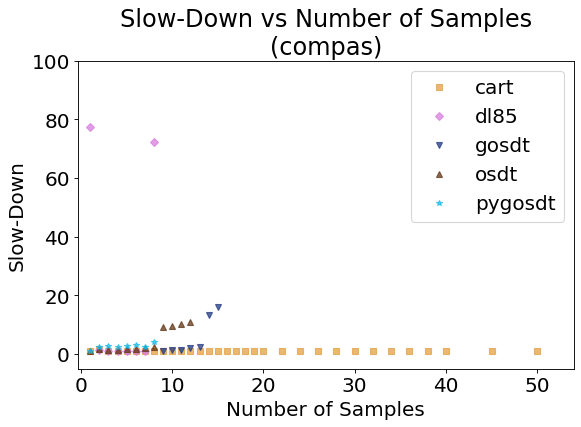}
            \includegraphics[scale=0.22]{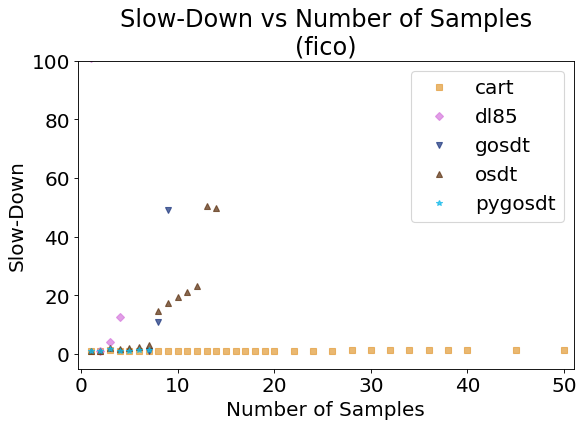}
        }
    \end{minipage}
}
\caption{
    Slow-down experienced by BinOCT, CART, DL8.5, GOSDT (C++), PyGOSDT (Python) and OSDT as a function of the number of samples taken from the continuous dataset ($\lambda$ = 0.3125 or max depth = 5).
}
\label{fig:slow_vs_sample}
\end{figure}

\subsection{Experiment: Time to Optimality}

\textbf{Collection and Setup}:
We ran this experiment on 4 data sets: \textbf{bar-7, tic-tac-toe, car-evaluation, compas-binary, fico-binary, monk-1, monk-2}, and \textbf{monk-3}. For each experiment, we run \textit{OSDT, GOSDT}, and \textit{PyGOSDT} with a regularization coefficient of 0.005. For each run we track the progress of the algorithm by plotting the minimum objective score seen so far. Once the algorithm terminates or reaches a time limit of 30 minutes, the values are stored in a file.

\textbf{Results}:
Figure \ref{fig:time_to_optimality} shows the different behaviors between GOSDT, PyGOSDT, and OSDT. In general, both PyGOSDT and GOSDT complete their certificate of optimality earlier than OSDT.

Note that PyGOSDT's implementation does not include high-priority bound updates. This causes PyGOSDT to maintain a higher objective score before making a sharp drop upon completion (with a certificate of optimality). GOSDT, on the other hand, behaves similarly to OSDT because both algorithms aggressively prioritize lowering the best observed objective score. We observe that under the tic-tac-toe data set this appears to be less advantageous. While PyGOSDT's progress initially appears less promising, it completed remarkably faster than both GOSDT and OSDT. This suggests that optimal prioritization is dependent on specifics of the optimization problem.

\begin{figure}[h]
  \centering
  \includegraphics[scale=0.22]{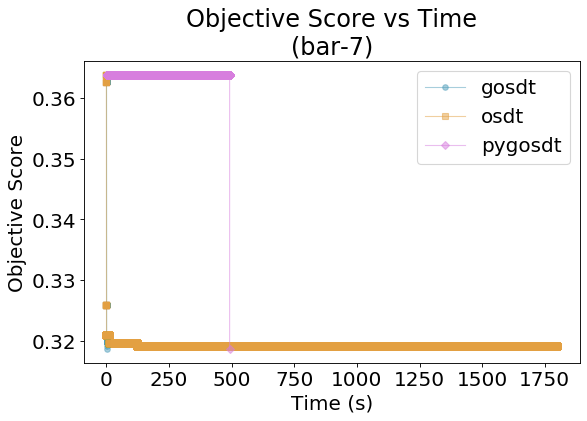}
  \includegraphics[scale=0.22]{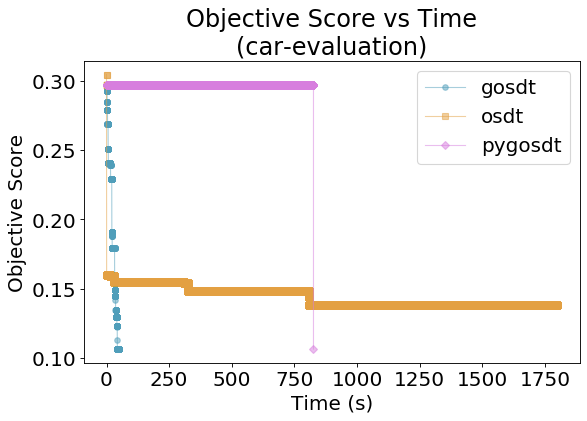}
  \includegraphics[scale=0.22]{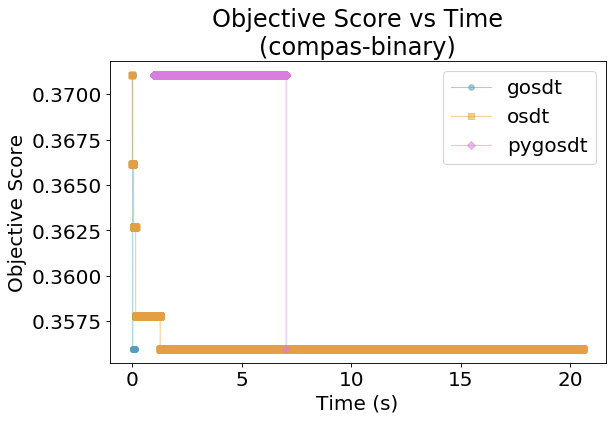}
  \includegraphics[scale=0.22]{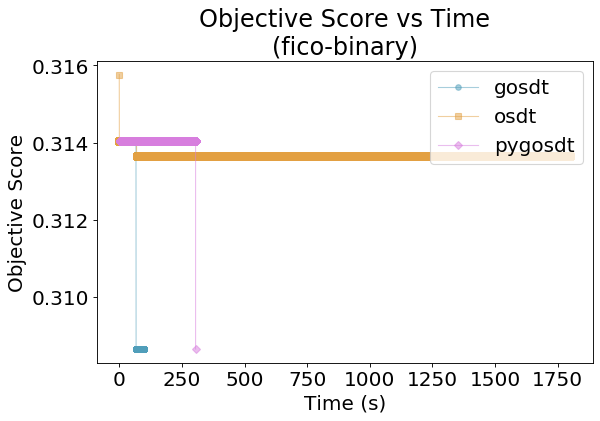}
  \includegraphics[scale=0.22]{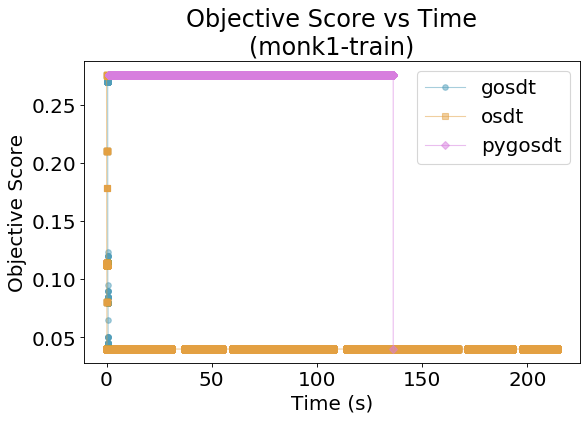}
  \includegraphics[scale=0.22]{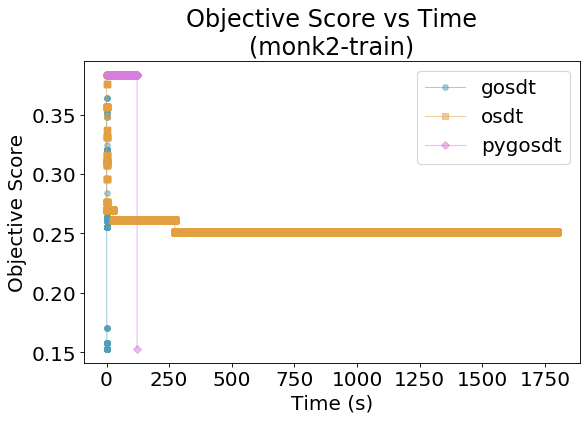}
  \includegraphics[scale=0.22]{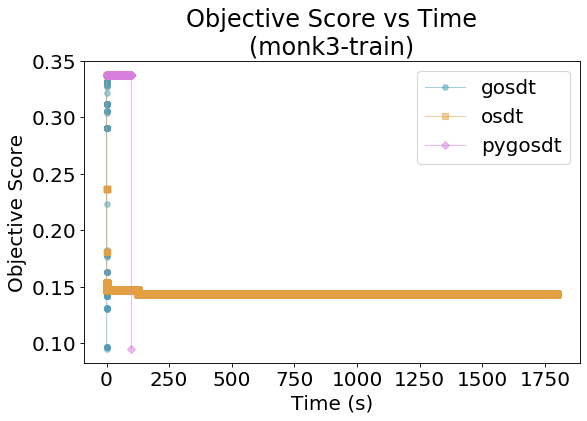}
  \includegraphics[scale=0.22]{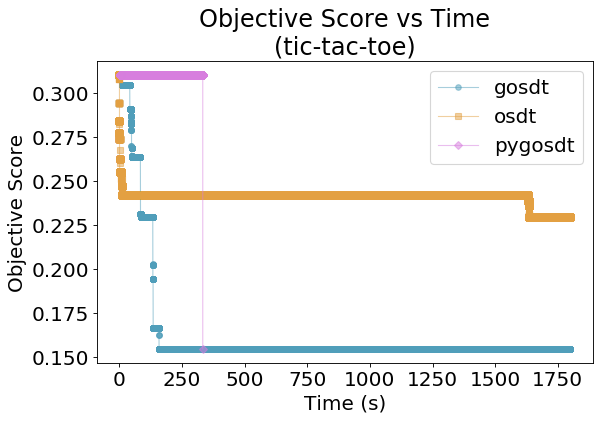}
  \caption{Best objective score of OSDT, GOSDT, and PyGOSDT over the course of their run time. ($\lambda$ = 0.005)
  \label{fig:time_to_optimality}}
\end{figure}

\subsection{Optimal Trees}
We present some of the trees that achieved the peak median accuracy from Section \ref{exp:accspar}. Figure \ref{fig:monk1tree} shows a comparison between the results of training BinOCT (a) and GOSDT (b) on the Monk 1 data set. GOSDT is able to produce a model with 20\% higher accuracy than BinOCT even though both trees have 8 leaves.
Figure \ref{fig:monk2tree} shows a comparison between DL8.5 (a) and GOSDT (b) on the Monk 2 data set. GOSDT is able to produce a model with 3\% higher accuracy than DL8.5 even though both trees have 7 leaves. Figure \ref{fig:tictactoetree} shows a comparison between BinOCT (a), DL8.5 (b), and GOSDT (c) on the Tic-Tac-Toe data set. GOSDT is able to produce a model with higher accuracy than both BinOCT and DL8.5 when all trees have 16 leaves.

\textbf{Comparison to True Model}:
For the results shown in Figure \ref{fig:monk1tree}, we know that the true model used to generate the data in Monk 1 is a set of logical rules:
\[
\begin{aligned}
class = (jacket = red) \lor (head = body).
\end{aligned}
\]
The data set we train on does not encode binary features for equality between two features (e.g., $head = body$) and categorical variables $head$ and $body$ are only encoded using $k-1$ binary rules (this means one value from each categorical variable will be expressed with a negation of all other values). Altogether, this means our encoding forces the true model to instead be expressed as the following:
\[
\begin{aligned}
class = & (jacket = red) \\
& \lor (head = round \land body = round) \\
& \lor (head = square \land body = square) \\
& \lor (head \neq round \land head \neq square \land body \neq round \land body \neq square) \\
\end{aligned}
\]

We can interpret the trees produced by BinOCT as the following set of logical rules:
\[
\begin{aligned}
class = & (jacket = red \land head = round ) \\
& \lor (jacket \neq red \land head = round \land body = round) \\
& \lor (jacket = red \land head \neq round \land body = round ) \\ 
& \lor (jacket \neq green \land head \neq round \land body \neq round).
& \end{aligned}
\]

We can interpret the trees produced by GOSDT as the following set of logical rules:
\[
\begin{aligned}
class = & (jacket = red) \\
& \lor (head = round \land body = round) \\
& \lor (head = square \land body = square) \\
& \lor (head \neq round \land head \neq square \land body \neq round \land body \neq square). \\
\end{aligned}
\]
In this instance, BinOCT produces a model that is similar to the true model but has a few mismatches. This is mainly due to the structural constraints of BinOCT.
GOSDT, after exploring a larger space while still penalizing complexity, is able to produce a model that perfectly matches with the ground truth.

\label{app:Trees}

\begin{figure}[h]
\centering
\subfloat[BinOCT (training accuracy: 90.9\%, test accuracy: 84.0\%)]{
\begin{minipage}[b]{0.49\textwidth}
  \centering
  \scalebox{0.70}{
    \begin{forest}
    [ {$head = round$}
        [ {$jacket = red$},edge label={node[midway, above, font=\small] {True}}
            [ {$head = round$} [ $Yes$ ] [ $Yes$ ] ] 
            [ {$body = round$} [ $Yes$ ] [ $No$ ]  ] 
        ]
        [ {$body = round$},edge label={node[midway, above, font=\small] {False}}
            [ {$jacket = red$} [ $Yes$ ] [ $No$ ] ] 
            [ {$jacket = green$} [ $No$ ] [ $Yes$ ]  ]
        ]
    ]
    \end{forest}}
\end{minipage}}
\subfloat[GOSDT (training accuracy: 100\%, test accuracy: 100\%)]{
\begin{minipage}[b]{0.49\textwidth}
  \centering
    \scalebox{0.75}{
    \begin{forest}
[ {$jacket = red$} 
    [ $Yes$, edge label={node[midway,above, font=\small] {True}}  ] 
    [ {$head = round$},edge label={node[midway,above, fill=white, font=\small] {False}}  
        [ {$body = round$} [ $Yes$ ] [ $No$ ] ]
        [ {$head = square$} 
            [ {$body = square$} [ $Yes$ ] [ $No$ ] ]
            [ {$body = round$} 
                [ $No$ ]
                [ {$body = square$} [ $No$ ] [ $Yes$ ] ] 
            ]
        ]
    ] 
]
    \end{forest}}
\end{minipage}}
\caption{Eight-leaf decision trees generated by BinOCT and GOSDT on the Monk1 dataset. The tree generated by BinOCT includes two useless splits (the head=round in the bottom left), while GOSDT avoids this problem. BinOCT is 91\% accurate, GOSDT is 100\% accurate.}
\label{fig:monk1tree}
\end{figure}
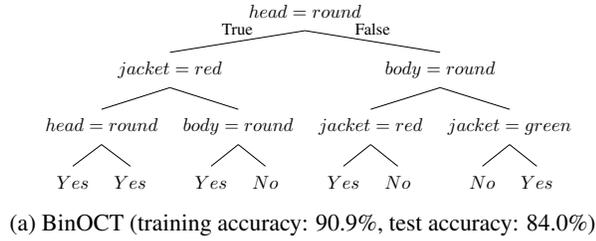
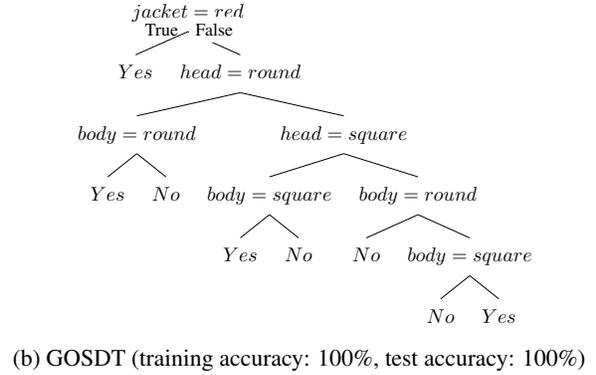

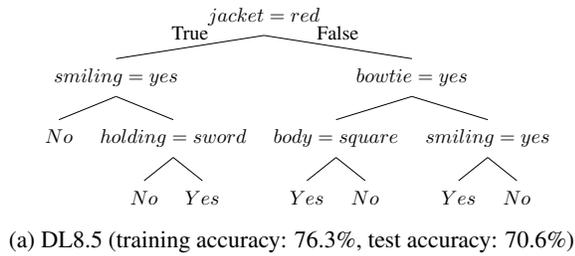
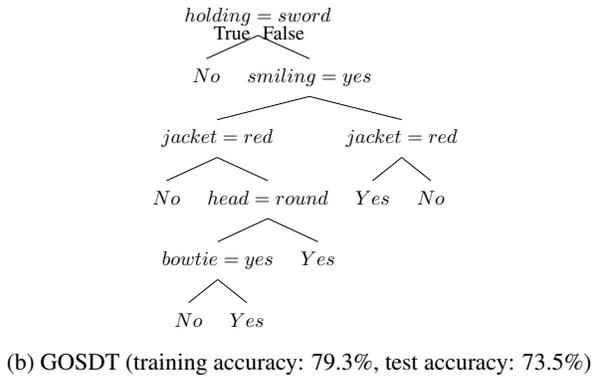
\begin{figure}[h]
\centering
\subfloat[DL8.5 (training accuracy: 76.3\%, test accuracy: 70.6\%)]{
\begin{minipage}[b]{0.49\textwidth}
  \centering
  \scalebox{0.75}{
    \begin{forest}
    [ {$jacket = red$}
        [ {$smiling = yes$},edge label={node[midway, above] {True}}
            [ $No$ ]
            [ {$holding = sword$} [ $No$ ] [ $Yes$ ] ] 
        ]
        [ {$bowtie = yes$},edge label={node[midway, above] {False}}
            [ {$body = square$} [ $Yes$ ] [ $No$ ] ]
            [ {$smiling = yes$} [ $Yes$ ] [ $No$ ] ]
        ]
    ]
    \end{forest}}
\end{minipage}}
\subfloat[GOSDT (training accuracy: 79.3\%, test accuracy: 73.5\%)]{
\begin{minipage}[b]{0.49\textwidth}
  \centering
    \scalebox{0.75}{
    \begin{forest}
    [ {$holding = sword$}
        [ $No$,edge label={node[midway, above] {True}} ]
        [ {$smiling = yes$},edge label={node[midway, above] {False}}
            [ {$jacket = red$}
                [ $No$ ]
                [ {$head = round$}
                    [ {$bowtie = yes$} [ $No$ ] [ $Yes$ ] ]
                    [ $Yes$ ]
                ]
            ]
            [ {$jacket = red$} [ $Yes$ ] [ $No$ ] ]
        ]
    ]
    \end{forest}}
\end{minipage}}
\caption{Seven-leaf decision trees generated by DL8.5 and GOSDT on the Monk2 dataset. With the same number of leaves, DL8.5 is 76.3\% accurate, GOSDT is 79.3\% accurate. 
}
\label{fig:monk2tree}
\end{figure}

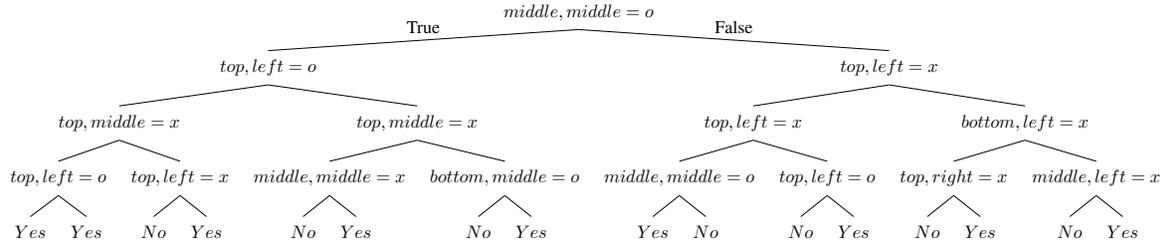
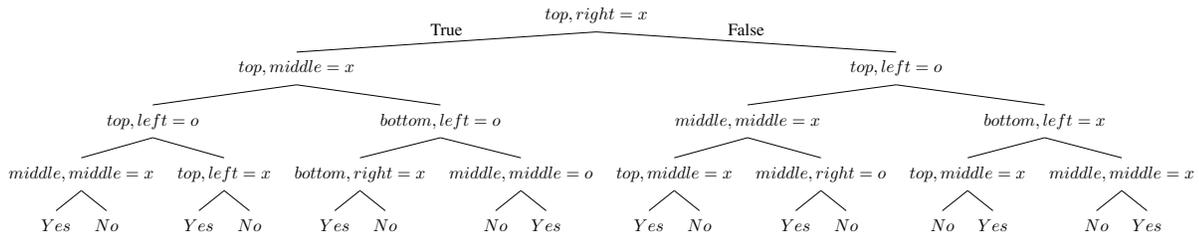
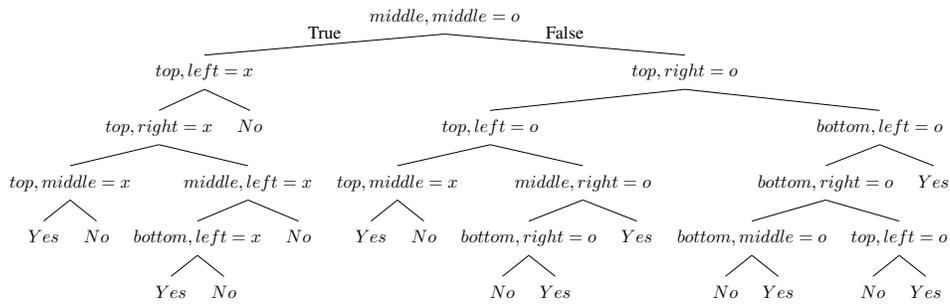
\begin{figure}[h]
\centering
\subfloat[BinOCT (training accuracy: 82.4\%, test accuracy 34.4\%)]{
\begin{minipage}[b]{0.99\linewidth}
\centering
\scalebox{0.68}{
\begin{forest}
[ {$middle,middle = o$}
    [ {$top,left = o$},edge label={node[midway, above] {True}}
        [ {$top,middle = x$}
            [ {$top,left = o$} [ $Yes$ ] [ $Yes$ ] ]
            [ {$top,left = x$} [ $No$ ] [ $Yes$ ] ]
        ]
        [ {$top,middle = x$}
            [ {$middle,middle = x$} [ $No$ ] [ $Yes$ ] ]
            [ {$bottom,middle = o$} [ $No$ ] [ $Yes$ ] ]
        ]
    ]
    [ {$top,left = x$},edge label={node[midway, above] {False}}
        [ {$top,left = x$}
            [ {$middle,middle = o$} [ $Yes$ ] [ $No$ ] ]
            [ {$top,left = o$} [ $No$ ] [ $Yes$ ] ]
        ]
        [ {$bottom,left = x$}
            [ {$top,right = x$} [ $No$ ] [ $Yes$ ] ]
            [ {$middle,left = x$} [ $No$ ] [ $Yes$ ] ]
        ]
    ]
]
\end{forest}}
\end{minipage}
}
\quad
\subfloat[DL8.5 (training accuracy: 86.9\%, test accuracy: 61.5\%)]{
\begin{minipage}[b]{0.99\linewidth}
\scalebox{0.65}{
\begin{forest}
[ {$top,right = x$}
    [ {$top,middle = x$},edge label={node[midway, above] {True}}
        [ {$top,left = o$}
            [ {$middle,middle = x$} [ $Yes$ ] [ $No$ ] ]
            [ {$top,left = x$} [ $Yes$ ] [ $No$ ] ]
        ]
        [ {$bottom,left = o$}
            [ {$bottom,right = x$} [ $Yes$ ] [ $No$ ] ]
            [ {$middle,middle = o$} [ $No$] [ $Yes$ ] ]
        ]
    ]
    [ {$top,left = o$},edge label={node[midway, above] {False}}
        [ {$middle,middle = x$}
            [ {$top,middle = x$} [ $Yes$ ] [ $No$ ] ]
            [ {$middle,right = o$} [ $Yes$ ] [ $No$ ] ]
        ]
        [ {$bottom,left = x$}
            [ {$top,middle = x$} [ $No$] [ $Yes$ ] ]
            [ {$middle,middle = x$}  [ $No$] [ $Yes$ ] ]
        ]
    ]
]
\end{forest}}
\end{minipage}
}
\quad
\subfloat[GOSDT (training accuracy: 90.9\%, test accuracy: 70.8\%)]{
\begin{minipage}[b]{0.99\linewidth}
\centering
\scalebox{0.68}{
\begin{forest}
[ {$middle,middle = o$}
    [ {$top,left = x$}, edge label={node[midway, above] {True}}
        [ {$top,right = x$}
            [ {$top,middle = x$} [ $Yes$ ] [ $No$ ] ]
            [ {$middle,left = x$}
                [ {$bottom,left = x$} [ $Yes$ ] [ $No$ ] ]
                [ $No$ ]
            ]
        ]
        [ $No$ ]
    ]
    [ {$top,right = o$}, edge label={node[midway, above] {False}}
        [ {$top,left = o$}
            [ {$top,middle = x$} [ $Yes$ ] [ $No$ ] ]
            [ {$middle,right = o$}
                [ {$bottom,right = o$} [ $No$ ] [ $Yes$ ] ]
                [ $Yes$ ]
            ]
        ]
        [ {$bottom,left = o$}
            [ {$bottom,right = o$}
                [ {$bottom,middle = o$} [ $No$ ] [ $Yes$ ] ]
                [ {$top,left = o$} [ $No$ ] [ $Yes$ ] ]
            ]
            [ $Yes$ ]
        ]
    ]
]
\end{forest}}
\end{minipage}
}
\caption{16-leaf decision trees generated by BinOCT, DL8.5, and GOSDT on the tic-tac-toe dataset. The tree generated by BinOCT includes some useless splits such as top,left=o on the bottom left and middle,middle=o near the center of the bottom row. These extra splits repeat earlier decisions from the tree, so they are clearly useless and lead to at least one empty leaf. DL8.5 also prefers to generate complete binary trees. GOSDT is more effective in generating sparse trees. With the same number of leaves, BinOCT is 82.4\% accurate, DL8.5 is 86.9\% accurate, and GOSDT is 90.9\% accurate. }
\label{fig:tictactoetree}
\end{figure}

\subsection{Summary of Experimental Results}
\textbf{Experiment G.5} shows that the new set of objective functions allows GOSDT to produce trees with a more efficient ROC curve than the standard accuracy objective assumed by other algorithms.

\textbf{Experiment G.6} shows that the regularized risk objective used by OSDT and GOSDT produces the most efficient training accuracy vs$.$ sparsity frontier. When placed under time constraints, GOSDT is able to produce more of the highly accurate models along this frontier than OSDT.

\textbf{Experiment G.7} shows that GOSDT is able to handle significantly more binary features than BinOCT, DL8.5, and, to a lesser extent, OSDT. Since binary features are used to encode thresholds over continuous features, GOSDT is able to handle continuous datasets of higher cardinality compared to other aforementioned methods.

\textbf{Experiment G.8} shows that GOSDT outpaces OSDT and PyGOSDT when it comes to reducing the optimality gap, this allows it to terminate with stronger optimality guarantees in the event of a premature termination.

\textbf{Experiment G.9} shows that optimizing an efficient training accuracy vs$.$ sparsity frontier allows GOSDT to more accurately capture the ground truth compared to BinOCT when subject to the same sparsity constraints.

To summarize, we began this experimental section by showing the benefits of optimizing more sophisticated objective functions. We then showed the benefits of a more efficient algorithm to support these objectives. Finally, we closed this section by combining these two elements to produce provably optimal and interpretable models and showcase their advantages.

\section{Algorithm}\label{App:algorithm}
In addition to the main GOSDT algorithm (Algorithm \ref{alg:gosdt_summary_app}), we present the subroutines \textit{get\_lower\_bound} (Algorithm \ref{alg:lowerbound}), \textit{get\_upper\_bound} (Algorithm \ref{alg:upperbound}), \textit{fails\_bound} (Algorithm \ref{alg:failsbound}), and \textit{split} (Algorithm \ref{alg:split}) used during optimization. We also present an extraction algorithm (Algorithm \ref{alg:extraction}) used to construct the optimal tree from the dependency graph once the main GOSDT algorithm completes.



\setcounter{algorithm}{0}
\begin{algorithm}[htb!]
\caption{GOSDT$(R, \x, \y, \lambda)$ \label{alg:gosdt_summary_app}}
\begin{tabbing}
xxx \= xx \= xx \= xx \= xx \= xx \kill
1: \> \textbf{input:} $R$, $Z$, $z^-$, $z^+$, $\lambda$ \comment{risk, unique sample set, negative sample set, positive sample set, regularizer} \\
2: \> $Q = \emptyset$ \comment{priority queue}\\
3: \> $G=\emptyset$ \comment{dependency graph}\\
4: \> $s_0 \leftarrow \{1,...,1\}$\comment{bit-vector of 1's of length $U$} \\
5: \> $p_0 \leftarrow$ FIND\_OR\_CREATE\_NODE($G, s_0$)\comment{node for root}\\
6: \> $Q.{\rm push}(s_0)$\comment{add to priority queue}\\
7: \> \textbf{while} $p_0.lb \neq p_0.ub$ \textbf{do}\\
8: \> \> $s\leftarrow Q.{\rm pop}()$\comment{index of problem to work on}\\
9: \> \> $p\leftarrow G.{\rm find}(s)$\comment{find problem to work on}\\
10: \> \> \textbf{if} $p.lb=p.ub$ \textbf{then}\\
11: \> \> \> \textbf{continue}\comment{problem already solved} \\
12: \> \> $(lb', ub') \leftarrow (\infty, \infty)$\comment{very loose starting bounds}\\
13: \> \> \textbf{for} each feature $j \in [1,M]$ \textbf{do}\\
14: \> \> \> $s_l, s_r \leftarrow \text{split}(s,j,Z)$ \comment{create children if they don't exist} \\
15: \> \> \> $p_l^j\leftarrow$FIND\_OR\_CREATE\_NODE$(G,s_l)$\\
16: \> \> \> $p_r^j\leftarrow$FIND\_OR\_CREATE\_NODE$(G,s_r)$\\
\>\>\comment{create bounds as if $j$ were chosen for splitting}
\\
17: \> \> \> $lb' \leftarrow \min(lb', p_l^j.lb + p_r^j.lb)$ \\
18: \> \> \> $ub' \leftarrow \min(ub', p_l^j.ub + p_r^j.ub)$ \\
\> \> \comment{signal the parents if an update occurred} \\
19: \> \> \textbf{if} $p.lb \neq lb'$ \textbf{or} $p.ub \neq ub'$  \textbf{then} \\
20: \> \> \> $p.ub \leftarrow \min(p.ub, ub')$\\
21: \> \> \> $p.lb \leftarrow \min(p.ub, \max(p.lb, lb'))$\\
22: \> \> \> \textbf{for} $p_{\pi} \in G.{\rm parent}(p)$ \textbf{do} \\
\> \> \> \> \comment{propagate information upwards}\\
23: \> \> \> \> $Q.{\rm push}(p_{\pi}.{\rm id}, {\rm priority}=1)$\\
24: \> \> \textbf{if} $p.lb = p.ub$ \textbf{then} \\
25: \> \> \> \textbf{continue} \comment{problem solved just now} \\
\> \> \comment{loop, enqueue all children that are dependencies} \\

26: \> \> \textbf{for} each feature $j \in [1,M]$ \textbf{do} \\
\> \> \comment{fetch $p_l^j$ and $p_r^j$ in case of update from other thread} \\
27: \> \> \> repeat line 14-16\\
28: \> \> \> $lb' \leftarrow p_l^j.lb + p_r^j.lb$ \\
29: \> \> \> $ub' \leftarrow p_l^j.ub + p_r^j.ub$ \\
30: \> \> \> \textbf{if} $lb' < ub'$ \textbf{and} $lb' \le p.ub$ \textbf{then} \\
31: \> \> \> \> $Q.{\rm push}(s_{l}, {\rm priority}=0)$ \\
32: \> \> \> \> $Q.{\rm push}(s_{r}, {\rm priority}=0)$ \\
33: \> \textbf{return}\\
---------------------------------------------------------------------------\\
34: \> \textbf{subroutine} FIND\_OR\_CREATE\_NODE(G,s)\\
35: \> \> if $G.{\rm find}(s) = {\rm NULL}$ \comment{$p$ not yet in dependency graph}\\
36: \> \> \> $p.id \leftarrow s$ \comment{identify $p$ by $s$}\\
37: \> \> \> $p.lb \leftarrow {\rm get\_lower\_bound}(s,Z,z^-,z^+)$\\
38: \> \> \> $p.ub \leftarrow {\rm get\_upper\_bound}(s,Z,z^-,z^+)$\\
39: \> \> \> \textbf{if} fails\_bounds$(p)$ \textbf{then}\\
40: \> \> \> \> $p.lb=p.ub$ \comment{no more splitting allowed}\\
41: \> \> \> G.insert(p) \comment{put $p$ in dependency graph}\\
42: \> \> \textbf{return} G.find(s)
\end{tabbing}
\end{algorithm}
\begin{algorithm}
\caption{get\_lower\_bound$(s, Z, z^-, z^+)$ $\rightarrow$ $lb$\label{alg:lowerbound})}
\begin{minipage}{1.0\linewidth}
\begin{tabbing}
xxx \= xxx \= xxx \kill
\textbf{input:} $s, Z, z^-, z^+$ \comment{support, unique sample set, negative sample set, positive sample set} \\
\textbf{output:} $lb$ \comment{Risk lower bound} \\
\comment{Compute the risk contributed if applying a class to every equivalence class independently} \\
\textbf{for} each equivalence class $u \in [1,U]$ \textbf{define}\\
\> \comment{Values provided in $Z$} \\
\> $z_u^- = \frac{1}{N}\sum_{i=1}^{N}\mathds{1}[y_i = 0 \land x_i = z_u]$ \\
\> $z_u^+ = \frac{1}{N}\sum_{i=1}^{N}\mathds{1}[y_i = 1 \land x_i = z_u]$ \\
\> \comment{Risk of assigning a class to equivalence class $u$} \\
\> $z_u^{\min} = \min(z_u^-, z_u^+)$ \\
\comment{Add all risks for each class $u$} \\
\comment{Add $2\lambda$ which is a lower bound of the complexity penalty} \\
$lb \leftarrow 2\lambda + \sum_u s_u z_u^{\min}$ \\
\textbf{return} $lb$
\end{tabbing}
\end{minipage}
\end{algorithm}
\begin{algorithm}
\caption{get\_upper\_bound$(s, Z, z^-, z^+)$ $\rightarrow$ $ub$\label{alg:upperbound}}
\begin{minipage}{1.0\linewidth}
\begin{tabbing}
xxx \= xxx \= xxx \kill
\textbf{input:} $s, Z, z^-, z^+$ \comment{support, unique sample set, negative sample set, positive sample set} \\
\textbf{output:} $ub$ \comment{Risk upper bound} \\
\comment{Compute the risk contributed if applying a single class to all samples in $s$} \\
\textbf{for} each equivalence class $u \in [1,U]$ \textbf{define}\\
\> \comment{Add up the positive and negative class weights under equivalence class $u$} \\
\> $z_u^- = \frac{1}{N}\sum_{i=1}^{N}\mathds{1}[y_i = 0 \land x_i = z_u]$ \\
\> $z_u^+ = \frac{1}{N}\sum_{i=1}^{N}\mathds{1}[y_i = 1 \land x_i = z_u]$ \\
\comment{Total the positive and negatives over all classes $u$, choosing the smaller total as the misclassification} \\
\comment{Add a single $\lambda$ for the complexity penalty of a leaf} \\
$ub \leftarrow \lambda +\text{min}(\sum_u s_u z_u^-, \sum_u s_u z_u^+)$ \\
\textbf{return} $ub$
\end{tabbing}
\end{minipage}
\end{algorithm}
\begin{algorithm}
\caption{fails\_bounds$(p)$ $\rightarrow$ $v$\label{alg:failsbound}}
\begin{minipage}{1.0\linewidth}
\begin{tabbing}
xxx \= xxx \= xxx \kill
\textbf{input:} $p$ \comment{current problem} \\
\textbf{output:} $v$ \comment{boolean indicating valid problem} \\
\comment{ If this expression is true then the lower bound on incremental accuracy is crossed by all descendents } \\
\comment{ This works because since $ub - lb$ is an upperbound on incremental accuracy for any descendent } \\
$incremental\_accuracy \leftarrow p.ub - p.lb \le 0$ \\
\comment{ If this expression is true then the lower bound on leaf classification accuracy is crossed } \\
$s \leftarrow p.id$ \\
$leaf\_accuracy \leftarrow  \frac{1}{N} \sum_{i=1}^N \mathds{1}[s_i = 1] \le 2\lambda$ \\
\textbf{if} $(incremental\_accuracy = True) \lor (leaf\_accuracy = True)$ \textbf{then} \\
\> \textbf{return} $True$ \\
\textbf{return} $False$\\
\end{tabbing}
\end{minipage}
\end{algorithm}
\begin{algorithm}
\caption{split$(s, j, Z)$ $\rightarrow$ $s_l, s_r$\label{alg:split}}
\begin{minipage}{1.0\linewidth}
\begin{tabbing}
xxx \= xxx \= xxx \kill
\textbf{input:} $s, j, Z$ \comment{support set, feature index, unique sample set} \\
\textbf{output:} $s_l, s_r$ \comment{left and right splits} \\
\comment{Create the left key which is the subset of $s$ such that feature $j$ tests negative} \\
$s_l = \{\mathds{1}[s_u = 1 \land Z_{u,j} = 0] | 1 \le u \le U\}$ \\
\comment{Create the right key which is the subset of $s$ such that feature $j$ tests positive} \\
$s_r = \{\mathds{1}[s_u = 1 \land Z_{u,j} = 1] | 1 \le u \le U\}$ \\
\textbf{return} $s_l, s_r$
\end{tabbing}
\end{minipage}
\end{algorithm}
\begin{algorithm}
\caption{extract$(t)$ $\rightarrow$ $s$ \comment{Extract optimal tree after running the algorithm}\label{alg:extraction}}
\begin{minipage}{1.0\linewidth}
\begin{tabbing}
xxx \= xxx \= xxx \= xxx \kill
\textbf{input:} $s$ \comment{Key of the problem from which we want to build a tree} \\
\textbf{output:} $t$ \comment{Optimal tree} \\
$p \leftarrow$ FIND\_OR\_CREATE\_NODE$(G,s)$ \comment{Find the node associated to this key} \\
$t \leftarrow$ None \comment{ Create a null tree } \\
$base\_bound \leftarrow p.ub$ \comment{ The risk if we end this node as a leaf } \\ 
$base\_prediction \leftarrow 0$ \comment{ The prediction if we end this node as a leaf } \\
$split\_bound \leftarrow \infty$ \comment{ The risk if we split this node } \\
$split\_feature \leftarrow 0$ \comment{ The index of the feature we should use to split this node } \\
\textbf{for} each feature $j \in [1,M]$ \textbf{do} \comment{ Check all possible features } \\
\> $s_l, s_r \leftarrow \text{split}(s,j,Z)$ \comment{Key of the the children for this split} \\
\> $p_l^j\leftarrow$FIND\_OR\_CREATE\_NODE$(G,s_l)$ \comment{Find left child} \\
\> $p_r^j\leftarrow$FIND\_OR\_CREATE\_NODE$(G,s_r)$ \comment{Find right child} \\
\> \comment{Check if the risk of this split is better than the best split risk so far } \\
\> \textbf{if} $p_l^j.ub + p_r^j.ub < split\_bound$ \textbf{then} \\
\> \> $split\_bound \leftarrow p_l^j.ub + p_r^j.ub$ \comment{ Update the best split risk } \\
\> \> $split\_feature \leftarrow j$ \comment{Best feature index to split on which minimizes loss upper bound} \\
\comment{ Calculate the total positive and negative weights of each equivalence class} \\
\textbf{for} each equivalence class $u \in [1,U]$ \textbf{define}\\
\> \comment{Values come from equivalence class matrix $Z$ as seen in Algorithm \ref{alg:upperbound}} \\
\> $z_u^- = \frac{1}{N}\sum_{i=1}^{N}\mathds{1}[y_i = 0 \land x_i = z_u]$ \comment {total negatives } \\
\> $z_u^+ = \frac{1}{N}\sum_{i=1}^{N}\mathds{1}[y_i = 1 \land x_i = z_u]$ \comment {total positives} \\
\comment{ Select only the positive and negative weights captured by $s$ } \\
$negatives \leftarrow \sum_u s_u z_u^-$ \\
$positives \leftarrow \sum_u s_u z_u^+$ \\
\comment{ Set the leaf prediction based on class with the higher selected total weight } \\
\textbf{if} $negatives < positives$ \textbf{then} \\
\> \comment{Leaf predicts the majority class as 1 since positive weights are higher} \\
\> $base\_prediction.pred \leftarrow 1$ \\
\comment{Base case: If the risk of remaining as a leaf is better than splitting, remain as leaf} \\
\textbf{if} $base\_bound \leq split\_feature$ \\
\> \comment{Construct and return a leaf node} \\
\> $t.type \leftarrow leaf$ \\
\> $t.prediction \leftarrow base\_prediction$ \\
\> \textbf{return} $t$ \\
\comment{Recursive case: One of the splits performs better than the leaf} \\
\comment{Generate left and right splits based on best split feature} \\
$s_l, s_r \leftarrow \text{split}(s,split\_feature,Z)$ \\
\comment{Recurse onto child keys to create left and right subtrees}\\
$t_l \leftarrow \text{extract}(s_l)$ \\
$t_r \leftarrow \text{extract}(s_r)$ \\
\comment{Construct and return a split node containing the left and right subtrees} \\
$t.type \leftarrow tree$ \\
$t.split \leftarrow split\_feature$ \\
$t.left \leftarrow t_l$ \\
$t.right \leftarrow t_r$ \\
\textbf{return} $t$
\end{tabbing}
\end{minipage}
\end{algorithm}

\end{document}